\documentclass{article}

\usepackage{microtype}
\usepackage{graphicx}
\usepackage{subfigure}
\usepackage{booktabs} 

\usepackage{url}            
\usepackage{amsfonts}       
\usepackage{nicefrac}       
\usepackage{microtype}      
\usepackage{algorithmic}
\usepackage{algorithm}
\usepackage{amsmath}
\usepackage{amssymb}
\usepackage{amsthm}
\usepackage{comment}
\usepackage{soul}

\usepackage{hyperref}

\usepackage{xcolor}
\definecolor{dark-red}{rgb}{0.4,0.15,0.15}
\definecolor{dark-blue}{rgb}{0.15,0.15,0.4}
\definecolor{medium-blue}{rgb}{0,0,0.5}
\hypersetup{
  colorlinks, linkcolor={dark-blue},
  citecolor={dark-blue}, urlcolor={medium-blue}
}

\renewcommand{\floatpagefraction}{.8}

\newcommand{\mE}{\mathcal{E}}

\newtheorem{theorem}{Theorem}
\newtheorem{proposition}[theorem]{Proposition}

\usepackage[accepted]{icml2020new}

\icmltitlerunning{Bayesian Deep Learning and a Probabilistic Perspective of Generalization}

\begin{document}

\twocolumn[
\icmltitle{Bayesian Deep Learning and a Probabilistic Perspective of Generalization}

\begin{icmlauthorlist}
\icmlauthor{Andrew Gordon Wilson}{} \quad
\icmlauthor{Pavel Izmailov}{} \\ 
New York University
\end{icmlauthorlist}

\icmlkeywords{Machine Learning, ICML}

\vskip 0.3in
]

\begin{abstract}
The key distinguishing property of a Bayesian approach is marginalization, rather
than using a single setting of weights.
Bayesian marginalization can particularly improve the accuracy and calibration of 
modern deep neural networks, which are typically underspecified by the data, and 
can represent many compelling but different solutions. We show that deep ensembles
provide an effective mechanism for approximate Bayesian marginalization, and propose
a related approach that further improves the predictive distribution by marginalizing within
basins of attraction, without significant overhead. We also investigate the prior over 
functions implied by a vague distribution over neural network weights, explaining the 
generalization properties of such models from a probabilistic perspective. From this 
perspective, we explain results that have been presented as mysterious and distinct to 
neural network generalization, such as the ability to fit images with random labels, and 
show that these results can be reproduced with Gaussian processes. We also show that
Bayesian model averaging alleviates double descent, resulting in monotonic performance
improvements with increased flexibility.
Finally, we provide
a Bayesian perspective on tempering for calibrating predictive distributions.
\end{abstract}

\section{Introduction}
\label{sec: intro}

Imagine fitting the airline passenger data in Figure~\ref{fig:airline}. Which model would you choose: 
(1) $f_1(x) = w_0 + w_1 x$, (2) $f_2(x) = \sum_{j=0}^{3} w_j x^{j}$, or (3) 
$f_3(x) = \sum_{j=0}^{10^4} w_j x^{j}$? 

\begin{figure}[!h]
	\centering
	\includegraphics[width=0.4\textwidth]{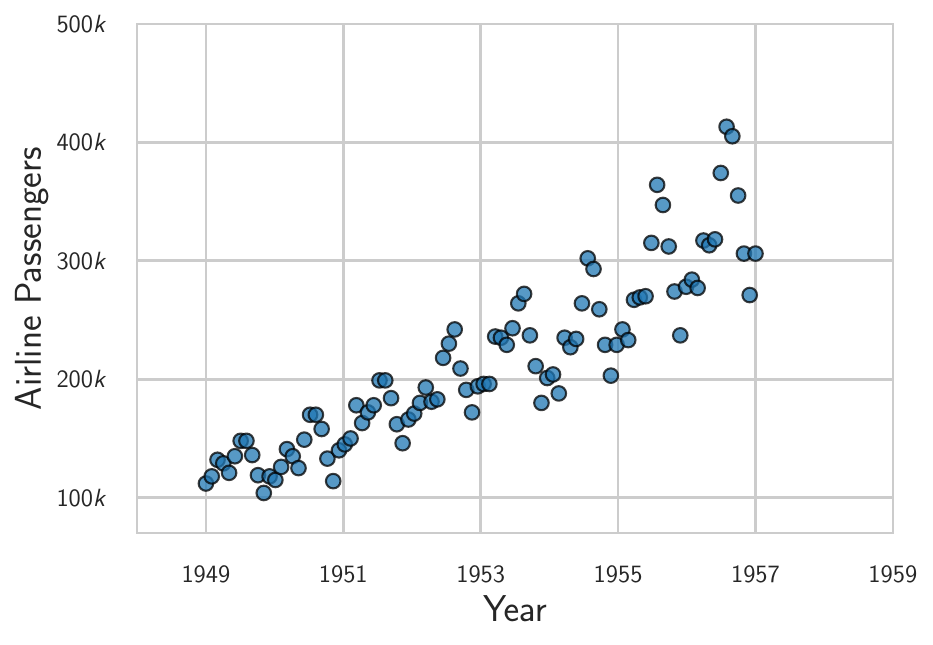}
	\caption{
    Airline passenger numbers recorded monthly.
} 
  \label{fig:airline}
  \vspace{-3mm}
\end{figure}

Put this way, most audiences 
overwhelmingly favour choices (1) and (2), for fear
of overfitting. But of these options, choice (3) most honestly
represents our beliefs. Indeed, it is likely that the ground truth explanation for 
the data is out of class for any of these choices, but there
is some setting of the coefficients $\{w_j\}$ in choice (3) which provides
a better description of reality than could be managed by choices (1) and
(2), which are special cases of choice (3). 
Moreover, our beliefs about the generative processes 
for our observations, which are often very sophisticated, 
typically ought to be independent of how many data 
points we happen to observe.

\begin{figure*}[h]
	\centering
	
	\subfigure[]{
		\includegraphics[clip, trim=0.cm .4cm 0.cm 0.cm,height=0.22\textwidth]{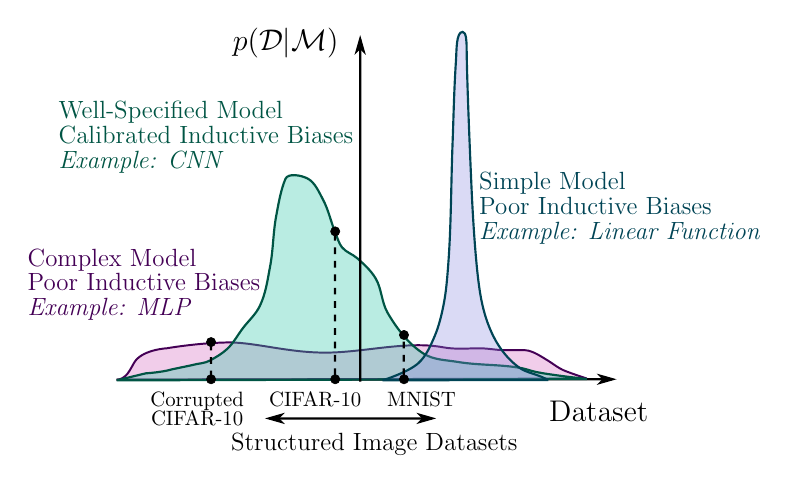}
		\label{fig:gen}
    }
    \hspace{-1.2cm}
	\subfigure[]{
		\includegraphics[clip, trim=0.cm .7cm 0.cm 0.cm,height=0.22\textwidth]{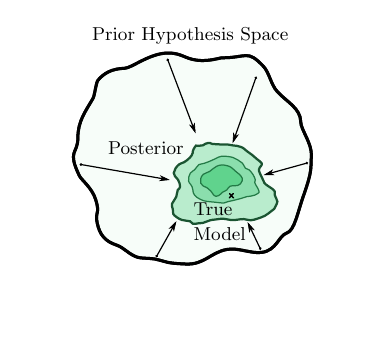}
		\label{fig:good}
    }
    \hspace{-2.2cm}
	\subfigure[]{
		\includegraphics[clip, trim=0.cm .7cm 0.cm 0.cm,height=0.22\textwidth]{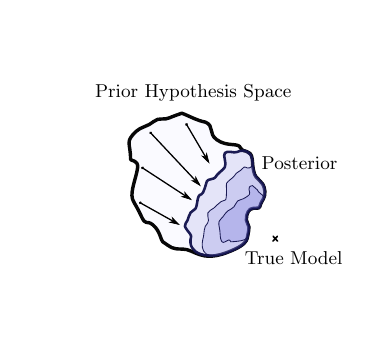}
		\label{fig:simple}
    }
    \hspace{-1.6cm}
	\subfigure[]{
		\includegraphics[clip, trim=0.cm .7cm 0.cm 0.cm,height=0.22\textwidth]{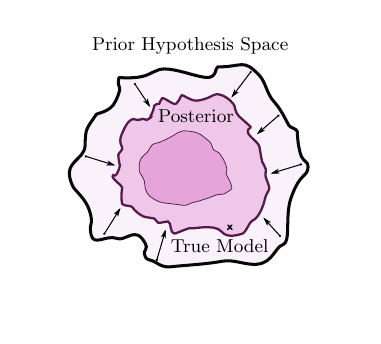}
		\label{fig:complex}
    }
	\caption{
        \textbf{A probabilistic perspective of generalization.} (a) Ideally, a model supports a wide range of datasets, but with inductive biases that provide 
        high prior probability to a particular class of problems being considered. Here, the CNN is preferred over the linear model 
        and the fully-connected MLP for CIFAR-10 (while we do not consider MLP models to in general have poor inductive biases, here we are considering a hypothetical example involving images and a very large MLP). (b) By representing a large hypothesis space, a model can contract around 
        a true solution, which in the real-world is often very sophisticated. (c) With truncated support, a model will converge to an erroneous solution. (d) Even
        if the hypothesis space contains the truth, a model will not efficiently contract unless it also has reasonable inductive biases.
        }
        \label{fig:conceptual}
\end{figure*}

And in modern practice, we are implicitly favouring
choice (3): we often use neural networks with millions of parameters to fit
datasets with thousands of points. Furthermore, 
non-parametric methods such as Gaussian processes often involve infinitely 
many parameters, enabling the flexibility for universal approximation \citep{rasmussen06}, yet in many cases 
provide very simple predictive distributions. Indeed, parameter counting is 
a poor proxy for understanding generalization behaviour.

From a probabilistic perspective, we argue that generalization depends 
largely on \emph{two} properties, the \emph{support} and the \emph{inductive biases}
of a model. Consider Figure \ref{fig:gen}, where on the horizontal axis we have a 
conceptualization of all possible datasets, and on the vertical axis the Bayesian \emph{evidence} 
for a model. The evidence, or marginal likelihood, $p(\mathcal{D}|\mathcal{M}) = \int p(\mathcal{D}|\mathcal{M},w) p(w) dw$, 
is the probability we would generate a dataset if we were to randomly sample from the prior 
over functions $p(f(x))$ induced by a prior over parameters $p(w)$. We define the support as the range of 
datasets for which $p(\mathcal{D}|\mathcal{M}) > 0$. We define the inductive biases as the relative prior
probabilities of different datasets --- the \emph{distribution of support} given by $p(\mathcal{D}|\mathcal{M})$.
A similar schematic to Figure~\ref{fig:gen} was used by \citet{mackay1992bayesian} to 
understand an Occam's razor effect in using the evidence for model selection; we believe it can also
be used to reason about model construction and generalization.

From this perspective, we want the support of the model to be large so that we can represent any hypothesis
we believe to be possible, even if it is unlikely. We would even want the model to be able to represent 
pure noise, such as noisy CIFAR \citep{zhang2016understanding}, as long as we honestly believe there is 
some non-zero, but potentially arbitrarily small, probability that the data are simply noise.
Crucially, we
also need the inductive biases to carefully represent which hypotheses we believe to be a priori likely for a particular 
problem class. If we are modelling images, then our model should have statistical properties, such as convolutional
structure, which are good descriptions of images.

Figure~\ref{fig:gen} illustrates three models. We can imagine the blue curve as a simple linear function, $f(x) = w_0 + w_1x$,
combined with a distribution over parameters $p(w_0, w_1)$, e.g., $\mathcal{N}(0,I)$, which induces a distribution over
functions $p(f(x))$. Parameters we sample from our prior $p(w_0, w_1)$ give rise to functions $f(x)$ that correspond to
straight lines with different slopes and intercepts. This model thus has truncated support: it cannot even represent a 
quadratic function. But because the marginal likelihood must normalize over datasets $\mathcal{D}$, this model assigns
much mass to the datasets it does support. The red curve could represent a large fully-connected 
MLP. This model is highly flexible, but distributes its support across datasets too evenly to be particularly compelling for 
many image datasets. The green curve could represent a convolutional neural network, which represents a compelling 
specification of support and inductive biases for image recognition: this model has the flexibility to represent many solutions, 
but its structural properties provide particularly good support for many image problems. 

With large support, we cast a wide enough net that the posterior can contract around the true solution to a given problem as in 
Figure~\ref{fig:good}, which in reality we often believe to be very sophisticated. On the other hand, the simple model will
have a posterior that contracts around an erroneous solution if it is not contained in the hypothesis space as in Figure~\ref{fig:simple}.
Moreover, in Figure~\ref{fig:complex}, the model has wide support, but does not contract around a good solution because its 
support is too evenly distributed.

Returning to the opening example, we can justify the high order polynomial by wanting large support. But we would still have to 
carefully choose the prior on the coefficients to induce a distribution over functions that would have reasonable inductive
biases. Indeed, this Bayesian notion of generalization is not based on a single number, but is a two dimensional concept. 
From this probabilistic perspective, it is crucial not to conflate the \emph{flexibility} of a model with the 
\emph{complexity} of a model class. Indeed Gaussian processes with RBF kernels have large support, and 
are thus flexible, but have inductive biases towards very simple solutions. We also see that \emph{parameter counting}
has no significance in this perspective of generalization: what matters is how a distribution over parameters combines with 
a functional form of a model, to induce a distribution over solutions.
Rademacher complexity \citep{mohri2009rademacher}, VC dimension \citep{vapnik1998adaptive}, and many conventional metrics, are by contrast \emph{one dimensional 
notions}, corresponding roughly to the support of the model, which is why they have been found to 
provide an incomplete picture of generalization in deep learning \citep{zhang2016understanding}.

In this paper we reason about Bayesian deep learning from a probabilistic perspective of generalization.
The key distinguishing property of a Bayesian approach is 
marginalization instead of optimization, where we represent solutions given by all settings of parameters 
weighted by their posterior probabilities, rather than bet everything on a single setting of parameters.
Neural networks are typically underspecified by the data,
and can represent many different but high performing models corresponding to different 
settings of parameters, which is exactly when marginalization will make the biggest difference 
for accuracy and calibration. Moreover, we clarify that the recent deep ensembles \citep{lakshminarayanan2017simple} are not 
a competing approach to Bayesian inference, but can be viewed as a compelling
mechanism for Bayesian marginalization. Indeed, we empirically demonstrate that deep ensembles can 
provide a \emph{better} approximation to the Bayesian predictive distribution than standard Bayesian 
approaches.  We further propose a new method, MultiSWAG, inspired by deep
ensembles, which marginalizes within basins of attraction --- achieving significantly improved 
performance, with a similar training time. 

We then investigate the properties of priors over functions induced
by priors over the weights of neural networks, showing that they have
reasonable inductive biases. 
We also show
that the mysterious generalization properties recently presented in \citet{zhang2016understanding} can be understood by reasoning about
prior distributions over functions, and are not specific to neural networks. Indeed, we show Gaussian processes can also perfectly
fit images with random labels, yet generalize on the noise-free problem. These results are a consequence
of large support but reasonable inductive biases for common problem settings. We further show that while 
Bayesian neural networks can fit the noisy datasets, the marginal likelihood has much better support for the 
noise free datasets, in line with Figure~\ref{fig:conceptual}. We additionally show that the multimodal marginalization
in MultiSWAG alleviates double descent, so as to achieve monotonic improvements in performance with model flexibility,
in line with our perspective of generalization. MultiSWAG also provides significant improvements in both accuracy and 
NLL over SGD training and unimodal marginalization. 
Finally we provide several perspectives on tempering in Bayesian deep learning.

In the Appendix we provide several additional experiments and results.
We also provide code at \url{https://github.com/izmailovpavel/understandingbdl}.

\section{Related Work}

Notable early works on Bayesian neural networks include \citet{mackay1992bayesian}, \citet{mackay1995probable}, and \citet{neal1996}. These works
generally argue in favour of making the model class for Bayesian approaches as flexible as possible, in line with \citet{box1973bayesian}. 
Accordingly, \citet{neal1996} pursued the limits of large Bayesian neural
networks, showing that as the number of hidden units approached infinity, these models become Gaussian processes with particular kernel functions.
This work harmonizes with recent work describing the neural tangent kernel \citep[e.g.,][]{jacot2018neural}.  

The marginal likelihood is often used for Bayesian hypothesis testing, model comparison, and hyperparameter tuning, 
with \emph{Bayes factors} used to select between models \citep{bayesfactors1995}. \citet[][Ch. 28]{mackay2003information}
uses a diagram similar to Fig~\ref{fig:gen} to show the marginal likelihood has an \emph{Occam's razor} property, favouring 
the simplest model consistent with a given dataset, even if the prior assigns equal probability to the various models. 
\citet{rasmussen01} reasons about how the marginal likelihood can favour large flexible models, as long as such models 
correspond to a reasonable distribution over functions.

There has been much recent interest in developing Bayesian approaches for modern deep learning, with new challenges and 
architectures quite different from what had been considered in early work. Recent work has largely focused on scalable inference 
\citep[e.g.,][]{blundell2015weight, gal2016dropout, kendall2017uncertainties, ritter2018scalable, khan2018fast, maddoxfast2019},
function-space inspired priors
\citep[e.g.,][]{yang2019output, louizos2019functional, sun2019functional, hafner2018reliable}, and developing 
flat objective priors in parameter space, directly leveraging the biases of the neural network functional form 
\citep[e.g,][]{nalisnick18}. \citet{wilson2020case} provides a note motivating Bayesian deep learning.

Early works tend to 
provide a connection between loss geometry and generalization using minimum description
length frameworks \citep[e.g.,][]{hinton1993keeping, hochreiter1997flat, mackay1995probable}.
Empirically, \citet{keskar2016large} argue that smaller batch SGD provides better generalization
than large batch SGD, by finding flatter minima. \citet{chaudhari2019entropy} and \citet{izmailov2018}
design optimization procedures to specifically find flat minima. 

By connecting flat solutions with ensemble
approximations, \citet{izmailov2018} also suggest that
functions associated with parameters in flat regions ought to provide different predictions on test data, 
for flatness to be helpful in generalization, which is distinct from the flatness in \citet{dinh2017sharp}.
\citet{garipov2018} also show that there are mode connecting curves, forming loss valleys, which contain
a variety of distinct solutions. We argue that flat regions of the loss containing a diversity of solutions is 
particularly relevant for Bayesian model averaging, since the model average will then contain many 
compelling and complementary explanations for the data. Additionally, \citet{huang2019understanding} describes neural networks
as having a \emph{blessing of dimensionality}, since flat regions will occupy much greater volume in a 
high dimensional space, which we argue means that flat solutions will dominate in the Bayesian model average.

\citet{smith2017bayesian} and 
\citet[][Chapter 28]{mackay2003information} additionally connect the width of the posterior with Occam factors; 
from a Bayesian perspective, larger width corresponds to a smaller Occam factor, and thus ought 
to provide better generalization. 
\citet{dziugaite2017computing} and \citet{smith2017bayesian} also provide different perspectives on  
the results in \citet{zhang2016understanding},
which shows that deep convolutional neural networks can fit CIFAR-10 with random labels and no training error.  
The PAC-Bayes bound of \citet{dziugaite2017computing} becomes vacuous when
applied to randomly-labelled binary MNIST.  \citet{smith2017bayesian} 
show that logistic regression can fit noisy labels on sub-sampled MNIST, interpreting the result from an Occam factor
perspective.

In general, PAC-Bayes provides a compelling framework for deriving explicit non-asymptotic generalization bounds for stochastic networks with distributions over parameters
\citep{mcallester1999pac, langford2002not, dziugaite2017computing, neyshabur2017exploring, neyshabur2018a, masegosa2019learning, jiang2019fantastic, guedj2019primer, alquier2021user}.
\citet{langford2002not} devised a PAC-Bayes generalization bound for small stochastic neural networks (two layer with two hidden units) achieving an improvement over the 
existing deterministic generalization bounds. \citet{dziugaite2017computing} extended this approach, optimizing a PAC-Bayes bound with respect to a parametric distribution 
over the weights of the network, exploiting the flatness of solutions discovered by SGD, for non-vacuous bounds with an overparametrized network on binary MNIST.
\citet{neyshabur2017exploring} also discuss the connection between PAC-Bayes bounds and sharpness, and \citet{neyshabur2018a} devises PAC-Bayes bounds based on spectral norms of the layers and the Frobenius norm of the weights of the network. \citet{achille2018emergence} additionally combine PAC-Bayes and information theoretic approaches to argue that flat minima have low information content. \citet{masegosa2019learning} also proposes variational and ensemble learning methods based on PAC-Bayes analysis under model misspecification. \citet{jiang2019fantastic} provide a review and comparison of several generalization bounds, including PAC-Bayes.

Our contributions are largely orthogonal and complementary to PAC-Bayes. PAC-Bayes bounds can be
be improved by, e.g.\ fewer parameters, and very compact priors, 
which can be different from what provides optimal generalization. 
From our perspective, model flexibility and priors with \emph{large} support, rather than compactness, 
are desirable. Moreover, we show the great significance of multi-basin marginalization for 
generalization, whereas multi-basin marginalization has a minimal logarithmic effect on PAC-Bayes 
bounds. Indeed, \emph{marginalization} --- a posterior weighted model average --- is our key focus, whereas
PAC-Bayes bounds are typically bounding the empirical risk of a single sample.
In general, our focus is complementary to PAC-Bayes, aiming to provide \emph{prescriptive} intuitions on 
model construction, inference, and neural network priors, as well as new 
connections between Bayesian model averaging and deep ensembles, benefits of Bayesian model 
averaging in the context 
of modern deep neural networks, views of marginalization that contrast with
simple Monte Carlo, and new methods for Bayesian marginalization in deep learning.

In other work,
\citet{pearce2018uncertainty} propose a modification of deep ensembles and
argue that it performs approximate Bayesian inference, and 
\citet{gustafsson2019evaluating} briefly mention how deep ensembles can be viewed as samples from 
an approximate posterior. In the context of deep ensembles, we believe it is natural to consider the 
BMA integral separately from the simple Monte Carlo approximation that is often used to approximate this integral;
to compute an accurate predictive distribution, we do not need samples from a posterior, or even a faithful approximation to 
the posterior. 

\citet{fort2019deep} considered the diversity of predictions produced
by models from multiple independent SGD runs,
and suggested to ensemble averages of SGD iterates.
Although MultiSWA (one of the methods considered in Section \ref{sec: emp}) is related to this idea, 
the crucial practical difference is that MultiSWA uses a learning rate schedule that 
selects for flat regions of the loss, the key to the success of the SWA method \citep{izmailov2018}.
Section \ref{sec: emp} also shows that MultiSWAG, which we propose for multimodal
Bayesian marginalization, outperforms MultiSWA.

Double descent, which describes generalization error that decreases, increases, and then again decreases with 
model flexibility, was demonstrated early by \citet{opper1990ability}. Recently, \citet{belkin2019reconciling} 
extensively demonstrated double descent, leading to a surge of modern interest, with \citet{nakkiran2019deep}
showing double descent in deep learning. \citet{nakkiran2020optimal} shows that tuned $l_2$ regularization can 
mitigate double descent. Alternatively, we show that Bayesian model averaging, particularly based on multimodal
marginalization, can alleviate even prominent double descent behaviour.

Tempering in Bayesian modelling has been considered under the names \emph{Safe Bayes}, 
\emph{generalized Bayesian inference}, and \emph{fractional Bayesian inference} \citep[e.g.,][]{de2019safe,grunwald2017inconsistency,barron1991minimum,
walker2001bayesian,zhang2006information,bissiri2016general,grunwald2012safe}. We provide 
several perspectives of tempering in Bayesian deep learning, and analyze the results in a
recent paper by \citet{wenzel2020good} that questions tempering for Bayesian neural networks.

\section{Bayesian Marginalization}
\label{sec: bayesianmarginalization}

Often the predictive distribution we want to compute is given by 
\begin{align}
p(y|x, \mathcal{D}) = \int p(y|x, w) p(w|\mathcal{D}) dw \,.
\label{eqn: bma}
\end{align}
The outputs are $y$ (e.g., regression values, class labels, \dots), indexed by inputs $x$ (e.g. spatial locations, images, \dots), the weights (or parameters) of the neural network $f(x;w)$ are $w$, and $\mathcal{D}$ are the data. Eq.~\eqref{eqn: bma} 
represents a \emph{Bayesian model average} (BMA). Rather than bet everything on one hypothesis --- with a single setting of parameters $w$ --- we want to use all settings of parameters, 
weighted by their posterior probabilities. This procedure is called \emph{marginalization} of the parameters $w$, as the predictive distribution of interest
no longer conditions on $w$. This is not a controversial equation, but simply the sum and product rules of probability.

\subsection{Importance of Marginalization in Deep Learning}
\label{sec: margimportance}

In general, we can view classical training as performing approximate Bayesian inference, using the approximate posterior 
$p(w | \mathcal{D}) \approx \delta(w=\hat{w})$ to compute Eq.~\eqref{eqn: bma}, where $\delta$ is a Dirac delta function that is zero everywhere 
except at $\hat{w} = \text{argmax}_w p(w|\mathcal{D})$.  In this case, we recover the standard predictive distribution $p(y|x,\hat{w})$. From this perspective,
many alternatives, albeit imperfect, will be preferable --- including impoverished Gaussian posterior approximations for $p(w|\mathcal{D})$, 
even if the posterior or likelihood are actually highly non-Gaussian and multimodal. 

The difference between a classical and Bayesian approach will depend on how sharp the posterior $p(w|\mathcal{D})$
becomes. If the posterior is sharply peaked, and the conditional predictive distribution $p(y|x,w)$ does not vary significantly where the posterior has mass, 
there may be almost no difference, since a delta function may then be a reasonable approximation
of the posterior for the purpose of BMA. However, modern neural networks are usually highly underspecified by the available data, and therefore have diffuse likelihoods $p(\mathcal{D}|w)$, not strongly favouring any one setting of parameters. Not only are
the likelihoods diffuse, but different settings of the parameters correspond to a diverse variety of compelling hypotheses for the data \citep{garipov2018, izmailov2019subspace}.
This is exactly the setting when we \emph{most} want to perform a Bayesian model average, which will lead to an ensemble containing
many different but high performing models, for better calibration \emph{and} accuracy than classical training.

\paragraph{Loss Valleys.} 
Flat regions of low loss (negative log posterior density $- \log p(w|\mathcal{D})$) are associated with good generalization \citep[e.g.,][]{hochreiter1997flat, hinton1993keeping, dziugaite2017computing, izmailov2018, keskar2016large}. While flat solutions that generalize poorly can be contrived through reparametrization \citep{dinh2017sharp}, the flat regions that lead to good generalization contain a \emph{diversity} of high performing models on test data \citep{izmailov2018}, corresponding to different parameter settings in those regions. And indeed, there are large contiguous regions of low loss that contain such solutions, even connecting together different SGD solutions \citep{garipov2018, izmailov2019subspace} (see also Figure~\ref{fig:loss_valleys}, Appendix).

Since these regions of the loss represent a large volume in a high-dimensional space \citep{huang2019understanding}, and provide a diversity of solutions, they will dominate in forming the predictive distribution in a Bayesian model average. By contrast, if the parameters in these regions provided similar functions, as would be the case in flatness obtained through reparametrization, these functions would be redundant in the model average. That is, although the solutions of high posterior density can provide poor generalization, it is the solutions that generalize well that will have greatest posterior \emph{mass}, and thus be automatically favoured by the BMA.

\paragraph{Calibration by Epistemic Uncertainty Representation.} It has been noticed that 
modern neural networks are often \emph{miscalibrated} in the sense that their predictions are 
typically \emph{overconfident} \citep{guo2017calibration}. For example, in classification the highest softmax output of a convolutional neural network is typically much larger than the probability of the associated class label. The fundamental reason for miscalibration is ignoring epistemic uncertainty. A neural network can represent many models that are consistent with our observations. By selecting only one, in a classical procedure, we lose uncertainty when the models disagree for a test point. In regression, we can visualize epistemic uncertainty by looking at the spread of the predictive distribution; as we move away from the data, there are a greater variety of consistent solutions, leading to larger uncertainty, as in Figure~\ref{fig:predictive}. 
We can further calibrate the model with tempering, which we discuss in the Appendix Section~\ref{sec:app_temperature}.

\paragraph{Accuracy.} An often overlooked benefit of Bayesian model averaging in \emph{modern} deep learning is improved \emph{accuracy}. If we average the predictions of many high performing models that disagree in some cases, we should see significantly improved accuracy. This benefit is now starting to be observed in practice \citep[e.g.,][]{izmailov2019subspace}. Improvements in accuracy are very convincingly exemplified by \emph{deep ensembles} \citep{lakshminarayanan2017simple}, which have been perceived as a competing approach to Bayesian methods, but in fact provides a compelling mechanism for approximate Bayesian model averaging, as we show in 
Section~\ref{sec: deepensembles}. We also demonstrate significant accuracy benefits for multimodal Bayesian marginalization in Section~\ref{sec:double_descent}.

\subsection{Beyond Monte Carlo}
\label{sec: beyondmc}

Nearly all approaches to estimating the integral in Eq.~\eqref{eqn: bma}, when it cannot be computed in closed form,
involve a \emph{simple Monte Carlo} approximation:
$p(y|x, \mathcal{D}) \approx \frac{1}{J} \sum_{j=1}^{J} p(y|x,w_j) \,,  w_j \sim p(w|\mathcal{D})$. 
In practice, the samples from the posterior $p(w|\mathcal{D})$ are also approximate, and
found through MCMC or deterministic methods. The 
deterministic methods approximate $p(w|\mathcal{D})$ with a different more convenient density
$q(w|\mathcal{D}, \theta)$ from which we can sample, often chosen to be Gaussian. The 
parameters $\theta$ are selected to make $q$ close to $p$ in some sense; for example,
variational approximations \citep[e.g.,][]{beal2003variational}, which have emerged as a popular deterministic approach, 
find $\text{argmin}_{\theta} \mathcal{KL}(q||p)$. Other standard deterministic approximations include  
Laplace \citep[e.g.,][]{mackay1995probable}, EP \citep{minka01}, and INLA \citep{rue2009approximate}.

From the perspective of estimating the predictive
distribution in Eq.~\eqref{eqn: bma}, we can view simple Monte Carlo 
as approximating the posterior
with a set of point masses, with locations given by samples from another approximate posterior $q$, even if $q$ is a continuous 
distribution. That is, 
$p(w | \mathcal{D}) \approx \sum_{j=1}^{J} \delta(w=w_j) \,, w_j \sim q(w | \mathcal{D})$. 

Ultimately, the goal is to accurately compute the predictive distribution in Eq.~\eqref{eqn: bma}, rather
than find a generally accurate representation of the posterior. In particular, we must carefully
represent the posterior in regions that will make the greatest contributions to the BMA integral. 
In terms of efficiently computing the predictive distribution, 
we do not necessarily want to place point masses at locations given by samples from the posterior. 
For example, functional diversity is important for a good approximation to the BMA integral, because 
we are summing together terms of the form $p(y|x,w)$; if two settings of 
the weights $w_i$ and $w_j$ each provide high likelihood (and consequently high posterior density), but give rise to similar functions $f(x; w_i)$, $f(x; w_j)$, 
then they will be largely redundant in the model average, and the second setting of parameters will not contribute much to estimating the BMA integral 
for the unconditional predictive distribution.
In Sections \ref{sec: deepensembles} and \ref{sec: emp}, we consider how various approaches approximate the predictive distribution.

\subsection{Deep Ensembles are BMA}
\label{sec: deepensembles}

\emph{Deep ensembles} \citep{lakshminarayanan2017simple} is fast becoming a gold standard for accurate
and well-calibrated predictive distributions. 
Recent reports \citep[e.g.,][]{ovadia2019can, ashukha2020pitfalls} show 
that deep ensembles appear to outperform some particular approaches 
to Bayesian neural networks for uncertainty representation, leading to the confusion that deep ensembles and Bayesian methods are competing
approaches. These methods are often explicitly referred to as non-Bayesian \citep[e.g.,][]{lakshminarayanan2017simple, ovadia2019can, wenzel2020good}.
To the contrary, we argue that deep ensembles are actually a compelling approach to Bayesian model averaging, 
in the vein of Section~\ref{sec: beyondmc}.

There is a fundamental difference between a Bayesian model average and some approaches to ensembling. The Bayesian
model average assumes that \emph{one} hypothesis (one parameter setting) is correct, and averages over models due to
an inability to distinguish between hypotheses given limited information \citep{minka2000bayesian}. As we observe more data, the posterior collapses onto a 
single hypothesis. If the true explanation for the data is a combination of hypotheses, then the Bayesian model average may
appear to perform worse as we observe more data. Some ensembling methods work by enriching the hypothesis space, and 
therefore do not collapse in this way. Deep ensembles, however, are formed by MAP or maximum likelihood retraining of the same 
architecture multiple times, leading to different basins of attraction. The deep ensemble will therefore collapse in the same way 
as a Bayesian model average, as the posterior concentrates. Since the hypotheses space (support) for a modern neural network 
is large, containing many different possible explanations for the data, posterior collapse will often be desirable.

Furthermore, by representing multiple basins of attraction, 
deep ensembles can provide a \emph{better} approximation to the BMA than the Bayesian approaches in 
\citet{ovadia2019can}. Indeed, the functional
diversity is important for a good approximation to the BMA integral, as per Section~\ref{sec: beyondmc}. The approaches referred to as Bayesian
in \citet{ovadia2019can} instead focus their approximation on a single basin, which may contain a lot of redundancy in function space, making a 
relatively minimal contribution to computing the Bayesian predictive distribution. On the other hand, retraining a neural network multiple times for
deep ensembles incurs a significant computational expense. The single basin approaches may be preferred if we are to control for computation.
We explore these questions in Section~\ref{sec: emp}.

\section{An Empirical Study of Marginalization}
\label{sec: emp}

\begin{figure}[t]
	\centering
	
    \hspace{-0.5cm}
    \includegraphics[width=0.45\textwidth]{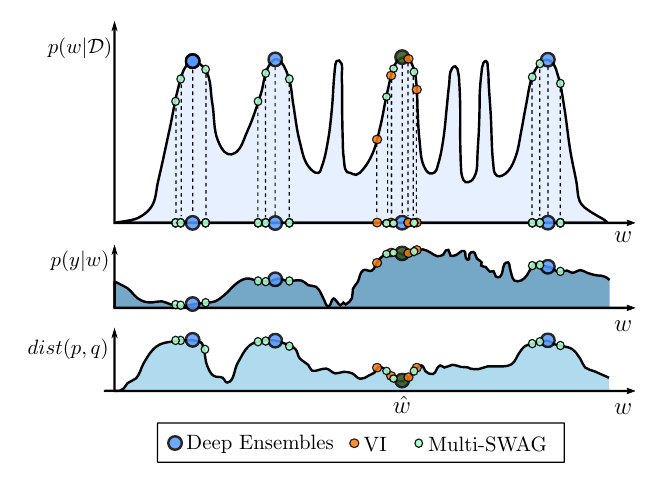}

	\caption{\textbf{Approximating the BMA.}  \\ 
	$p(y|x,\mathcal{D}) = \int p(y|x,w)p(w | \mathcal{D}) dw$. \textbf{Top:} 
	$p(w|\mathcal{D})$, with representations
	from VI (orange)
	deep ensembles (blue), MultiSWAG (red). 
	\textbf{Middle:} 
	$p(y|x,w)$ as a 
	function of $w$ for a test input $x$. This function does not vary much within modes, but
	changes significantly between modes. 
	\textbf{Bottom:} Distance between
	the true predictive distribution and the approximation, as a function of representing a posterior
	at an additional point $w$, assuming we have sampled the mode in dark green.
	There is more to be gained by exploring new basins, than continuing to explore the same basin.
        }
        \label{fig:beyondmc}
\end{figure}

\begin{figure*}[h]
	\centering
	
	\subfigure[Exact]{
		\includegraphics[width=0.23\textwidth]{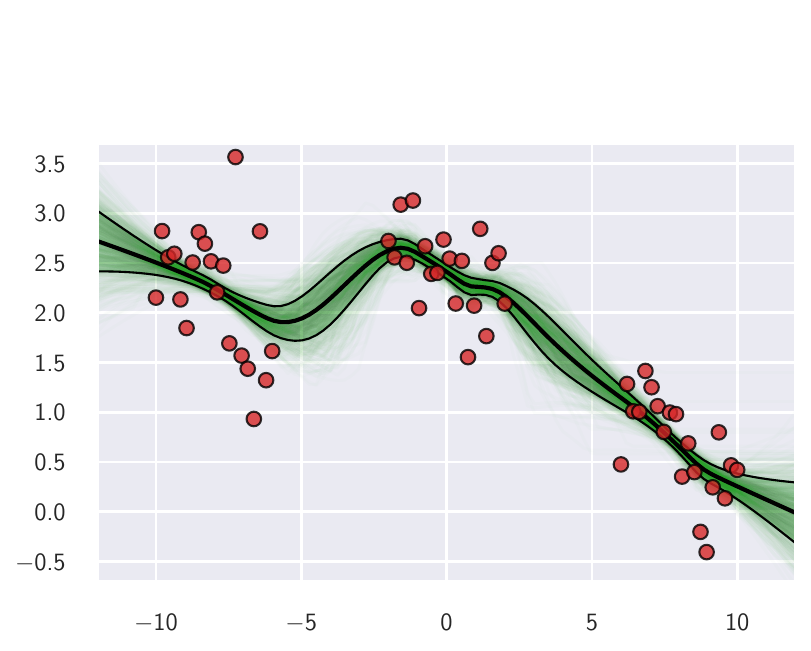}
    }
	\subfigure[Deep Ensembles]{
		\includegraphics[width=0.23\textwidth]{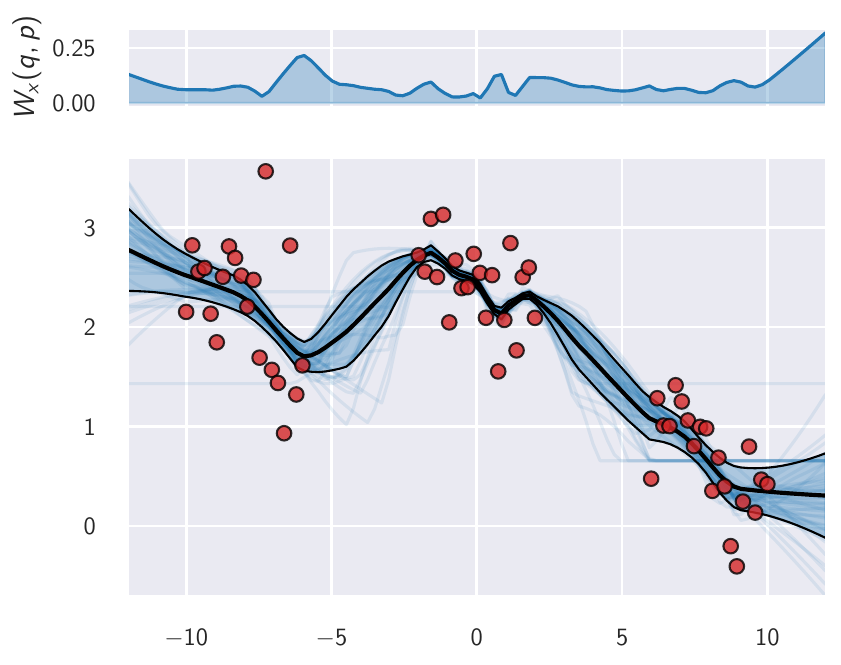}
    }
	\subfigure[Variational Inference]{
		\includegraphics[width=0.23\textwidth]{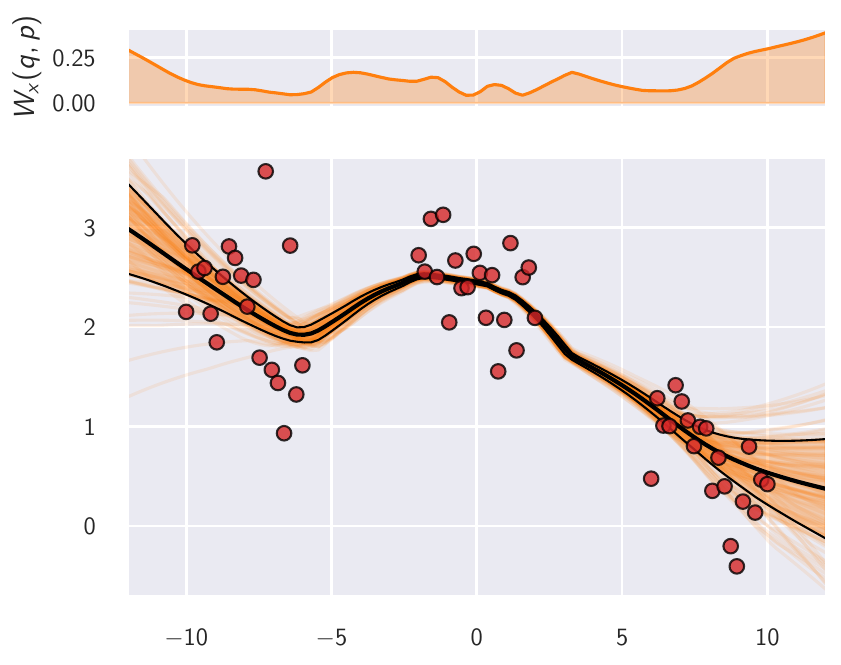}
    }
	\subfigure[Discrepancy with True BMA]{
		\includegraphics[width=0.23\textwidth]{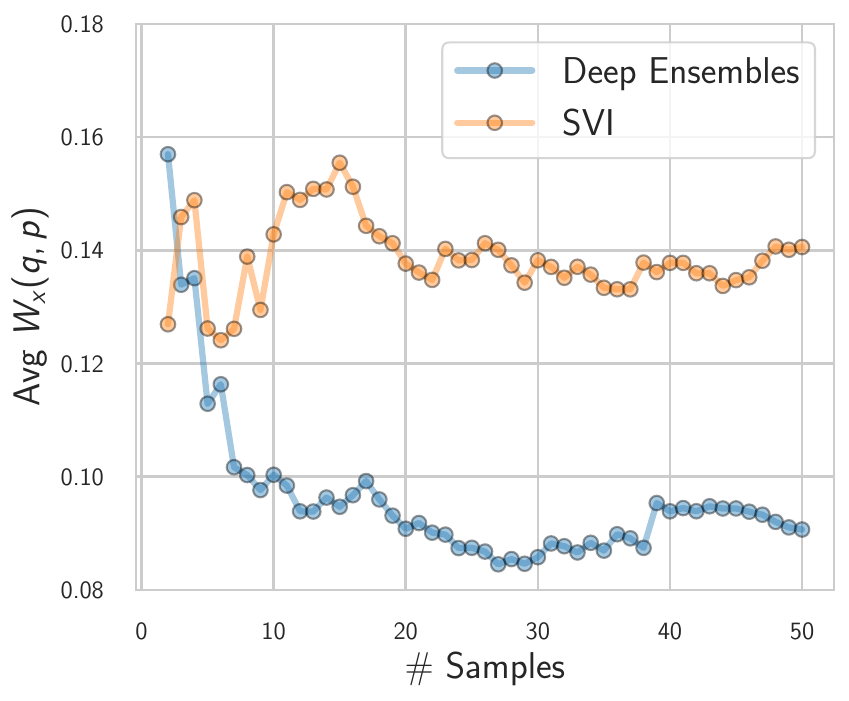}
    }
	\caption{\textbf{Approximating the true predictive distribution.}
        \textbf{(a)}: A close approximation of the true predictive distribution 
        obtained by combining $200$ HMC chains.
        \textbf{(b)}: Deep ensembles predictive distribution using $50$ independently
        trained networks.
        \textbf{(c)}: Predictive distribution for factorized variational inference (VI).
        \textbf{(d)}: Convergence of the predictive distributions for deep ensembles
        and variational inference as a function of the number of samples; 
        we measure the average Wasserstein distance between the 
        marginals in the range of input positions. The multi-basin deep ensembles approach provides a more faithful
        approximation of the Bayesian predictive distribution than the conventional
        single-basin VI approach, which is overconfident between data clusters. 
        The top panels show the Wasserstein distance between the true predictive
        distribution and the deep ensemble and VI approximations, as a function of
        inputs $x$.
        }
        \label{fig:predictive}
\end{figure*}

\begin{figure*}[t]
	\centering
	
	\includegraphics[width=0.95\textwidth]{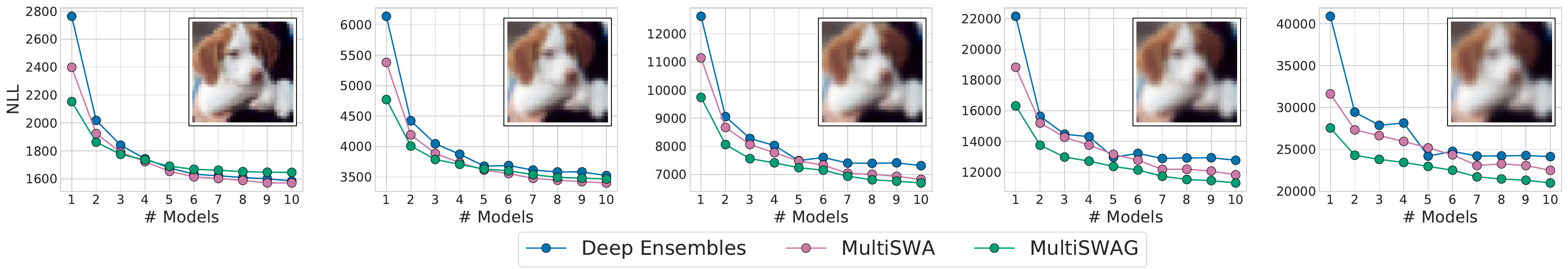}
	\caption{
		Negative log likelihood for Deep Ensembles, MultiSWAG and MultiSWA 
		using a PreResNet-20 on CIFAR-10 
		with varying intensity of the \textit{Gaussian blur} corruption.
        The image in each plot shows the intensity of corruption.
		For all levels of intensity, MultiSWAG and MultiSWA outperform Deep Ensembles
		for a small number of independent models.
		For high levels of corruption MultiSWAG significantly outperforms other
		methods even for many independent models.
        We present results for other corruptions in the Appendix.
    }
    \label{fig:ovadia}
\end{figure*}

We have shown that deep ensembles can be interpreted as an approximate approach to Bayesian marginalization, which selects
for functional diversity by representing multiple basins of attraction in the posterior. Most Bayesian deep learning methods 
instead focus on faithfully approximating a posterior within a single basin of attraction. We propose a new method, MultiSWAG, 
which combines these two types of approaches. MultiSWAG combines multiple independently trained 
SWAG approximations \citep{maddoxfast2019}, to create a mixture of Gaussians approximation to the posterior, with each 
Gaussian centred on a different basin of attraction. We note that MultiSWAG does not require \emph{any} additional training
time over standard deep ensembles.

We illustrate the conceptual difference between deep ensembles, a standard variational single basin approach, and MultiSWAG, 
in Figure~\ref{fig:beyondmc}. In the top panel, we have a conceptualization of a multimodal posterior. VI 
approximates the posterior with multiple samples within a single basin. But we see in the middle panel that the conditional
predictive distribution $p(y|x,w)$ does not vary significantly within the basin, and thus each additional sample contributes minimally
to computing the marginal predictive distribution $p(y|x,\mathcal{D})$. On the other hand, $p(y|x,w)$ varies significantly between basins,
and thus each point mass for deep ensembles contributes significantly to the marginal predictive distribution. 
By sampling within the basins, MultiSWAG provides additional contributions to the predictive distribution. In the bottom panel, we have the 
gain in approximating the predictive distribution when adding a point mass to the representation of the posterior, as a function 
of its location, assuming we have already sampled the mode in dark green. Including samples from 
different modes provides significant gain over continuing to sample from the same mode, and including weights in wide basins provide 
relatively more gain than the narrow ones.

In Figure \ref{fig:predictive} we evaluate single basin and multi-basin approaches in a case where we can near-exactly
compute the predictive distribution. We provide details for generating the data and training the models in Appendix \ref{sec:appendix_hmc}.
We see that the predictive distribution
given by deep ensembles is qualitatively closer to the true distribution, compared to the single basin variational method:
between data clusters, the deep ensemble approach provides a similar representation of epistemic uncertainty, whereas
the variational method is extremely overconfident in these regions. Moreover, we see that the Wasserstein distance 
between the true predictive distribution and these two approximations quickly shrinks with number of samples for deep 
ensembles, but is roughly independent of number of samples for the variational approach. Thus the deep ensemble 
is providing a better approximation of the Bayesian model average in Eq.~\eqref{eqn: bma} than the single basin 
variational approach, which has traditionally been labelled as the Bayesian alternative. 

\begin{figure*}[]
	\subfigure[$\alpha=0.02$]{
		\includegraphics[width=0.2\textwidth]{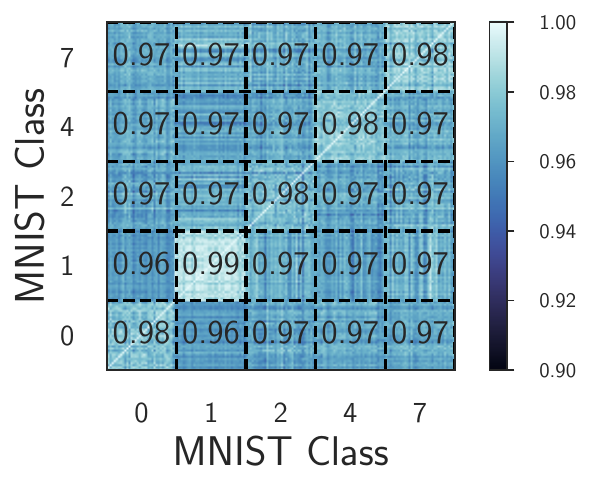}
    }
	\subfigure[$\alpha=0.1$]{
		\includegraphics[width=0.2\textwidth]{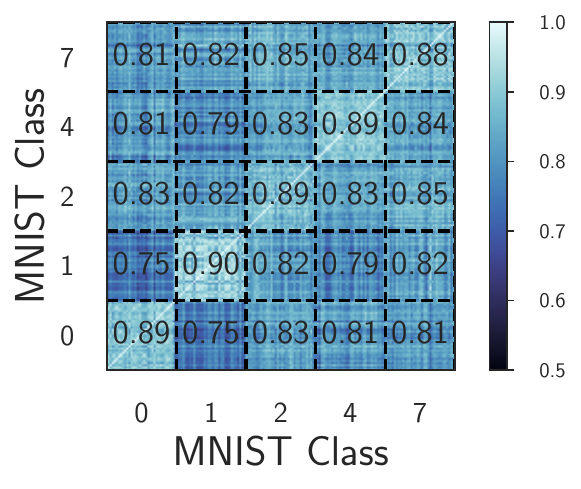}
    }
	\subfigure[$\alpha=1$]{
		\includegraphics[width=0.2\textwidth]{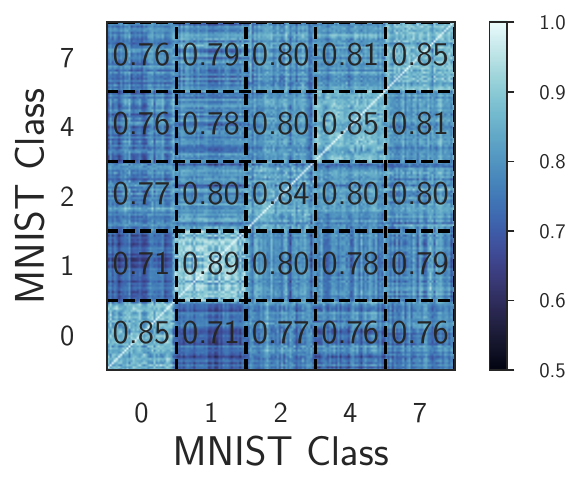}
    }
	\subfigure[]{
		\includegraphics[width=0.22\textwidth]{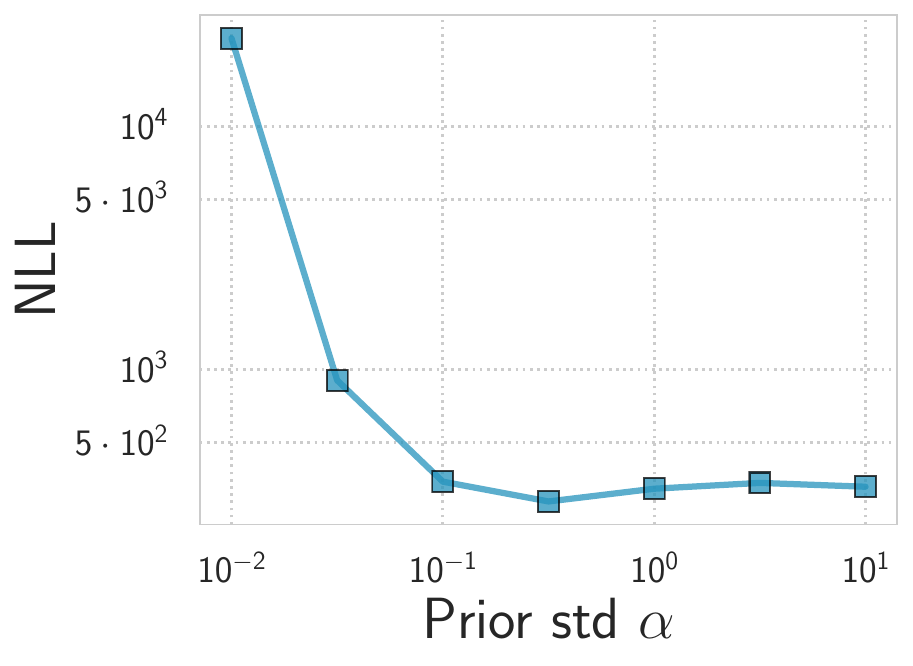}
        \label{fig:lenet_priordep1}
    }
	\caption{\textbf{Induced prior correlation function.}
        Average pairwise prior correlations for pairs of objects
        in classes $\{0, 1, 2, 4, 7\}$ of MNIST induced by LeNet-5 for $p(f(x;w))$ 
        when $p(w) = \mathcal{N}(0,\alpha^2I)$. 
        Images in the same class have higher prior correlations than images from different classes,
        suggesting that $p(f(x;w))$ has desirable inductive biases. The correlations slightly
        decrease with increases in $\alpha$. 
           \textbf{(d)}: NLL of an ensemble of $20$ 
        SWAG samples on MNIST as a function of 
        $\alpha$ using a LeNet-5.
        }
        \label{fig:prior_depnew}
\end{figure*}

Next, we evaluate MultiSWAG under distribution shift on the CIFAR-10 dataset 
\citep{krizhevsky2014cifar}, replicating the setup in \citet{ovadia2019can}. We consider $16$ data corruptions, 
each at $5$ different levels of severity, introduced by \citet{hendrycks2019benchmarking}. 
For each corruption, we evaluate the performance of deep ensembles and MultiSWAG
varying the training budget. For deep ensembles we show performance as a function of the number of independently
trained models in the ensemble. For MultiSWAG we show performance as a function
of the number of independent SWAG approximations that we construct; 
we then sample $20$ models from each of these approximations to construct the final ensemble.

While the training time for MultiSWAG is the same as for deep ensembles, at test
time MultiSWAG is more expensive, as the corresponding ensemble consists of a 
larger number of models. To account for situations when test time is constrained, we also propose
MultiSWA, a method that ensembles independently trained SWA solutions \citep{izmailov2018}.
SWA solutions are the means of the corresponding Gaussian SWAG approximations.
\citet{izmailov2018} argue that SWA solutions approximate the local ensembles represented by 
SWAG with a single model.

In Figure~\ref{fig:ovadia} we show the negative log-likelihood as a function of the number of independently
trained models for a Preactivation ResNet-20 on CIFAR-10 corrupted with 
Gaussian blur with varying levels of intensity (increasing from left to right) in 
Figure \ref{fig:ovadia}. MultiSWAG outperforms deep ensembles significantly on highly corrupted
data.
For lower levels of corruption, MultiSWAG works particularly well when only
a small number of independently trained models are available. We note that MultiSWA also
outperforms deep ensembles, and has the same computational requirements at training 
and test time as deep ensembles.
We present results for other types of corruption in Appendix 
Figures~\ref{fig:app_ovadia}, \ref{fig:app_ovadia_blur}, \ref{fig:app_ovadia_digital}, \ref{fig:app_ovadia_weather},
showing similar trends. In general, there is an extensive evaluation of MultiSWAG in the Appendix.

Our perspective of generalization is deeply connected with Bayesian marginalization. In order to best realize the benefits of 
marginalization in deep learning, we need to consider as many hypotheses as possible through multimodal posterior approximations, 
such as MultiSWAG. In Section~\ref{sec:double_descent} we return to MultiSWAG, showing how it can entirely alleviate prominent 
double descent behaviour, and lead to striking improvements in generalization over SGD and single basin marginalization, for both
accuracy and NLL.

\section{Neural Network Priors}
\label{sec: nnpriors}

A prior over parameters $p(w)$ combines with the functional form of a model $f(x;w)$ to induce a distribution over functions $p(f(x;w))$. It is this distribution over functions
that controls the generalization properties of the model; the prior over parameters, in isolation, has no meaning. Neural networks are imbued with structural properties that
provide good inductive biases, such as translation equivariance, hierarchical representations, and sparsity.  In the sense of Figure~\ref{fig:conceptual}, the prior will have large support, due to the flexibility of neural networks, but its inductive biases provide the most mass to datasets which are representative of problem settings where neural networks are often applied.
In this section, we study the properties of the induced distribution over functions. We directly continue the discussion of priors in Section~\ref{sec: rethinking}, with a focus on examining the noisy CIFAR results in \citet{zhang2016understanding}, from a probabilistic perspective of generalization. These sections are best read together.

We also provide several additional experiments in the Appendix. In Section \ref{sec:app_prior_analysis}, we present analytic results on the
dependence of the prior distribution in function space on the variance of the prior over parameters, considering also layer-wise parameter priors with 
ReLU activations. As part of a discussion on tempering,
in Section~\ref{sec: prioreffect} we study the effect of $\alpha$ in $p(w) = \mathcal{N}(0,\alpha^2 I)$ on prior class probabilities for individual 
sample functions $p(f(x;w))$, the predictive distribution, and posterior samples as we observe varying amounts of data. In Section~\ref{sec:app_correlations}, we further study the correlation structure over images induced by neural network priors, subject to perturbations
of the images. In Section~\ref{sec:app_prior_details} we provide additional experimental details.

\begin{figure*}[h]
	\centering
	
	\subfigure[Prior Draws]{
		\includegraphics[height=0.17\textwidth]{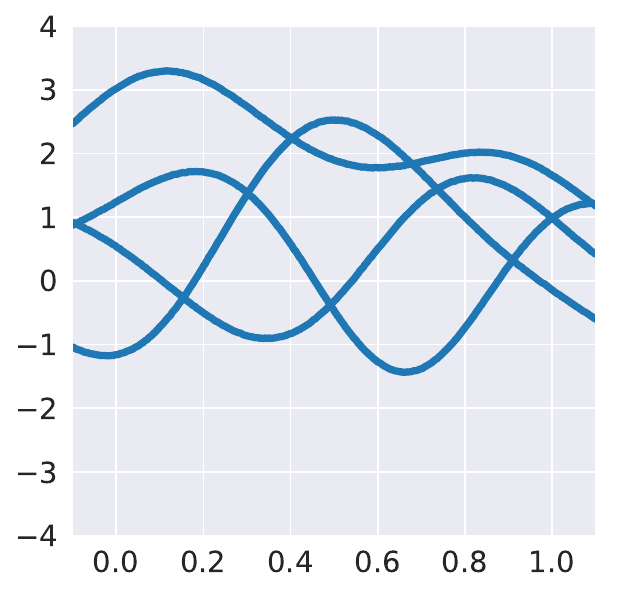}
		\label{fig: gpprior}
    }
    \hspace{-0.1cm}
	\subfigure[True Labels]{
		\includegraphics[height=0.17\textwidth]{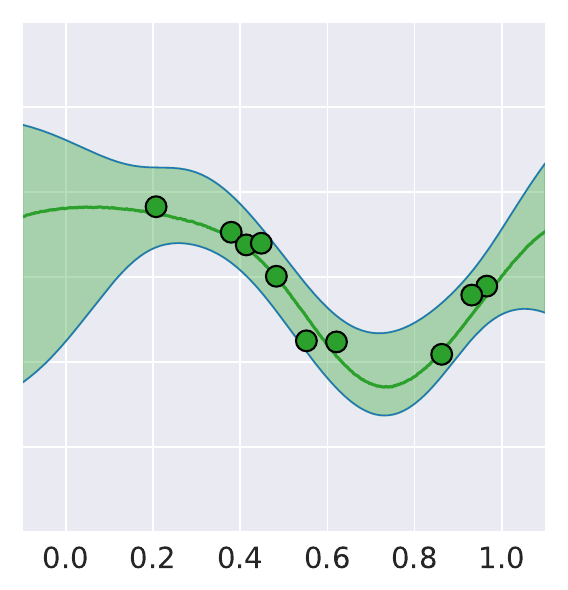}
		\label{fig: gpfit}
    }
    \hspace{-0.1cm}
	\subfigure[Corrupted Labels]{
		\includegraphics[height=0.17\textwidth]{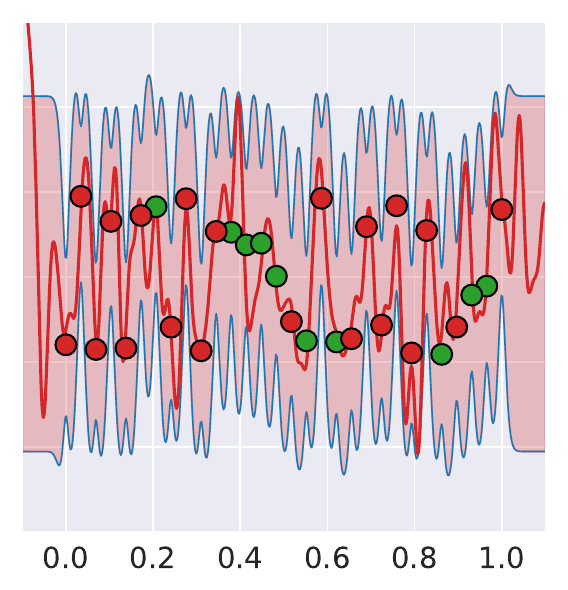}
		\label{fig: gpnoise}
	}
	    \hspace{-0.1cm}
         \subfigure[Gaussian Process]{
        \includegraphics[height=0.17\textwidth]{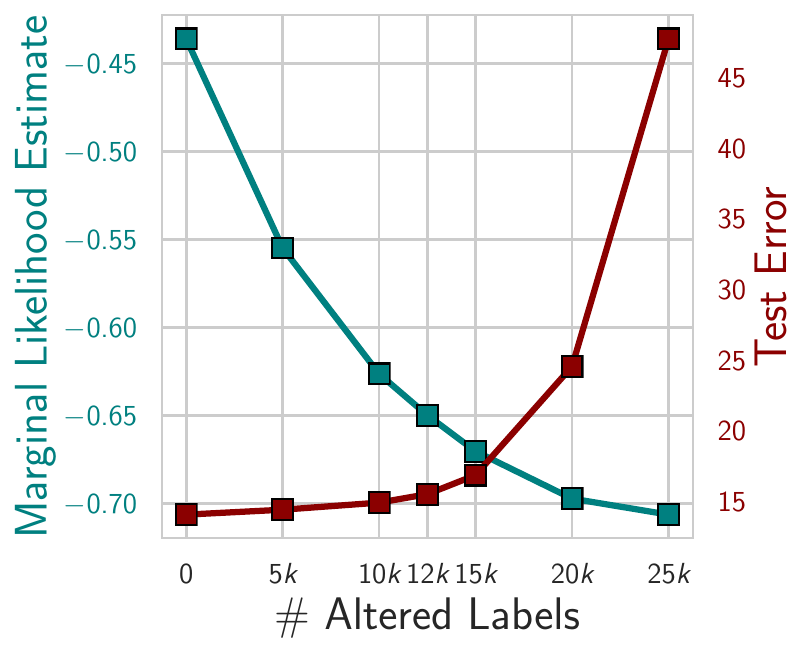}
        \label{fig: gpmarg}
    }
		    \hspace{-0.1cm}
	\subfigure[PreResNet-20]{
        \includegraphics[height=0.17\textwidth]{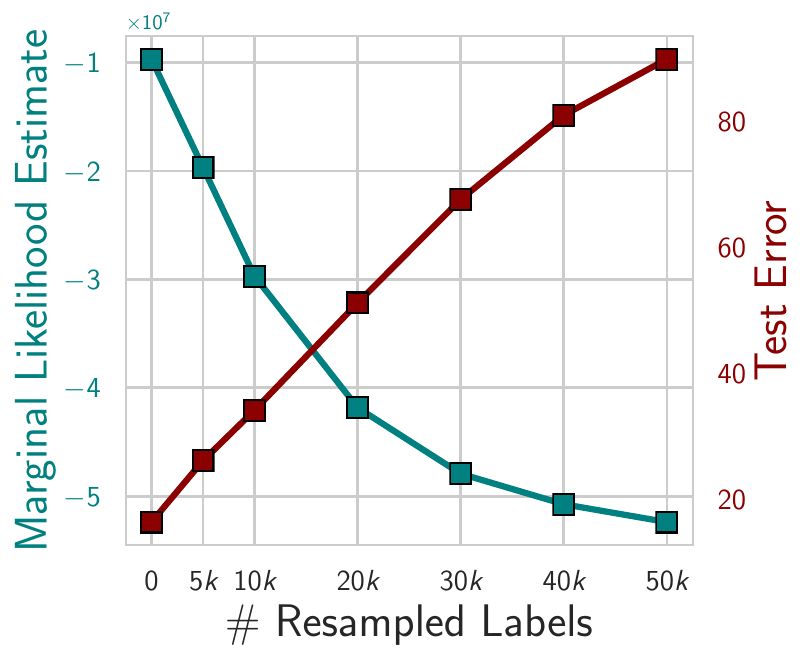}
        \label{fig: bnnmarg}
    }
	\caption{\textbf{Rethinking generalization.} \textbf{(a)}: Sample functions from a Gaussian process
	prior. \textbf{(b)}: GP fit (with 95\% credible region) to structured data generated as 
        $y_{\text{green}}(x) = \sin(x \cdot 2 \pi) + \epsilon,~~ \epsilon \sim \mathcal N(0, 0.2^2)$.
         \textbf{(c)}: GP fit, with no training error, after a significant addition of corrupted data in red, drawn from $\text{Uniform}[0.5, 1]$.
        \textbf{(d)}: Variational GP marginal likelihood with RBF kernel for two classes of CIFAR-10.
        \textbf{(e)}: Laplace BNN marginal likelihood for a PreResNet-20 on CIFAR-10 with different fractions of random labels.
        The marginal likelihood for both the GP and BNN
        decreases as we increase the level of corruption in the labels, suggesting reasonable inductive biases 
        in the prior over functions. Moreover, both the GP and BNN have 100\% training accuracy on images with fully
        corrupted labels.
        }
        \label{fig:gp_reg}
\end{figure*}

\subsection{Deep Image Prior and Random Network Features}
\label{sec: deepimage}

Two recent results provide strong evidence that vague Gaussian priors over parameters, when combined with a neural network architecture, induce a distribution
over functions with useful inductive biases. In the \emph{deep image prior}, \citet{ulyanov2018deep} show that \emph{randomly initialized} 
convolutional neural networks \emph{without training} provide excellent performance for image denoising, super-resolution, and inpainting. This result demonstrates
the ability for a sample function from a random prior over neural networks $p(f(x;w))$ to capture low-level image statistics, before any training. 
Similarly, \citet{zhang2016understanding} shows that pre-processing CIFAR-10 with a \emph{randomly initialized untrained} convolutional neural network 
dramatically improves the test performance of a simple Gaussian kernel on pixels from 54\% accuracy to 71\%. Adding $\ell_2$ regularization only
improves the accuracy by an additional 2\%. These results again indicate that \emph{broad} Gaussian priors over parameters induce reasonable priors over 
networks, with a minor additional gain from decreasing the variance of the prior in parameter space, which corresponds to $\ell_2$ regularization.

\subsection{Prior Class Correlations}
\label{sec:prior_corrs}

In Figure \ref{fig:prior_depnew} we study the prior correlations in the outputs of 
the LeNet-5 convolutional network \citep{lecun1998gradient} on objects of different
MNIST classes.
We sample networks with weights $p(w) = \mathcal{N}(0,\alpha^2 I)$,
and compute the values of logits corresponding to the first class for all pairs
of images and compute correlations of these logits.
For all levels of $\alpha$ the correlations between objects corresponding to the
same class are consistently higher than the correlation between objects of different
classes, showing that the network induces a reasonable prior similarity metric over
these images. Additionally, we observe that the prior correlations somewhat 
decrease as we increase $\alpha$, showing 
that bounding the norm of the weights has some minor utility, in accordance with 
Section~\ref{sec: deepimage}. Similarly, in panel (d) we see that the NLL significantly 
decreases as $\alpha$ increases in $[0,0.5]$, and then slightly increases, 
but is relatively constant thereafter.

In the Appendix, we further describe analytic results and illustrate the effect of $\alpha$ on 
sample functions.

\subsection{Effect of Prior Variance on CIFAR-10}
\label{sec:prior_dep}

We further study the effect of the parameter prior standard deviation $\alpha$, measuring
performance of approximate Bayesian inference for CIFAR-10 with a 
Preactivation ResNet-20 \citep{he2016deep} and VGG-16 \citep{simonyan2014very}.
For each of these architectures we run SWAG \citep{maddoxfast2019} with fixed 
hyper-parameters and varying $\alpha$.
We report the results in Figure \ref{fig:prior_dep}(d), (h).
For both architectures, the performance is near-optimal in the range
$\alpha \in [10^{-2}, 10^{-1}]$. Smaller $\alpha$ constrains
the weights too much. Performance is reasonable and becomes mostly insensitive 
to $\alpha$ as it continues to increase, due to the inductive biases of the functional
form of the neural network.

\section{Rethinking Generalization}
\label{sec: rethinking}

\begin{figure*}[h]
	\centering
	
	\subfigure[True Labels (Err)]{
		\includegraphics[height=0.17\textwidth]{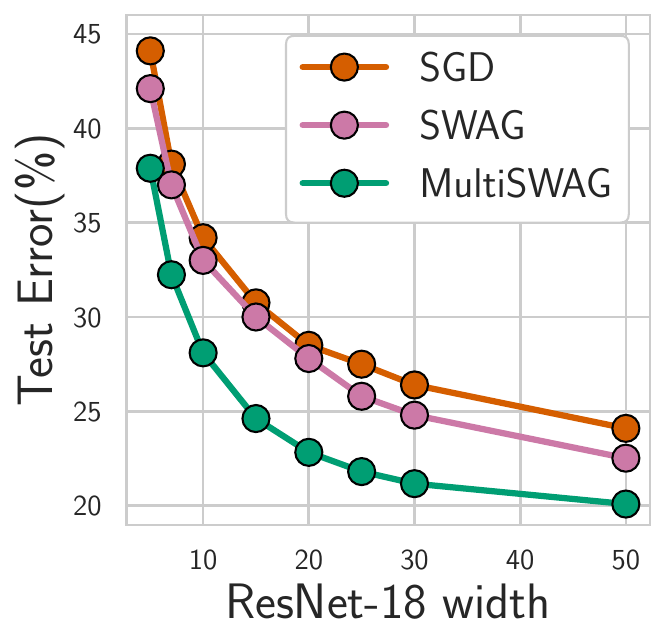}
		\label{fig:dd_acc}
    }
    \hspace{-0.1cm}
	\subfigure[True Labels (NLL)]{
		\includegraphics[height=0.17\textwidth]{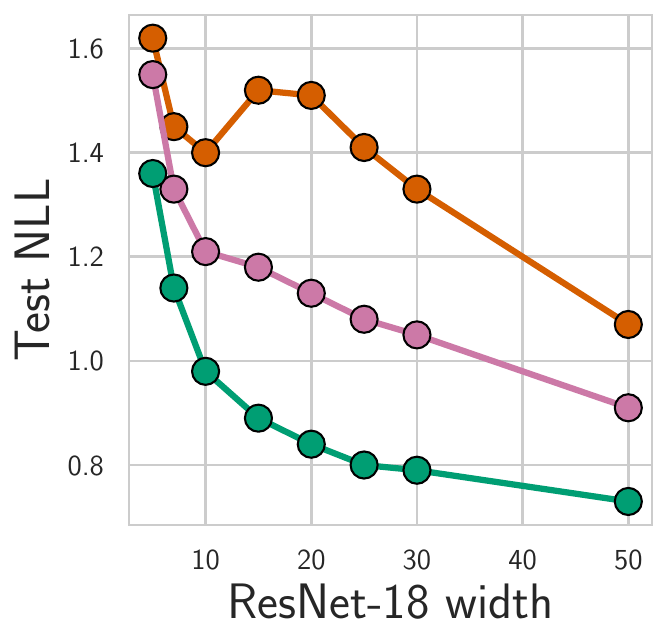}
		\label{fig:dd_nll}
    }
	\hspace{-0.1cm}
    \subfigure[Corrupted (Err)]{
	    \includegraphics[height=0.17\textwidth]{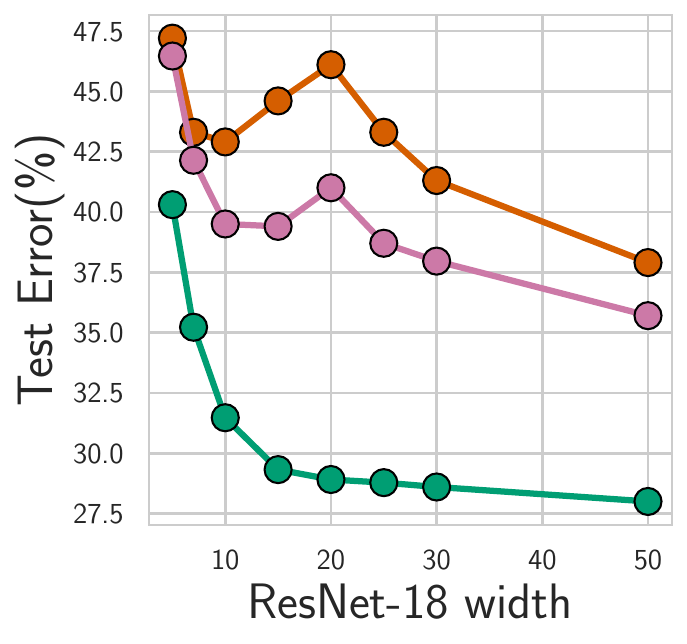}
        \label{fig:dd_corr_acc}
    }
    \hspace{-0.1cm}
	\subfigure[Corrupted (NLL)]{
		\includegraphics[height=0.17\textwidth]{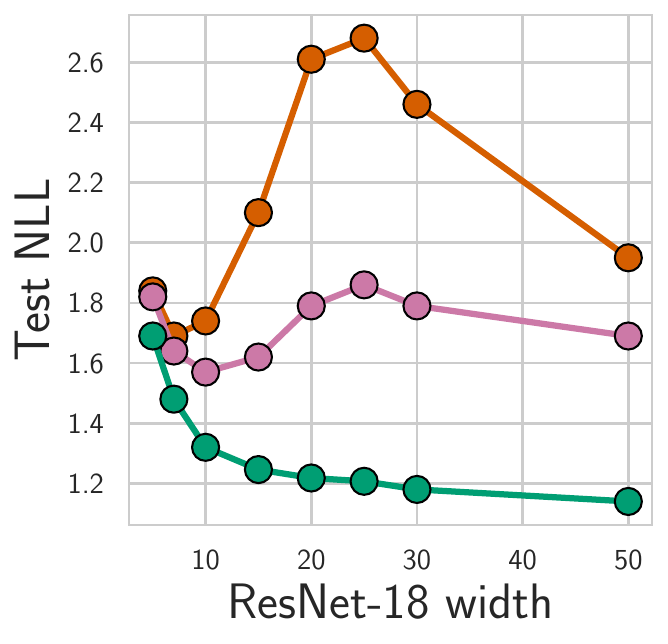}
		\label{fig:dd_corr_nll}
	}
    \hspace{-0.1cm}
	\subfigure[Corrupted (\# Models)]{
		\includegraphics[height=0.17\textwidth]{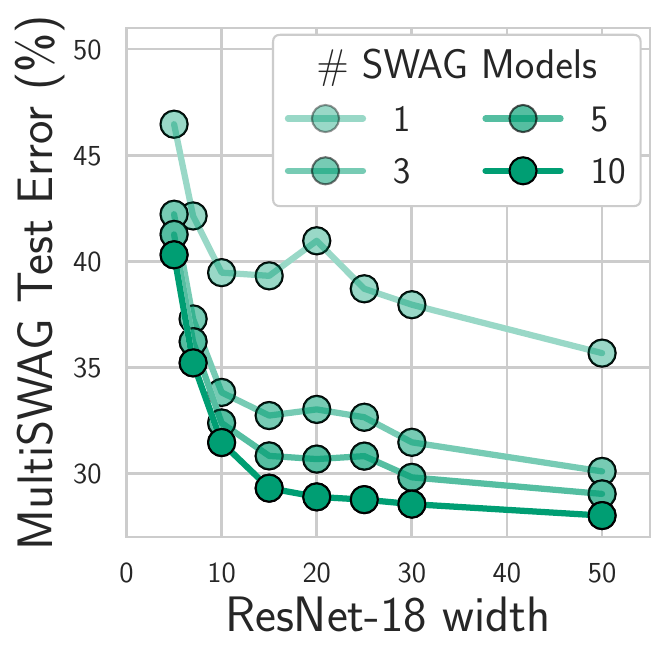}
        \label{fig:dd_multiswag}
    }
	\caption{\textbf{Bayesian model averaging alleviates double descent.} 
    \textbf{(a)}: Test error and \textbf{(b)}: NLL loss for ResNet-18 with varying width on CIFAR-100
    for SGD, SWAG and MultiSWAG.
    \textbf{(c)}: Test error and \textbf{(d)}: NLL loss when 20\% of the labels are randomly reshuffled.
    SWAG reduces double descent, and MultiSWAG, which marginalizes over multiple modes, entirely 
    alleviates double descent both on the original labels and under label
    noise, both in accuracy and NLL. 
    \textbf{(e)}: Test errors for MultiSWAG with varying number of independent SWAG models; 
    error monotonically decreases with increased number of independent models,
    alleviating double descent. We also note that MultiSWAG provides significant improvements in accuracy
    and NLL over SGD and SWAG models.
	See Appendix Figure \ref{fig:app_double_descent} for additional results.
    }
    \label{fig:double_descent}
\end{figure*}

\citet{zhang2016understanding} demonstrated that deep neural networks have sufficient
capacity to fit randomized labels on popular image classification tasks, and suggest this 
result requires re-thinking generalization to understand deep learning.

We argue, however, that this behaviour is not puzzling from a probabilistic perspective, 
is not unique to neural networks,
and cannot be used as evidence against Bayesian neural networks (BNNs) with vague parameter
priors. 
Fundamentally, the resolution is the view
presented in the introduction: from a probabilistic perspective, generalization is at least
a \emph{two-dimensional} concept, related to support (flexibility), which should be as large
as possible, supporting even noisy solutions, 
and inductive biases that represent relative prior probabilities of solutions.

Indeed, we demonstrate that the behaviour in \citet{zhang2016understanding} that was treated
as mysterious and specific to neural networks can be exactly reproduced by Gaussian processes (GPs).
Gaussian processes are an ideal choice for this experiment, because they are popular Bayesian
non-parametric models, and they assign a prior directly in function space. Moreover, GPs have
remarkable flexibility, providing universal approximation with popular covariance functions such
as the RBF kernel. Yet the functions that are a priori \emph{likely} under a GP with an 
RBF kernel are relatively simple. We describe GPs further in the Appendix, and 
\citet{rasmussen06} provides an extensive introduction.

We start with a simple example to illustrate the ability for a GP with an RBF kernel to easily fit
a corrupted dataset, yet generalize well on a non-corrupted dataset, in Figure~\ref{fig:gp_reg}.
In Fig~\ref{fig: gpprior}, we have sample functions from a GP prior over functions $p(f(x))$, showing 
that likely functions under the prior are smooth and well-behaved. In Fig~\ref{fig: gpfit} we see
the GP is able to reasonably fit data from a structured function. And in Fig~\ref{fig: gpnoise} the GP
is also able to fit highly corrupted data, with essentially no structure; although these data are not a 
likely draw from the prior, the GP has support for a wide range of solutions, including noise.

We next show that GPs can replicate the generalization behaviour described in \citet{zhang2016understanding}
(experimental details in the Appendix). When applied to CIFAR-10 images with random labels, 
\emph{Gaussian processes achieve 100\% train accuracy}, and 10.4\% test accuracy (at the level of random guessing). 
However, the same model trained on the true labels  achieves a training accuracy of 72.8\% and a test accuracy of 54.3\%. 
Thus, the generalization behaviour described in \citet{zhang2016understanding} is not unique to neural networks, and can 
be described by separately understanding the support and the inductive biases of a model. 

Indeed, although Gaussian processes support CIFAR-10 images with random labels, they are not likely under the 
GP prior. In Fig~\ref{fig: gpmarg}, we compute the approximate GP marginal likelihood on a binary CIFAR-10 classification 
problem, with labels of varying levels of corruption. We see as the noise in the data increases, the approximate marginal 
likelihood, and thus the prior support for these data, decreases. In Fig~\ref{fig: bnnmarg}, we see a similar trend for a Bayesian
neural network. Again, as the fraction of corrupted labels increases, the approximate marginal likelihood decreases, showing that 
the prior over functions given by the Bayesian neural network has less support for these noisy datasets. We provide further
experimental details in the Appendix.

\citet{dziugaite2017computing} and \citet{smith2017bayesian} provide complementary perspectives on \citet{zhang2016understanding},
for MNIST; \citet{dziugaite2017computing} show non-vacuous PAC-Bayes bounds for the noise-free binary MNIST but not noisy MNIST,
and \citet{smith2017bayesian} show that logistic regression can fit noisy labels on subsampled MNIST, interpreting the results from an Occam factor 
perspective.

\section{Double Descent}
\label{sec:double_descent}

\emph{Double descent} \citep[e.g.,][]{belkin2019reconciling} 
describes generalization error that decreases, increases, and 
then again decreases, with increases in model flexibility. The first
decrease and then increase is referred to as the \emph{classical regime}: 
models with increasing flexibility are increasingly able to capture structure and 
perform better, until they begin to overfit. The next regime is referred to as
the \emph{modern interpolating regime}. The existence of the interpolation regime 
has been presented as mysterious generalization behaviour in deep learning. 

However, our perspective of generalization suggests that performance should
monotonically improve as we increase model flexibility when we use Bayesian 
model averaging with a reasonable prior. Indeed, in the opening example of 
Figure~\ref{fig:airline}, we would in principle want to use the most flexible possible
model. Our results in Section \ref{sec: nnpriors} show that standard BNN priors induce
structured and useful priors in the function space, so we should not expect
double descent in Bayesian deep learning models that perform reasonable 
marginalization.

To test this hypothesis, we evaluate MultiSWAG, SWAG
and standard SGD with ResNet-18 models of varying width, following \citet{nakkiran2019deep},
measuring both error and negative log likelihood (NLL).
For the details, see Appendix \ref{sec:app_details}.
We present the results in Figure \ref{fig:double_descent} and Appendix Figure \ref{fig:app_double_descent}.

First, we observe that models trained with SGD indeed suffer from double descent,
especially when the train labels are partially corrupted (see panels \ref{fig:dd_corr_acc}, \ref{fig:dd_corr_nll}). 
We also see that SWAG, a unimodal posterior approximation, reduces the extent of double descent. Moreover,
MultiSWAG, which performs a more exhaustive \emph{multimodal} Bayesian model average 
\emph{completely mitigates double descent}: the performance of MultiSWAG solutions increases
monotonically with the size of the model, showing no double descent even under significant label corruption,
for both accuracy and NLL. We also found that deep ensembles
follow a similar pattern to MultiSWAG in Figure~\ref{fig:dd_corr_acc}, also mitigating double descent, with slightly
worse accuracy (about 1-2\%). This result is in line with our perspective of Section~\ref{sec: deepensembles} of deep
ensembles providing a better approximation to the Bayesian predictive distribution than conventional
single-basin Bayesian marginalization procedures.

Our results highlight the importance of marginalization over multiple
modes of the posterior: under $20\%$ label corruption SWAG clearly suffers from double
descent while MultiSWAG does not. In Figure \ref{fig:dd_multiswag} we show how the
double descent is alleviated with increased number of independent modes marginalized
in MultiSWAG.

These results also clearly show that MultiSWAG provides significant improvements in
\emph{accuracy} over both SGD and SWAG models, in addition to NLL, an often 
overlooked advantage of Bayesian model averaging we discuss in 
Section~\ref{sec: margimportance}.

Recently, \citet{nakkiran2020optimal} show that carefully tuned $l_2$ regularization can 
help mitigate double descent. Alternatively, we show that Bayesian model averaging, particularly 
based on multimodal marginalization, can mitigate prominent double descent behaviour.
The perspective in Sections \ref{sec: intro} and \ref{sec: bayesianmarginalization}
predicts this result: models with reasonable priors and effective 
Bayesian model averaging should monotonically improve with increases in flexibility.

\section{Temperature Scaling}
\label{sec:app_temperature}

\begin{figure*}
    \def \panelwidth {0.09\textwidth}
	\subfigure[$\alpha=0.01$]{
        \begin{tabular}{c}
		\scriptsize\rotatebox{90}{\quad\quad Prior}~~\rotatebox{90}{\quad~~Sample 1}\quad\includegraphics[height=\panelwidth]{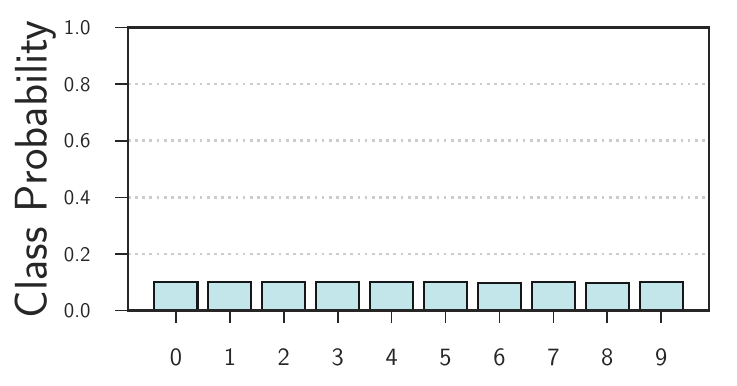}\\
		\scriptsize\rotatebox{90}{\quad\quad Prior}~~\rotatebox{90}{\quad~~Sample 2}\quad\includegraphics[height=\panelwidth]{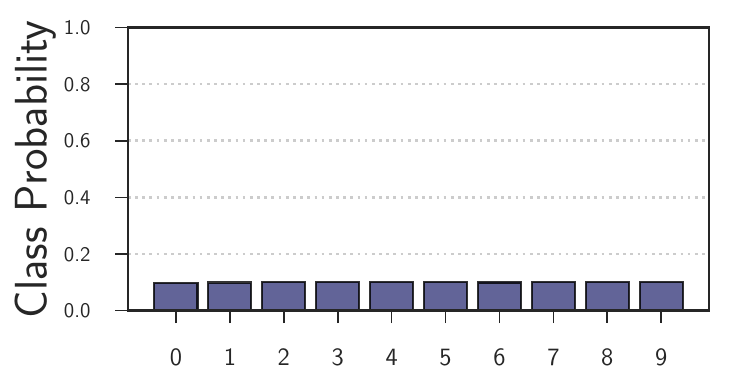}\\
		\scriptsize\rotatebox{90}{\quad\quad Prior}~~\rotatebox{90}{\quad Predictive}\quad\includegraphics[height=\panelwidth]{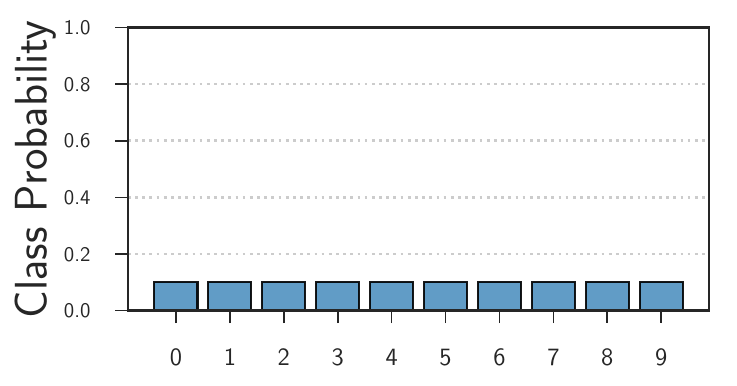}
        \end{tabular}
    }
    \hspace{-0.5cm}
	\subfigure[$\alpha=0.1$]{
        \begin{tabular}{c}
		\includegraphics[height=\panelwidth]{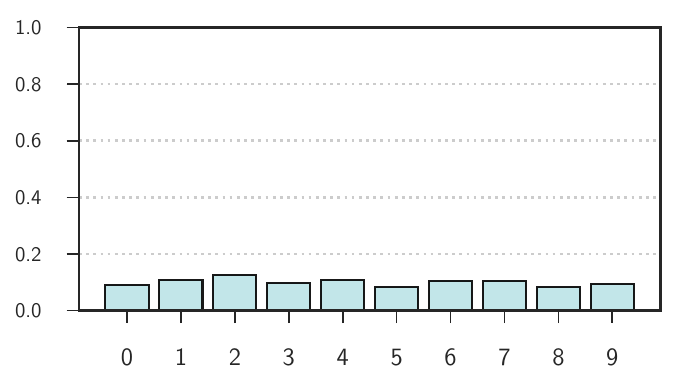}\\
		\includegraphics[height=\panelwidth]{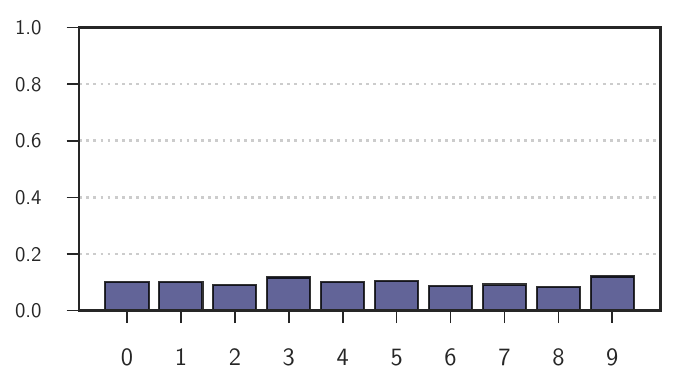}\\
		\includegraphics[height=\panelwidth]{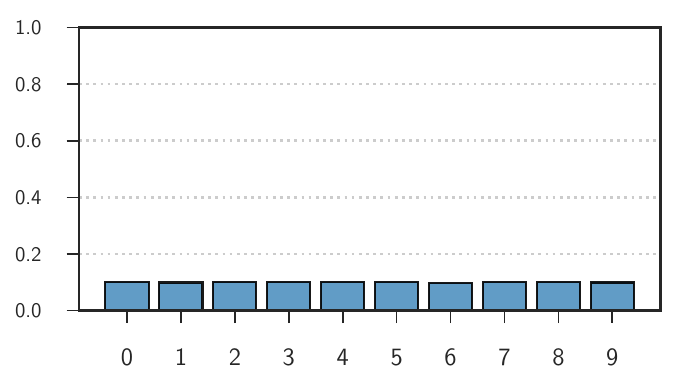}
        \end{tabular}
    }
    \hspace{-0.5cm}
	\subfigure[$\alpha=0.3$]{
        \begin{tabular}{c}
		\includegraphics[height=\panelwidth]{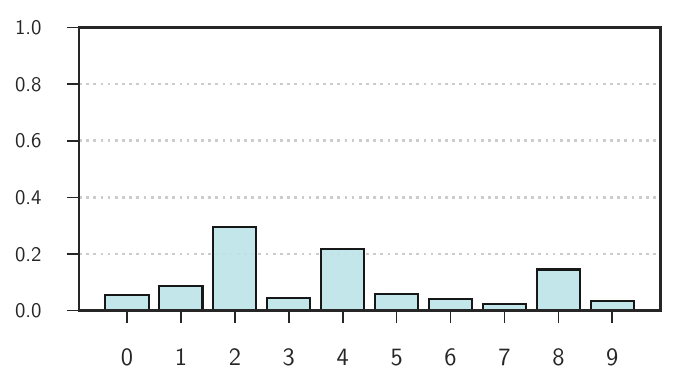}\\
		\includegraphics[height=\panelwidth]{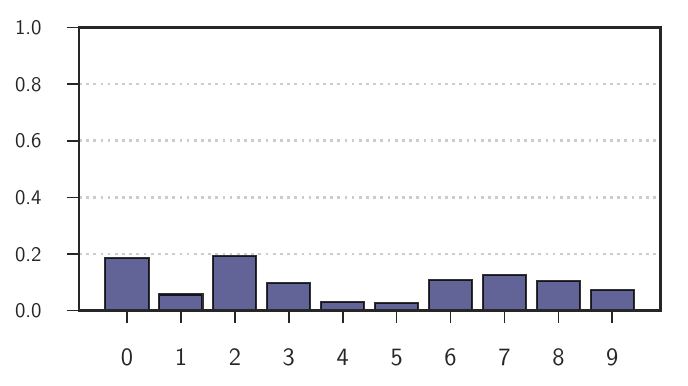}\\
		\includegraphics[height=\panelwidth]{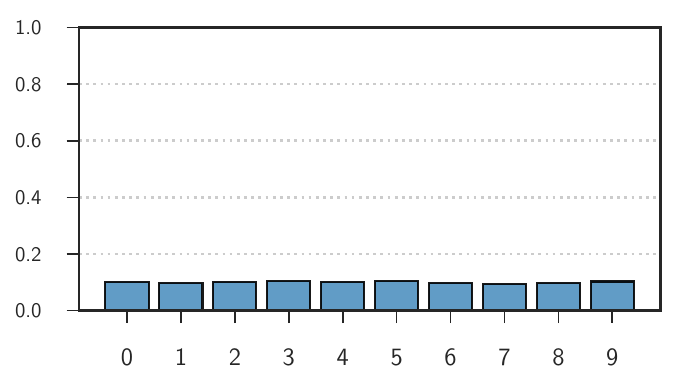}
        \end{tabular}
    }
    \hspace{-0.5cm}
	\subfigure[$\alpha=1.$]{
        \begin{tabular}{c}
		\includegraphics[height=\panelwidth]{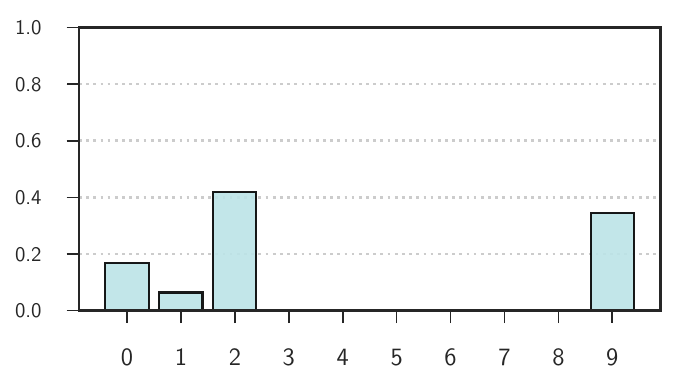}\\
		\includegraphics[height=\panelwidth]{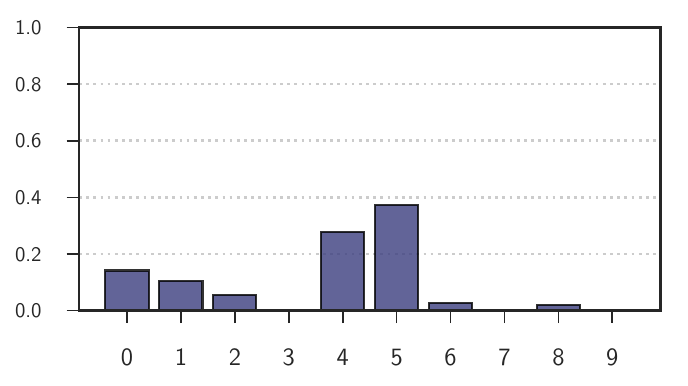}\\
		\includegraphics[height=\panelwidth]{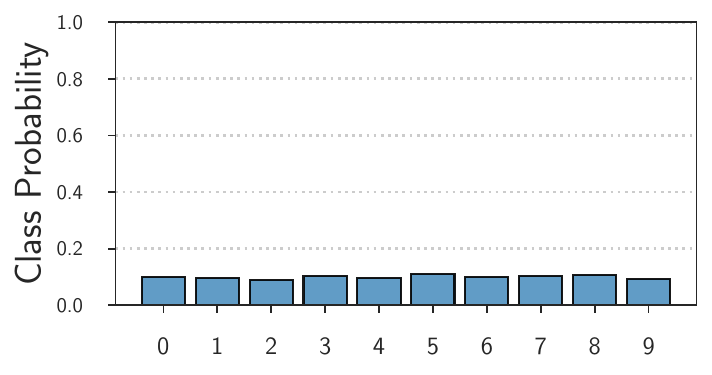}
        \end{tabular}
    }
    \hspace{-0.5cm}
	\subfigure[$\alpha=\sqrt{10}$]{
        \begin{tabular}{c}
		\includegraphics[height=\panelwidth]{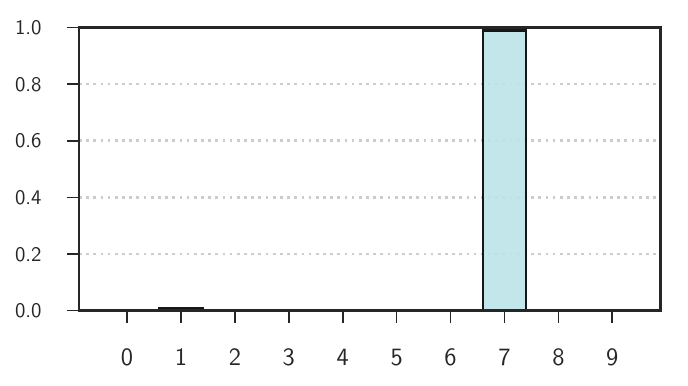}\\
		\includegraphics[height=\panelwidth]{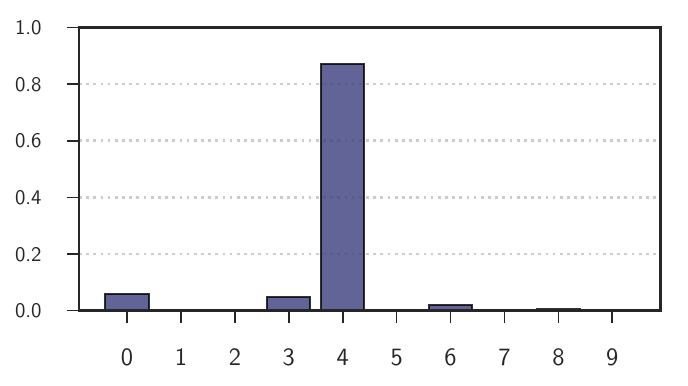}\\
		\includegraphics[height=\panelwidth]{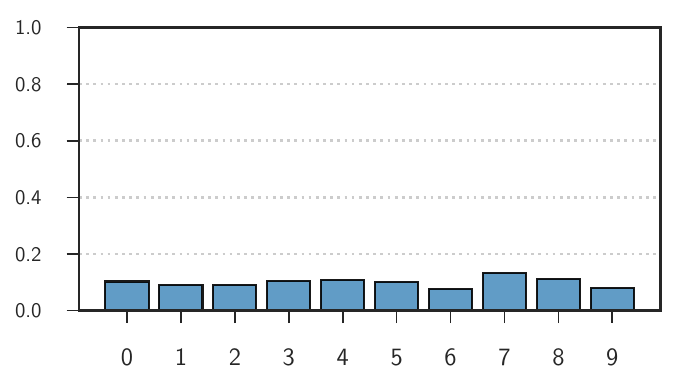}
        \end{tabular}
    }
	\subfigure[]{
		\includegraphics[width=0.23\textwidth]{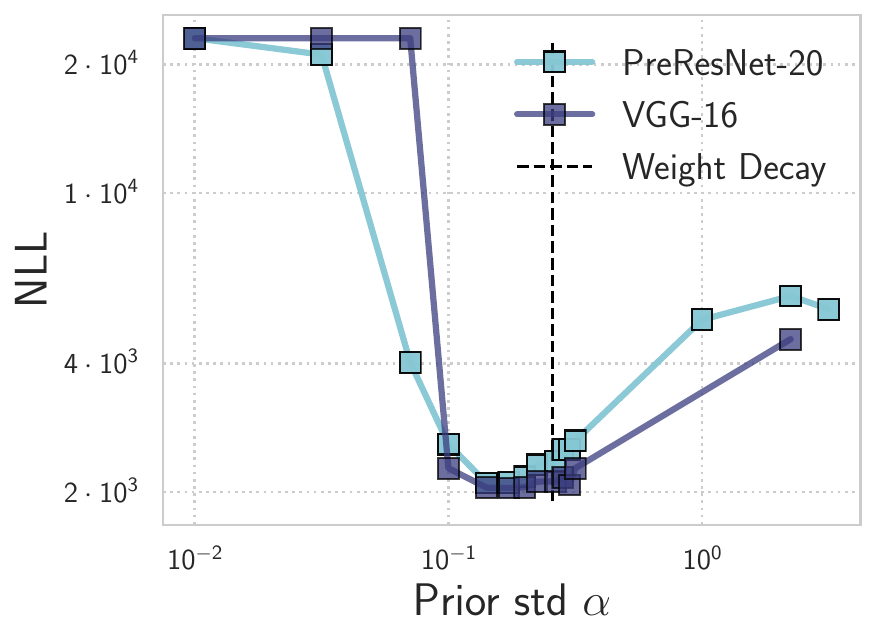}
    }
	\subfigure[]{
		\includegraphics[width=0.23\textwidth]{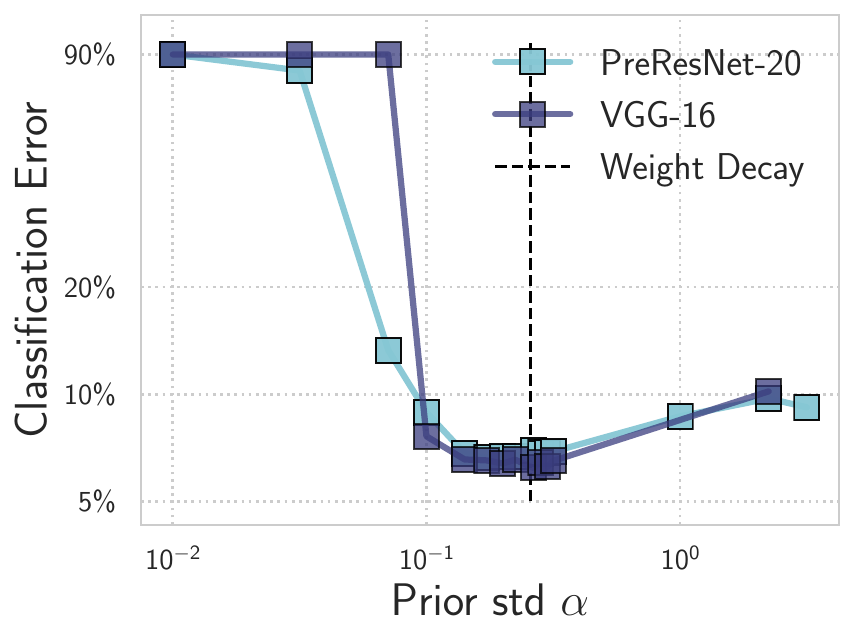}
    }
	\centering
	\caption{\textbf{Effects of the prior variance $\alpha^2$.}
        \textbf{(a)}--\textbf{(e)}: Average class probabilities over all of CIFAR-10 
        for two sample prior functions $p(f(x;w))$ (two top rows) and 
        predictive distribution (average over $200$ samples of weights, bottom row) for varying settings of $\alpha$ in 
        $p(w)=\mathcal{N}(0,\alpha^2 I)$. 
        \textbf{(f)}: NLL and \textbf{(g)} classification error of an ensemble of $20$ SWAG samples on CIFAR-10 as a function of 
        $\alpha$ using a Preactivation ResNet-20 and VGG-16.
        The NLL is high for overly small $\alpha$ and near-optimal in the range of
        $[0.1, 0.3]$. The NLL remains relatively low for vague priors corresponding to 
        large values of $\alpha$.
        }
        \label{fig:priorscale}
\end{figure*}

\begin{figure*}
	\subfigure[Prior ($\alpha = \sqrt{10}$)]{
        \begin{tabular}{c}
		\includegraphics[height=0.09\textwidth]{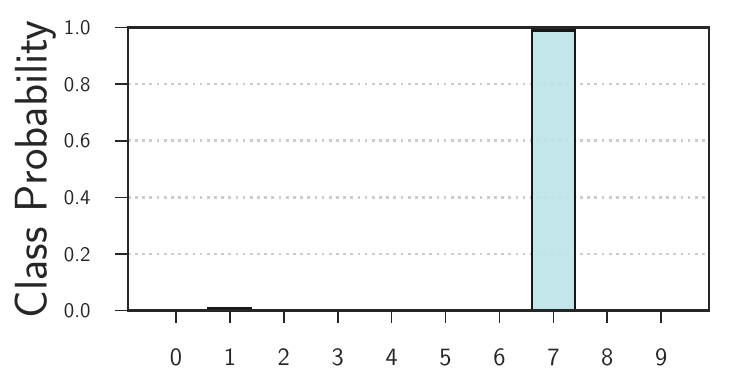}\\
		\includegraphics[height=0.09\textwidth]{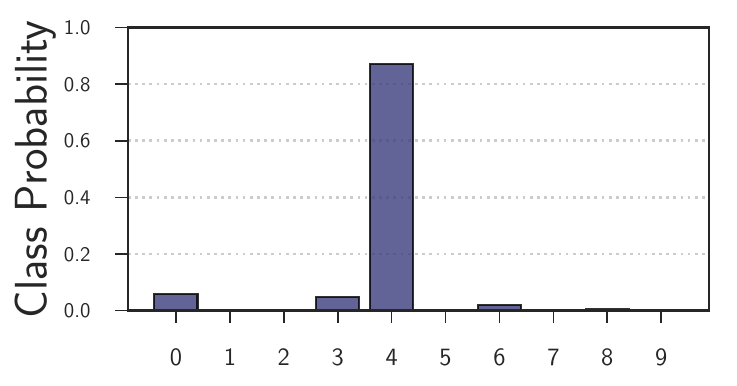}
        \end{tabular}
        \label{fig: priorinit}
    }
	\subfigure[$10$ datapoints]{
        \begin{tabular}{c}
		\includegraphics[height=0.09\textwidth]{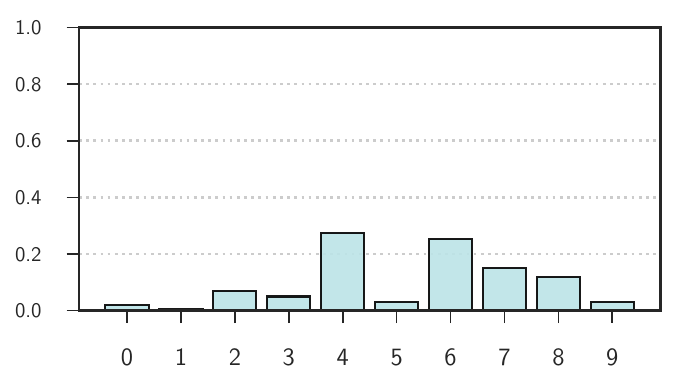}\\
		\includegraphics[height=0.09\textwidth]{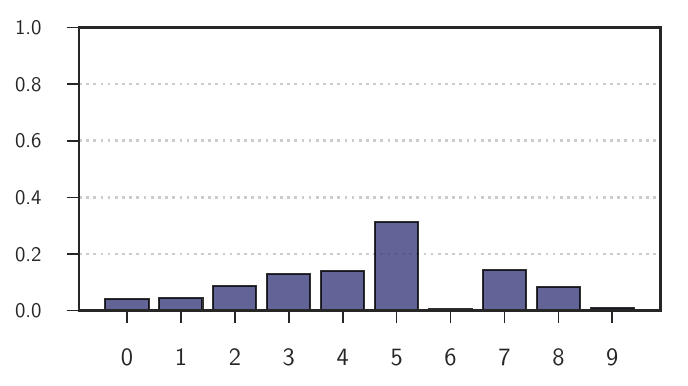}
        \end{tabular}
        \label{fig: priortemp10}
    }
	\subfigure[$100$ datapoints]{
        \begin{tabular}{c}
		\includegraphics[height=0.09\textwidth]{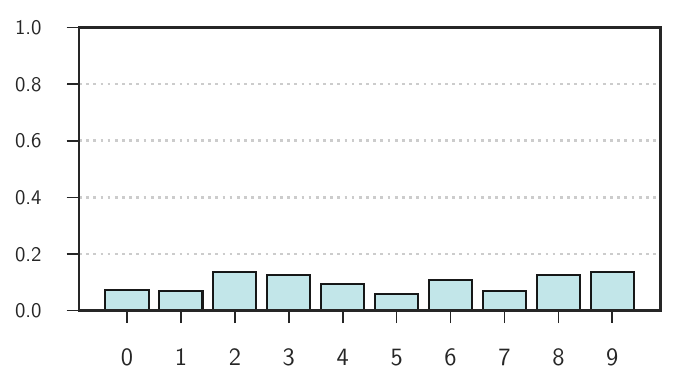}\\
		\includegraphics[height=0.09\textwidth]{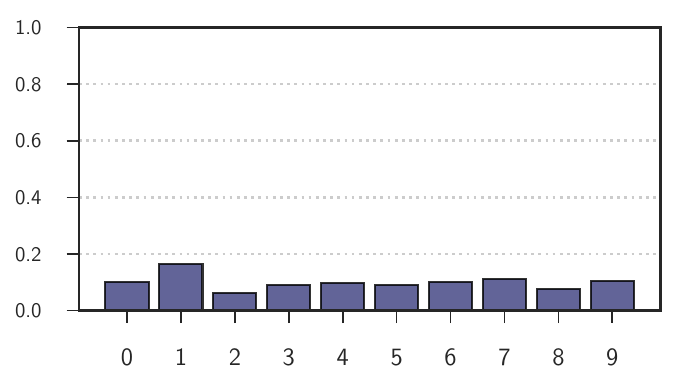}
        \end{tabular}
        \label{fig: priortemp100}
    }
	\subfigure[$1000$ datapoints]{
        \begin{tabular}{c}
		\includegraphics[height=0.09\textwidth]{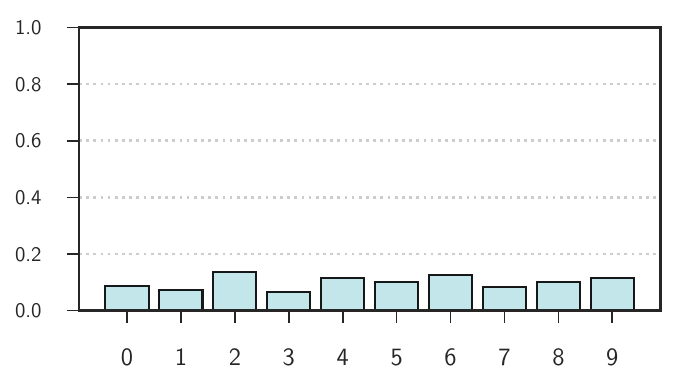}\\
		\includegraphics[height=0.09\textwidth]{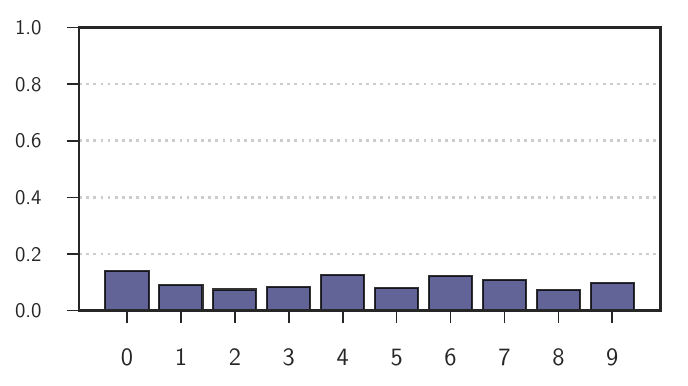}
        \end{tabular}
        \label{fig: priortemp1000}
    }
	\centering
	\caption{\textbf{Adaptivity of posterior variance with data}. We sample two functions $f(x;w)$ from the distribution over functions induced by a distribution
	    over weights, starting with the prior $p(w) = \mathcal{N}(0, 10\cdot I)$, in combination with a PreResNet-20. We measure class probabilities
	    averaged across the CIFAR-10 test set, as we vary the amount of available training data. Although the prior variance is too large, such that the softmax saturates for logits sampled
	    from the prior, leading to one class being favoured, we see that the posterior quickly adapts to correct the scale of the logits in the presence of data. In Figure~\ref{fig:priorscale} we also 
	    show that the prior variance can easily be calibrated such that the prior predictive distribution, even before observing data, is high entropy.
    }
    \label{fig: priordata}
\end{figure*}

The standard Bayesian posterior distribution is given by
\begin{align}
    p(w \vert \mathcal D) = 
    \frac 1 Z p(\mathcal D \vert w) p(w),
\end{align}
where $p(\mathcal{D} | w)$ is a likelihood, $p(w)$ is a prior, and $Z$ is a normalizing constant.

In Bayesian deep learning it is typical to consider the \textit{tempered} posterior
\begin{align}
    p_T(w \vert \mathcal D) = 
    \frac 1 {Z(T)} p(\mathcal D \vert w)^{1/T} p(w),
\end{align}
where $T$ is a \textit{temperature} parameter, and $Z(T)$ is the normalizing
constant corresponding to temperature $T$.
The temperature parameter controls how the prior and likelihood interact in the posterior:
\begin{itemize}
    \item $T < 1$ corresponds to \textit{cold posteriors}, where the posterior
    distribution is more concentrated around solutions with high likelihood.
    \item $T = 1$ corresponds to the standard Bayesian posterior distribution.
    \item $T > 1$ corresponds to \textit{warm posteriors}, where the prior 
    effect is stronger and the posterior collapse is slower.
\end{itemize}
Tempering posteriors is a well-known practice in statistics, where it goes by the names
\emph{Safe Bayes}, \emph{generalized Bayesian inference}, and \emph{fractional
Bayesian inference} \citep[e.g.,][]{de2019safe,grunwald2017inconsistency,barron1991minimum,
walker2001bayesian,zhang2006information,bissiri2016general,grunwald2012safe}. Safe Bayes
has been shown to be natural from a variety of perspectives, including from prequential, learning
theory, and minimum description length frameworks \citep[e.g.,][]{grunwald2017inconsistency}.

Concurrently with our work, \citet{wenzel2020good} noticed that successful Bayesian 
deep learning methods tend to use cold posteriors. They provide an empirical study that shows 
that Bayesian neural networks (BNNs) with cold posteriors outperform models with SGD based 
maximum likelihood training, while BNNs with $T=1$ can perform worse than the 
maximum likelihood solution. They claim that cold posteriors sharply deviate from the 
Bayesian paradigm, and consider possible reasons for why tempering is helpful in Bayesian
deep learning. 

In this section, we provide an alternative view and argue that tempering is not at odds with 
Bayesian principles. Moreover, for virtually any 
realistic model class and dataset, it would be highly surprising if $T=1$ \emph{were} in fact the
best setting of this hyperparameter. Indeed, as long as it is practically convenient, we would advocate 
tempering for essentially \emph{any} model, especially parametric models that do not scale their 
capacity automatically with the amount of available information. Our position is that at a high level
Bayesian methods are trying to combine honest beliefs with data to form a posterior. By reflecting 
the belief that the model is misspecified, the tempered posterior is often more of a \emph{true posterior}
than the posterior that results from ignoring our belief that the model misspecified. 

Finding that $T<1$ helps for Bayesian neural networks is neither surprising nor discouraging.
And the actual results of the experiments in \citet{wenzel2020good}, which show great 
improvements over standard SGD training, are in fact very encouraging of deriving inspiration
from Bayesian procedures in deep learning.

We consider (1) tempering under misspecification (Section~\ref{sec: misspec}); (2) tempering in terms of
overcounting data (Section~\ref{sec: counting}); (3) how tempering compares to changing the observation
model (Section~\ref{sec: temporlike}); (4) the effect of the prior in relation to the experiments of \citet{wenzel2020good}
(Section~\ref{sec: prioreffect}); (5) the effect of approximate inference, including how tempering can help in efficiently 
estimating parameters even for the untempered posterior (Section~\ref{sec: approxinf}).

This section shows how tempering can be a reasonable procedure, and addresses several of 
the points in \citet{wenzel2020good}, particularly on prior misspecification.

Since the original publication of our paper, there have been many papers discussing the cold posterior 
effect. In our follow-up work \citep{izmailov2021bayesian}, we show that there is no cold posterior 
effect in any of the examples of \citet{wenzel2020good} if we remove data augmentation.
In \citet{kapoor2022aleatoric}, we 
show precisely how data augmentation leads to underconfidence in Bayesian classification, and 
how posterior tempering can more naturally reflect our beliefs about aleatoric uncertainty than using $T=1$. 
We also show that the cold posterior effect can be removed in the presence of data augmentation by 
using a Dirichlet observation model, which explicitly enables one to represent aleatoric uncertainty.

\subsection{Tempering Helps with Misspecified Models}
\label{sec: misspec}

Many works explain how tempered posteriors help under model misspecification  
\citep[e.g.,][]{de2019safe,grunwald2017inconsistency,barron1991minimum,
walker2001bayesian,zhang2006information,bissiri2016general,grunwald2012safe}.
In fact, \citet{de2019safe} and \citet{grunwald2017inconsistency} provide several simple 
examples where Bayesian inference fails to provide good convergence behaviour for 
untempered posteriors. While it is easier
to show theoretical results for $T>1$, several of these works also show that $T<1$
can be preferred, even in well-specified settings, 
and indeed recommend learning $T$ from data, for example by cross-validation 
\citep[e.g.,][]{grunwald2012safe, de2019safe}.

\paragraph{Are we in a misspecified setting for Bayesian neural networks?} 
Of course. And it would be irrational to proceed as if it were otherwise. 
Every model is misspecified. In the context of Bayesian neural networks specifically,
the mass of solutions expressed by the prior outside of the datasets we typically consider 
is likely much larger than desired for most applications.
We can calibrate for this discrepancy through tempering. The resulting tempered posterior will be more
in line with our beliefs than pretending the model is not misspecified and finding the untempered posterior.

\emph{Non-parametric models}, such as Gaussian processes, attempt to side-step model misspecification
by growing the number of free parameters (information capacity) automatically with the amount of available
data. In parametric models, we take much more of a manual guess about the model capacity. In the case of 
deep neural networks, this choice is not even close to a \emph{best guess}; it was once the case that 
architectural design was a large component of works involving neural networks, but now it is more standard
practice to choose an off-the-shelf architecture, without much consideration of model capacity.  We do not 
believe that knowingly using a misspecified model to find a posterior is more reasonable (or Bayesian) 
than honestly reflecting the belief that the model is misspecified and then using a tempered posterior. For parametric
models such as neural networks, it is to be expected that the capacity is particularly misspecified.

\subsection{Overcounting Data with Cold Posteriors}
\label{sec: counting}

The criticism of cold posteriors raised by \citet{wenzel2020good} is largely based
on the fact that decreasing temperature leads to overcounting data in the 
posterior distribution.

However, a similar argument can be made against marginal likelihood maximization 
(also known as \textit{empirical Bayes} or \textit{type 2 maximum likelihood}). Indeed,
here, the prior will depend on the same data as the likelihood, which can lead to miscalibrated
predictive distributions \citep{darnieder2011bayesian}.

Nonetheless, empirical Bayes has been embraced and widely adopted in Bayesian machine learning
\citep[e.g.,][]{bishop06, rasmussen06, mackay2003information, minka2001automatic}, as embodying 
several Bayesian principles. Empirical Bayes has been 
particularly embraced in seminal work on Bayesian neural networks \citep[e.g.,][]{mackay1992bayesian, mackay1995probable},
where it has been proposed as a principled approach to learning hyperparameters, such as the scale of 
the variance for the prior over weights, automatically embodying Occam's razor. While there is in this case 
some deviation from the fully Bayesian paradigm, the procedure, which depends on marginalization, is nonetheless 
clearly inspired by Bayesian thinking --- and it is thus helpful to reflect this inspiration and provide understanding of how it 
works from a Bayesian perspective.

There is also work showing the marginal likelihood can lead to miscalibrated Bayes factors under model misspecification. 
Attempts to calibrate these factors \citep{xu2019calibrated}, as part of the Bayesian paradigm, are highly reminiscent of 
work on safe Bayes.

\subsection{Tempered Posterior or Different Likelihood?}
\label{sec: temporlike}

In some cases, the tempered posterior for one model is an untempered posterior using a different likelihood function.
Specifically, consider regression with a Gaussian likelihood and noise variance $\sigma^2$:
\begin{align*}
\begin{split}
  p(y \vert x, w) &= \mathcal N(y \vert f(x, w), \sigma^2) \\ 
  &= \frac{1}{\sqrt{2 \pi \sigma^2}} \cdot \exp\left(-\frac{(y - f(x, w))^2}{2 \sigma^2}\right),
\end{split}
\end{align*}
where $f(x, w)$ is the prediction of the model $f$ with parameters $w$ on the input $x$.
Then, tempering the likelihood, we achieve 
\begin{align*}
\begin{split}
  p(y \vert x, w)^{1/T} 
  &= \frac{1}{\sqrt{2 \pi \sigma^2}^{1/T}} \cdot \exp\left(-\frac{(y - f(x, w))^2}{2 T \sigma^2}\right) \\
  &= \mathcal N(y \vert f(x, w), T \sigma^2) \cdot \sqrt{\frac{(2 \pi T \sigma^{2})}{(2 \pi \sigma^2)^{1 / T}}}\\
  &= \mathcal N(y \vert f(x, w), T \sigma^2) \cdot C,
\end{split}
\end{align*}
where $C$ is a renormalization constant that does not depend on the parameters $w$ of the model.
In this case, the standard Bayesian posterior in the model with noise variance $T \sigma^2$ is equal to the posterior of temperature $T$
in the original model with noise variance $\sigma^2$. Section 4.1 of \citet{grunwald2017inconsistency} considers a related construction.

The predictive distribution differs for the two models; even though the
posteriors coincide, the likelihoods for a new datapoint $y^*$ are different:
\begin{align}
	\label{eq:bma_predictive}
	\int p(y^* \vert w) p(w) dw \ne
	\int p_T(y^* \vert w) p(w) dw \,.
\end{align}
For the Gaussian model described above, the predictions of the tempered model and model
with modified likelihood will have the same mean but different predictive variance.

\citet{wenzel2020good} provide a construction of a likelihood function that is equivalent to tempering
for classification problems.
For a general likelihood function $p(\mathcal D \vert w)$, we can consider a modified likelihood of the form
\begin{align}
  \hat p(\mathcal D \vert w) = p(\mathcal D \vert w)^{1 / T} \cdot C(w),
\end{align}
where $C(w)$ is a renormalization constant that in general depends on the parameters $w$ and inputs $x$, but not the target values $y$.
The standard posterior $\hat p(w \vert \mathcal D)$ in the model with the modified likelihood will then coincide with the 
tempered posterior $p_T(w \vert \mathcal D)$ in the original model up to $C(w)$:
\begin{align}
  \frac{\hat p(w \vert \mathcal D)}{p_T(w \vert \mathcal D)} = C(w).
\end{align}

This subsection has been updated to add a discussion of renormalization, resolving a minor technical point. We however disagree with 
the discussion in \citet{wenzel2020good} on the connection between tempered posteriors and different likelihoods, and believe it misses 
the point. First, in many cases the tempered posterior can be exactly recovered by changing the likelihood on the training data, such as for regression with Gaussian noise. 
Moreover, as above, we can simply introduce a 
renormalization constant that does not depend on the labels $y$, to preserve the equivalence of a tempered posterior
and a modified likelihood in any model.
This renormalization constant can be viewed as a prior over the parameters that we implicitly define with respect to the predictions on the training data;
such a prior can, for example, encode beliefs about how confident the network should be in its predictions on the training data.

Furthermore, as \citet{wenzel2020good} show, the tempered softmax likelihood with $T<1$
can be viewed as a valid likelihood if we introduce a new class, not observed in the training data. While they discard that particular
interpretation, it is not unreasonable to include an unobserved class, since our observation model may want to recognize
that we have not observed all possible classes, and therefore retain an additional label.
This extra class can, for example, correspond to all the possible images that do not belong to any of the classes in our dataset.
Finally, new work \citep{kapoor2022aleatoric}
shows that we can naturally interpret the tempered likelihood as using the multinomial observation model, assuming $1 / T$ counts of the label are observed for each of the training datapoints, 
which is perfectly valid.

We present other key considerations in the discussion of tempering for Bayes posteriors in other 
parts of this section.

\vspace{-2mm}
\subsection{Effect of the Prior}
\label{sec: prioreffect}

While a somewhat misspecified prior will certainly interact with the utility of tempering, we do not 
believe the experiments in \citet{wenzel2020good} provide evidence that even the prior $p(w) = \mathcal{N}(0,I)$ 
is misspecified to any serious extent. For a relatively wide range of distributions over $w$, the functional form of 
the network $f(x;w)$ can produce a generally reasonable distribution over functions $p(f(x;w))$. In Figure~\ref{fig: priordata}, we reproduce the findings in \citet{wenzel2020good} showing sample functions $p(f(x;w))$ 
corresponding to the prior $p(w) = \mathcal{N}(0,10\cdot I)$ strongly favour a single class over the dataset. While this
behaviour appears superficially dramatic, we note it is simply an artifact a miscalibrated signal variance. A 
miscalibrated signal variance interacts with a quickly saturating soft-max link function to provide a seemingly dramatic
preference to a given class. If we instead use $p(w) = \mathcal{N}(0,\alpha^2 I)$, for quite a range of $\alpha$, then 
sample functions provide reasonably high entropy across labels averaged over the dataset, as in Figure~\ref{fig:priorscale}.
For individual points the posterior samples have different particular class preferences. $\alpha$
can be easily determined through cross-validation, as in Figure~\ref{fig:priorscale}, or specified as a standard value used
for $L_2$ regularization ($\alpha = 0.24$ in this case).

However, even with the inappropriate prior scale, we see in panels (a)--(e) of Figure \ref{fig:priorscale} that the unconditional predictive distribution \emph{is}
completely reasonable. Moreover, the prior variance represents a \emph{soft} prior bias, and will quickly update with data. In Figure~\ref{fig: priordata} we show posterior samples after observing $10, 100,$ and $1000$ data points.

Other aspects of the prior, outside of the prior signal variance, will have a much greater effect on the inductive biases of the model.
For example, the induced covariance function $\text{cov}(f(x_i,w), f(x_j,w))$ reflects the induced similarity metric over data instances; 
through the covariance function we can answer, for instance, whether the model believes a priori that a translated image is similar to the 
original. Unlike the signal variance of the prior, the prior covariance function will continue to have a significant effect on posterior inference
for even very large datasets, and strongly reflects the structural properties of the neural network. We explore these structures of the prior 
in Figure~\ref{fig:prior_dep}.

\subsection{The Effect of Inexact Inference}
\label{sec: approxinf}

We have to keep in mind what we ultimately use posterior samples to compute. Ultimately,
we wish to estimate the predictive distribution given by the integral in Equation~\eqref{eqn: bma}.
With a finite number of samples, the tempered posterior could be used to provide a better approximation
to the expectation of the predictive distribution associated with untempered posterior. 

Consider a simple example, where we wish to estimate the mean of a high-dimensional Gaussian distribution
$\mathcal{N}(0,I)$. Suppose we use $J$ independent samples. The mean of these samples is also Gaussian
distributed, $\mu \sim \mathcal{N}(0, \frac{1}{J} I)$. In Bayesian deep learning, the dimension $d$ is typically
on the order $10^7$, and $J$ would be on the order of $10$. The norm of $\mu$ would be highly concentrated
around $\frac{\sqrt{10^7}}{\sqrt{10}} = 1000$. In this case, sampling from a tempered posterior with $T<1$ would
lead to a better approximation of the Bayesian model average associated with an untempered posterior.

Furthermore, no current sampling procedure will be providing samples that are close to 
independent samples from the true posterior of a Bayesian neural network. The posterior landscape is far too 
multimodal and complex for there to be any reasonable coverage. The approximations we have are practically
useful, and often preferable to conventional training, but we cannot realistically proceed with analysis assuming 
that we have obtained true samples from a posterior. While we would expect that some value of $T \ne 1$ would be 
preferred for any finite dataset in practice, it is conceivable that some of the results in 
\citet{wenzel2020good} may be affected by the specifics of the approximate inference
technique being used.

We should be wary not to view Bayesian model averaging purely through the prism of simple Monte Carlo, as 
advised in Section~\ref{sec: beyondmc}. Given a finite computational budget, our goal in effectively approximating
a Bayesian model average is \emph{not} equivalent to obtaining good samples from the posterior.

\section{Discussion}

\begin{quotation}
\emph{``It is now common practice for Bayesians to fit models that have more parameters than the number of data points\dots Incorporate every imaginable possibility into the model space: for example, if it is conceivable that a very simple model might be able to explain the data, one should include simple models; if the noise might have a long-tailed distribution, one should include a hyperparameter which controls the heaviness of the tails of the distribution; if an input variable might be irrelevant to a regression, include it in the regression anyway.''} \citet{mackay1995probable} 
\end{quotation}

We have presented a probabilistic perspective of generalization, which depends on the support and inductive biases of the model. The support should be as large possible, but the inductive biases must be well-calibrated to a given problem class. We argue that Bayesian neural networks embody these properties --- and through the lens of probabilistic inference, explain generalization behaviour that has previously been viewed as mysterious. Moreover, we argue that Bayesian marginalization is particularly compelling for neural networks, show how deep ensembles provide a practical mechanism for marginalization, and propose a new approach that generalizes deep ensembles to marginalize within basins of attraction. We show that this multimodal approach to Bayesian model averaging, MultiSWAG, can entirely alleviate double descent, to enable monotonic performance improvements with increases in model flexibility, as well significant improvements in generalization accuracy and log likelihood over SGD and single basin marginalization.

There are certainly many challenges to estimating the integral for a Bayesian model average in modern deep learning, including a high-dimensional parameter space, and a complex posterior landscape. But viewing the challenge indeed as an integration problem, rather than an attempt to obtain posterior samples for a simple Monte Carlo approximation, provides opportunities for future progress. Bayesian deep learning has been making fast practical advances, with approaches that now enable better accuracy and calibration over standard training, with minimal overhead. 

We finish with remarks about future developments for Bayesian neural network priors, and approaches to research in Bayesian deep learning.

\subsection{The Future for BNN Priors}
\label{sec:app_future}

We provide some brief remarks about future developments for BNN priors. Here we have explored relatively simple parameter
priors $p(w) = \mathcal{N}(0,\alpha^2 I)$. While these priors are simple in parameter space, they interact with the neural network architecture 
to induce a sophisticated prior over functions $p(f(x;w))$, with many desirable properties, including a reasonable correlation 
structure over images. However, these parameter priors can certainly still be 
improved. As we have seen, even tuning the value of the signal variance $\alpha^2$, an analogue of the $L_2$ regularization often used in deep learning, 
can have a noticeable affect on the induced prior over functions --- though this affect is quickly modulated by data. 
Layer-wise priors, such that parameters in each layer have a different
signal variance, are intuitive: we would expect later layers require precise determination, while parameters in earlier layers could reasonably
take a range of values. But one has to be cautious; as we show in Appendix Section~\ref{sec:app_prior_analysis}, with ReLU activations 
different signal variances in different layers can be degenerate, combining together to affect only the output scale of the network. 

A currently popular sentiment is that we should directly build function-space BNN priors, often taking inspiration from Gaussian
processes. While we believe this is a promising direction, one should proceed with caution. If we contrive priors over parameters
$p(w)$ to induce distributions over functions $p(f)$ that resemble familiar models such as Gaussian processes with RBF kernels, 
we could be throwing the baby out with the bathwater. Neural networks are useful as their own model class precisely because they 
have different inductive biases from other models. 

A similar concern applies to taking infinite width limits in Bayesian neural networks.
In these cases we recover Gaussian processes with interpretable kernel functions; because these models are easier to use and analyze,
and give rise to interpretable and well-motivated priors, it is tempting to treat them as drop-in replacements for the parametric analogues. 
However, the kernels for these models are \emph{fixed}. In order for 
a model to do effective representation learning, we must learn a similarity metric for the data. Training a neural network in many ways is
like \emph{learning} a kernel, rather than using a fixed kernel. \citet{mackay98} has also expressed concerns in treating these limits as 
replacements for neural networks, due to the loss of representation learning power.

Perhaps the distribution over functions induced by a network in combination with a 
generic distribution over parameters $p(w)$ may be hard to interpret --- but this distribution will contain the equivariance properties,
representation learning abilities, and other biases that make neural networks a compelling model class in their own right.

\subsection{``But is it \emph{really} Bayesian?''}
\label{sec:app_note}

We finish with an editorial comment about approaches to research within Bayesian deep 
learning. There is sometimes a tendency to classify work as \emph{Bayesian} or \emph{not Bayesian},
with very stringent criteria for what qualifies as \emph{Bayesian}. Moreover, the implication, and sometimes
even explicit recommendation, is that if an approach is not unequivocally Bayesian in every respect, then we 
should not term it as Bayesian, and we should instead attempt to understand the procedure through entirely 
different non-Bayesian mechanisms. We believe this mentality encourages tribalism, 
which is not conducive to the best research, or creating the best performing methods. 
What matters is not a debate about
semantics,
but making rational modelling choices given a particular problem setting, and trying to understand
these choices. Often these choices can largely be inspired by a Bayesian approach --- in which case it desirable to 
indicate this source of inspiration. And in the semantics debate, who would be the arbiter of what gets to be called 
Bayesian? Arguably it ought to be an evolving definition.

Broadly speaking, what makes Bayesian approaches distinctive is a posterior weighted marginalization over
parameters. And at a high level, Bayesian methods are about combining our honest beliefs with data to form a posterior.
In actuality, no fair-minded researcher entirely believes the prior over parameters, the functional 
form of the model (which is part of the prior over functions), or the likelihood. From this perspective, it is broadly 
compatible with a Bayesian philosophy to reflect misspecification in the modelling procedure itself, which is achieved through
tempering. In this sense, the \emph{tempered posterior} is more reflective of a \emph{true posterior} than the posterior that
results from ignoring our belief that the model is misspecified.

Moreover, basic probability theory indicates that marginalization is desirable. 
While marginalization cannot
in practice be achieved exactly, we can try to improve over conventional training, which as we have discussed can be
viewed as approximate marginalization. Given computational constraints, effective marginalization is not equivalent
to obtaining accurate samples from a posterior. As we have discussed, simple Monte Carlo is only one 
of many mechanisms for marginalization. Just like we how expectation propagation \citep{minka01} focuses its approximation to 
factors in a posterior where it will most affect the end result, we should focus on representing the posterior where 
it will make the biggest difference to the model average. As we have shown, deep ensembles are a reasonable mechanism
up to a point. After having trained many independent models, there are added benefits to marginalizing within basins, given
the computational expense associated with retraining an additional model to find an additional basin of attraction.

We should also not hold Bayesian methods to a double standard. Indeed, it can be hard to interpret or understand the prior,
the posterior, and whether the marginalization procedure is optimal. But it is also hard to interpret the choices behind the 
functional form of the model, or the rationale behind classical procedures where we bet everything on a single global optimum 
--- when we know
there are many global optima and many of them will perform well but provide different solutions, and many others will 
not perform well. We should apply the same level of scrutiny to all 
modelling choices, consider the alternatives, and not be paralyzed if a procedure is not optimal in every respect.

\paragraph{Acknowledgements.} AGW and PI are supported by an Amazon Research Award, Facebook Research, 
NSF I-DISRE 193471, NIH R01 DA048764-01A1, NSF IIS-1563887, and NSF IIS-1910266. We thank Greg Benton
for helpful discussions.


\bibliography{mbibnew}
\bibliographystyle{icml2020}

\newpage\null

\appendix

\renewcommand{\floatpagefraction}{.8}

\begin{figure*}
    \begin{tabular}{ccccc}
	\includegraphics[height=0.16\textwidth]{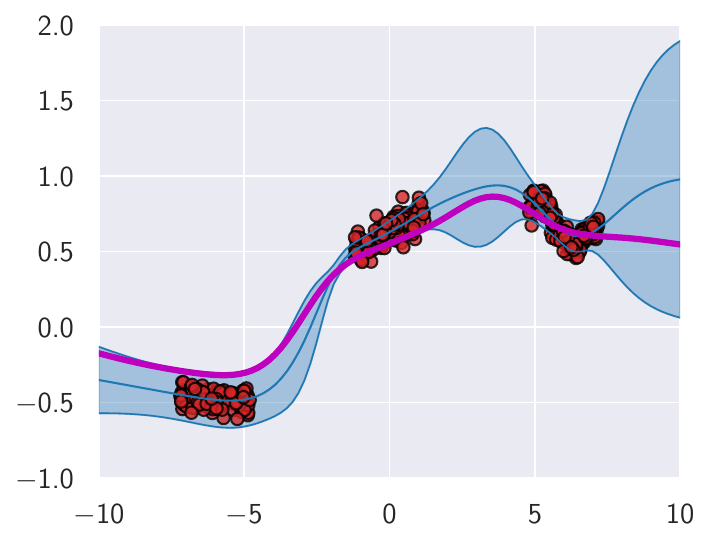} &
	\includegraphics[height=0.16\textwidth]{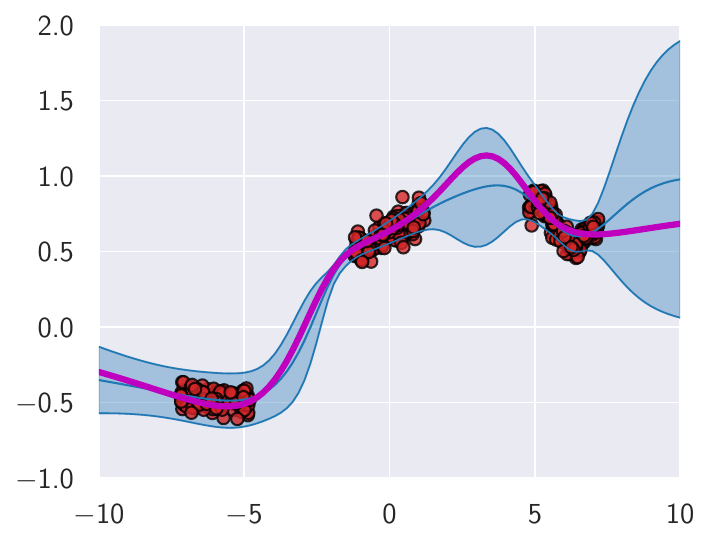} &
	\includegraphics[height=0.16\textwidth]{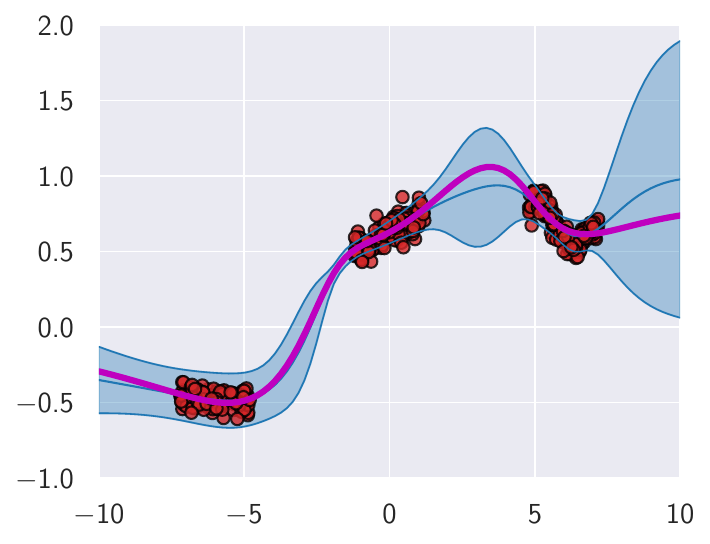} &
	\includegraphics[height=0.16\textwidth]{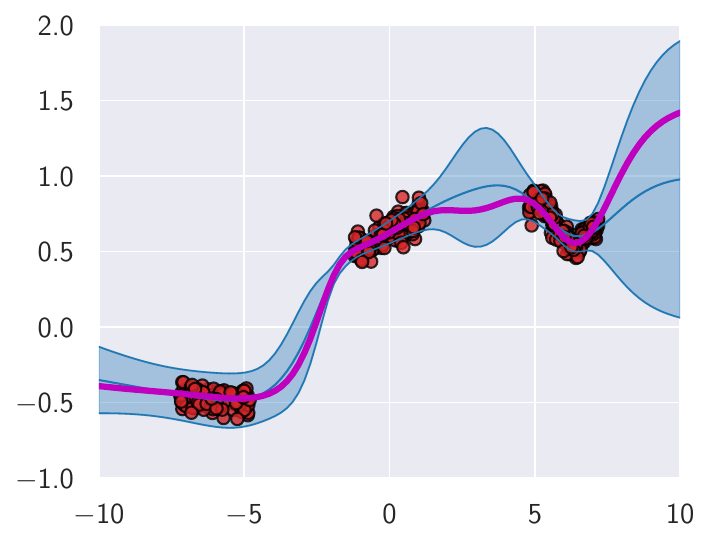} & \\
	\includegraphics[height=0.16\textwidth]{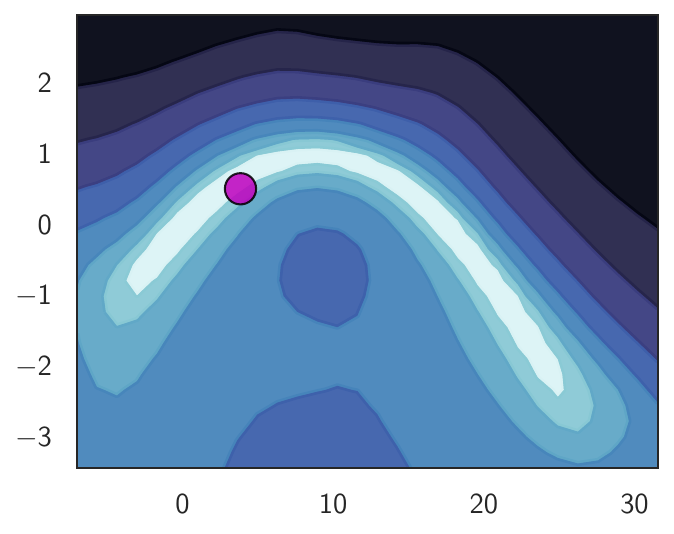} &
	\includegraphics[height=0.16\textwidth]{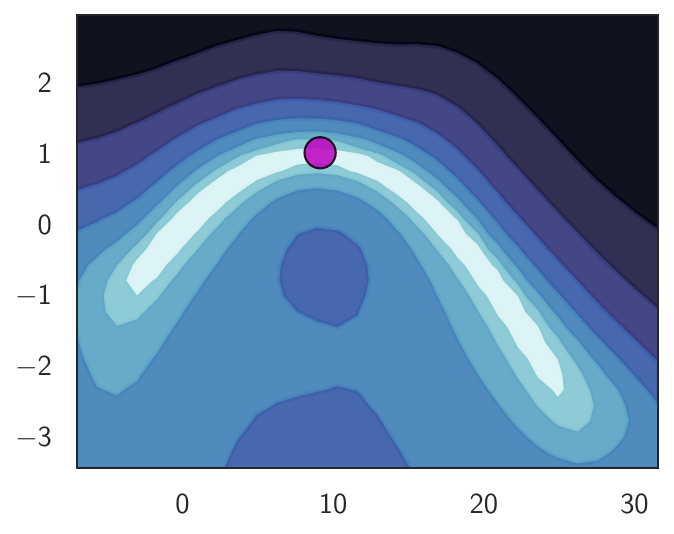} &
	\includegraphics[height=0.16\textwidth]{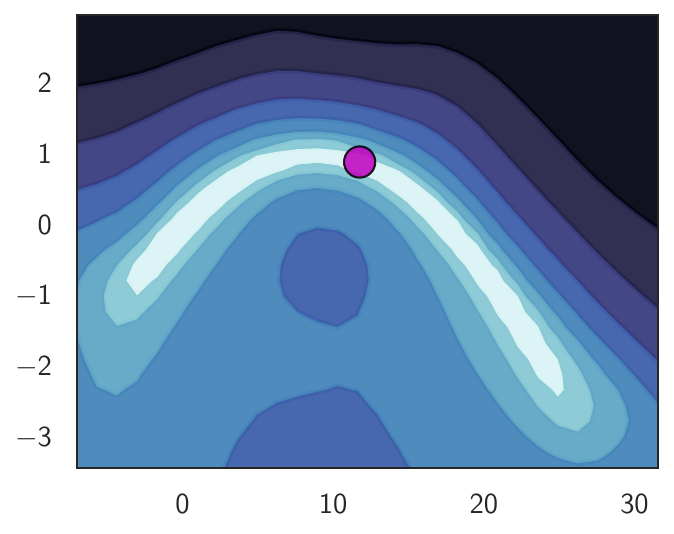} &
	\includegraphics[height=0.16\textwidth]{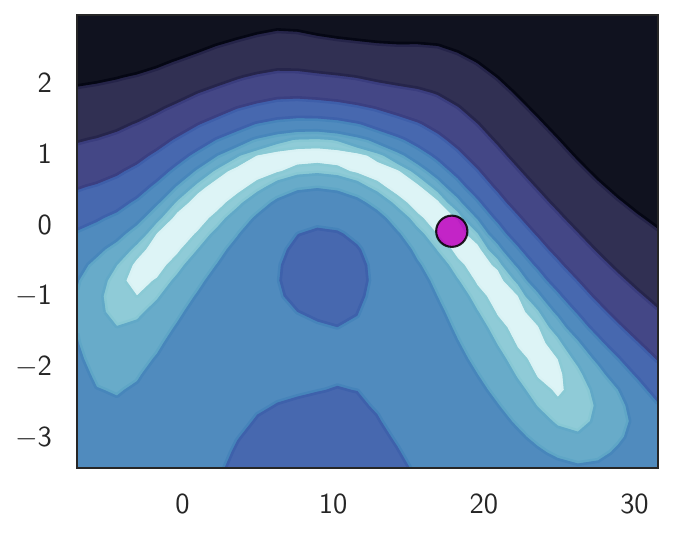} &
    \hspace{-0.5cm}
	\includegraphics[height=0.16\textwidth]{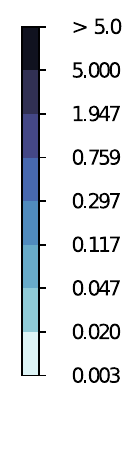}
    \end{tabular}
	\centering
	\caption{\textbf{Diversity of high performing functions.}
            \textbf{Bottom}: a contour plot
            of the posterior log-density in the subspace containing a pair of 
            independently trained modes (as with deep ensembles), and a path of 
            high posterior density connecting these modes.
            In each panel, the purple point represents a sample from the posterior in the 
            parameter subspace.
            \textbf{Top}: the predictive distribution constructed from samples
            in the subspace. 
            The shaded blue area shows the $3\sigma$-region of the predictive
            distribution at each of the input locations, and the blue line 
            shows the mean of the predictive distribution.
            In each panel, the purple line shows the predictive
            function corresponding to the sample shown in the corresponding
            bottom row panel.
			For the details of the experimental setup see Section 5.1 of 
			\citet{izmailov2019subspace}. 
        }
        \label{fig:loss_valleys}
\end{figure*}

\section*{Appendix Outline}

This appendix is organized as follows.
In Section \ref{sec:app_loss}, we visualize predictive functions corresponding
to weight samples within high posterior density valleys on a regression problem.
In Section \ref{sec: gps}, we provide background material on Gaussian processes.
In Section \ref{sec:appendix_ovadia}, we present further results comparing MultiSWAG
and MultiSWA to Deep Ensembles under data distribution shift on CIFAR-10.
In Section \ref{sec:app_details}, we provide the details of all experiments
presented in the paper.
In Section \ref{sec:app_prior_analysis}, we present analytic results on the
dependence of the prior distribution in function space on the variance of 
the prior over parameters.
In Section \ref{sec:app_correlations}, we study the prior correlations between
BNN logits on perturbed images.

\section{Loss Valleys}
\label{sec:app_loss}

We demonstrate that different points along the valleys of high posterior density
(low loss) connecting pairs of independently trained optima \citep{garipov2018,draxler2018essentially,fort2019large}
correspond to different predictive functions.
We use the regression example from \citet{izmailov2019subspace} and show the
results in Figure \ref{fig:loss_valleys}.

\section{Gaussian processes}
\label{sec: gps}

With a Bayesian neural network, a distribution over parameters $p(w)$ induces a distribution over functions
$p(f(x;w))$ when combined with the functional form of the network. Gaussian processes (GPs) are often used to 
instead \emph{directly} specify a distribution over functions. 

A Gaussian process is a distribution over functions, $f(x) \sim \mathcal{GP}(m,k)$, such that any collection of function values,
queried at any finite set of inputs $x_1, \ldots, x_n$, has a joint Gaussian distribution:
\begin{align}
f(x_1), \ldots, f(x_n) \sim \mathcal N(\mu, K) \,.
\label{eqn:gp}
\end{align}
The mean vector, $\mu_i = \mathbb{E}[f(x_i)] = m(x_i)$, and covariance matrix, $K_{ij} = \text{cov}(f(x_i),f(x_j)) = k(x_i, x_j)$, are determined by the \emph{mean function} $m$
and \emph{covariance function} (or \emph{kernel}) $k$ of the Gaussian process.

The popular RBF kernel has the form
\begin{align}
\label{eq:rbf}
k(x_i, x_j) = \exp \left(-\frac{1}{2\ell^2}\|x_i - x_j\|^2 \right) \,.
\end{align}
The \emph{length-scale} hyperparameter $\ell$ controls the extent of correlations between function values. If $\ell$ is large,
sample functions from a GP prior are simple and slowly varying with inputs $x$. 

Gaussian processes with RBF kernels (as well as many other standard kernels) assign positive density to any set of 
observations. Moreover, these models are \emph{universal approximators} \citep{rasmussen06}: as the number of observations increase, 
they are able to approximate any function to arbitrary precision.

Work on Gaussian processes in machine learning was triggered by the observation that Bayesian neural networks
become Gaussian processes with particular kernel functions as the number of hidden units approaches infinity \citep{neal1996}. 
This result resembles recent work on the neural tangent kernel \citep[e.g.,][]{jacot2018neural}.

\section{Deep Ensembles and MultiSWAG Under Distribution Shift}
\label{sec:appendix_ovadia}

In Figures \ref{fig:app_ovadia}, \ref{fig:app_ovadia_blur}, \ref{fig:app_ovadia_digital}, \ref{fig:app_ovadia_weather}
we show the negative log-likelihood for Deep Ensembles, MultiSWA and MultiSWAG 
using PreResNet-20 on CIFAR-10 with various corruptions as a function of the number of
independently trained models (SGD solutions, SWA solutions or SWAG models, respectively).
For MultiSWAG, we generate $20$ samples from each independent SWAG model.
Typically MultiSWA and MultiSWAG significantly outperform Deep Ensembles when a small number of
independent models is used, or when the level of corruption is high.

In Figure \ref{fig:app_ovadia_box}, following \citet{ovadia2019can}, we show the 
distribution of negative log likelihood, accuracy and expected calibration error 
as we vary the type of corruption.
We use a fixed training time budget: $10$ independently trained models for every
method. For MultiSWAG we ensemble $20$ samples from each of the $10$ SWAG
approximations.
MultiSWAG particularly achieves better NLL than the other two methods, and MultiSWA
outperforms Deep Ensembles; the difference is especially pronounced for higher
levels of corruption.
In terms of ECE, MultiSWAG again outperforms the other two methods for 
higher corruption intensities.

We note that \citet{ovadia2019can} found Deep Ensembles to be a very strong
baseline for prediction quality and calibration under distribution shift.
For this reason, we focus on Deep Ensembles in our comparisons.

\section{Details of Experiments}
\label{sec:app_details}

In this section we provide additional details of the experiments presented in the paper.

\subsection{Approximating the True Predictive Distribution}
\label{sec:appendix_hmc}

For the results presented in Figure \ref{fig:predictive} we used a network
with $3$ hidden layers of size $10$ each. The network takes two inputs: $x$
and $x^2$. We pass both $x$ and $x^2$ as input to ensure that the network can 
represent a broader class of functions. The network outputs a single number $y = f(x)$.

To generate data for the plots, we used a randomly-initialized neural network
of the same architecture described above. 
We sampled the weights from an isotropic Gaussian with variance $0.1^2$ and added
isotropic Gaussian noise with variance $0.1^2$ to the outputs:
\begin{align*}
    y = f(x;w) + \epsilon(x),
\end{align*}
with $w \sim \mathcal N(0, 0.1 ^2 \cdot I)$, 
$\epsilon(x) \sim \mathcal N(0, 0.1^2 \cdot I)$.
The training set consists of $120$ points shown in Figure \ref{fig:predictive}.

For estimating the ground truth we ran $200$ chains of Hamiltonian Monte Carlo (HMC)
using the \texttt{hamiltorch} package \citep{cobb2019introducing}. We initialized each chain
with a network pre-trained with SGD for $3000 $ steps, then 
ran Hamiltonian Monte Carlo (HMC) for $2000$ steps, producing $200$ samples.

For Deep Ensembles, we independently trained $50$ networks with SGD for $20000$
steps each.
We used minus posterior log-density as the training loss.
For SVI, we used a fully-factorized Gaussian approximation initialized at an SGD solution
trained for $20000$ steps.
For all inference methods we set prior variance to 
$10^2$ and noise variance to $0.02^2$.

\paragraph{Discrepancy with true BMA.}
For the results presented in panel (d) of Figure \ref{fig:predictive}
we computed Wasserstein distance between the predictive distribution
approximated with HMC and the predictive distribution for Deep Ensembles
and SVI. 
We used the one-dimensional Wasserstein distance function\footnote{\url{https://docs.scipy.org/doc/scipy/reference/generated/scipy.stats.wasserstein_distance.html}}
from the \textit{scipy} package \citep{virtanen2020scipy}.
We computed the Wasserstein distance between marginal distributions at 
each input location, and averaged the results over the input locations.
In the top sub-panels of panels (b), (c) of Figure \ref{fig:predictive} we additionally visualize
the marginal Wasserstein distance between the HMC predictive distribution and
Deep Ensembles and SVI predictive distrbutions respectively for each input location.

\subsection{Deep Ensembles and MultiSWAG}

We evaluate Deep Ensembles, MultiSWA and MultiSWAG under distribution shift in
Section \ref{sec: emp}.
Following \citet{ovadia2019can}, we use a PreResNet-20 network and the 
CIFAR-10 dataset with different types of corruptions introduced in \citet{hendrycks2019benchmarking}.
For training individual SGD, SWA and SWAG models we use the hyper-parameters 
used for PreResNet-164 in \citet{maddoxfast2019}.
For each SWAG model we sample $20$ networks and ensemble them.
So, Deep Ensembles, MultiSWA and MultiSWAG are all evaluated under the same training
budget; Deep Ensembles and MultiSWA also use the same test-time budget.

For producing the corrupted data we used the code\footnote{\url{https://github.com/hendrycks/robustness/blob/master/ImageNet-C/create_c/make_cifar_c.py}}
released by  \citet{hendrycks2019benchmarking}.
We had issues producing the data for the \textit{frost} corruption type, so 
we omit it in our evaluation, and include \textit{Gaussian blur} which
was not included in the evaluation of \citet{hendrycks2019benchmarking}.

\begin{figure*}[]
	\subfigure[$\alpha=0.02$]{
		\includegraphics[width=0.2\textwidth]{figs/priors/prior_correlations_2e2.pdf}
    }
	\subfigure[$\alpha=0.1$]{
		\includegraphics[width=0.2\textwidth]{figs/priors/prior_correlations_1e1.pdf}
    }
	\subfigure[$\alpha=1.$]{
		\includegraphics[width=0.2\textwidth]{figs/priors/prior_correlations_1e0.pdf}
    }
	\subfigure[]{
		\includegraphics[width=0.22\textwidth]{figs/priors/lenet_swag_prior_dep.pdf}
        \label{fig:lenet_priordep1}
    }
	\subfigure[$\alpha=0.02$]{
		\includegraphics[width=0.2\textwidth]{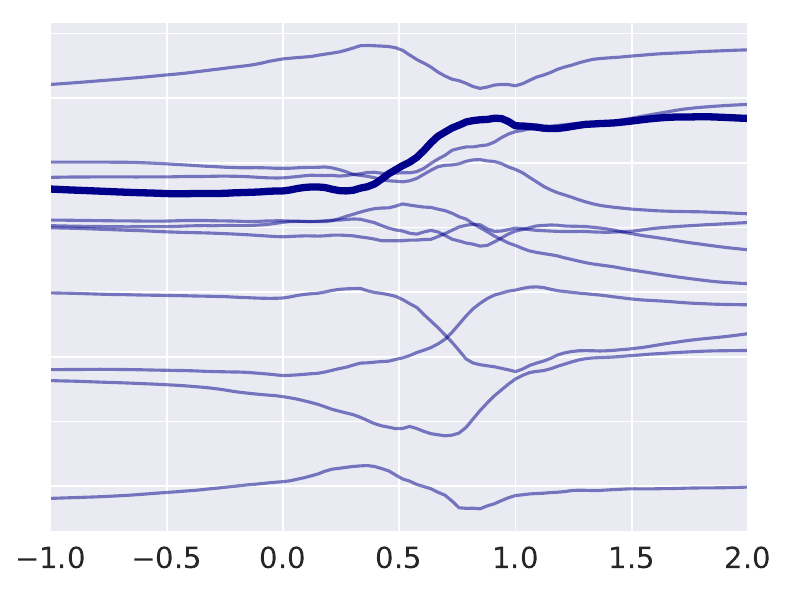}
    }
	\subfigure[$\alpha=0.1$]{
		\includegraphics[width=0.2\textwidth]{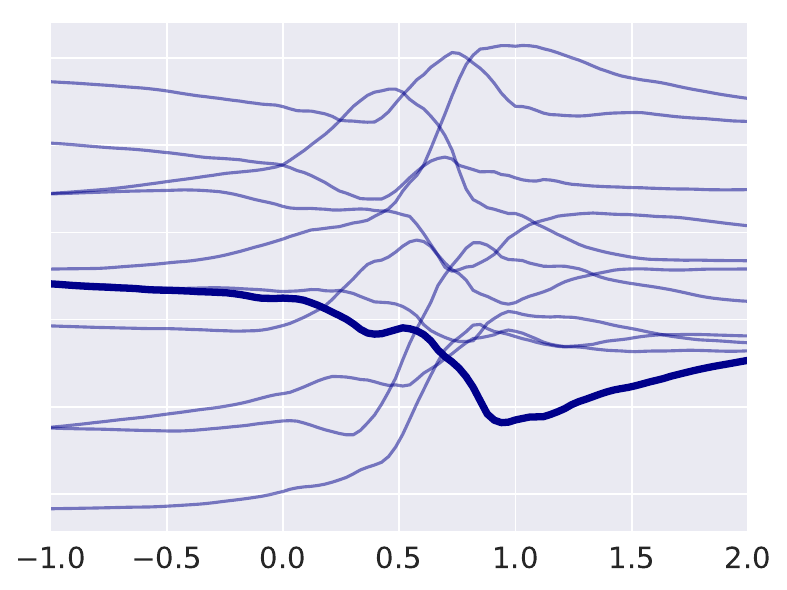}
    }
	\subfigure[$\alpha=1$]{
		\includegraphics[width=0.2\textwidth]{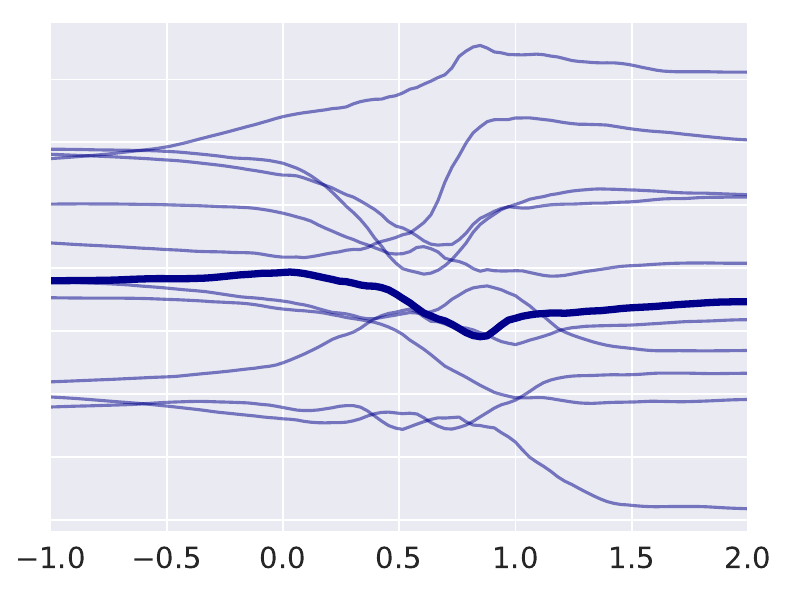}
    }
	\subfigure[]{
		\includegraphics[width=0.22\textwidth]{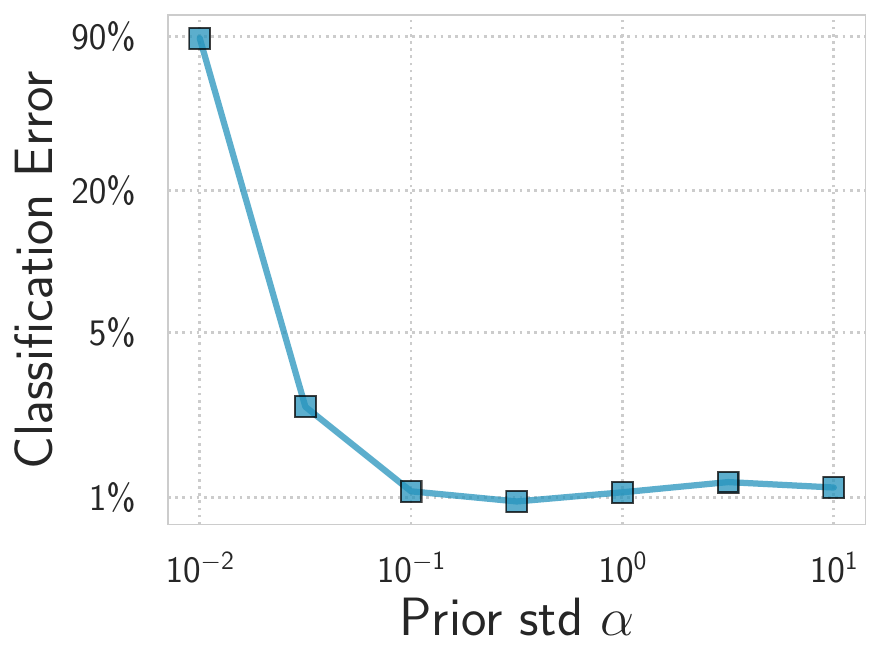}
        \label{fig:lenet_priordep2}
    }
	\centering
	\caption{
        \textbf{(a)}--\textbf{(c)}: Average pairwise prior correlations for pairs of objects
        in classes $\{0, 1, 2, 4, 7\}$ of MNIST induced by LeNet-5 for $p(f(x;w))$ 
        when $p(w) = \mathcal{N}(0,\alpha^2I)$. 
        Images in the same class have higher prior correlations than images from different classes,
        suggesting that $p(f(x;w))$ has desirable inductive biases. The correlations slightly
        decrease with increases in $\alpha$.
        Panels \textbf{(e)}--\textbf{(g)} show sample functions from LeNet-5 along
        the direction connecting a pair of MNIST images of $0$ and $1$ digits.
        The complexity of the samples increases with $\alpha$.
        \textbf{(d)}: NLL and \textbf{(h)} classification error of an ensemble of $20$ 
        SWAG samples on MNIST as a function of 
        $\alpha$ using a LeNet-5.
        The NLL is high for overly small $\alpha$ and near-optimal for larger values
        with an optimum near $\alpha = 0.3$.
        }
        \label{fig:prior_dep}
\end{figure*}

\subsection{Neural Network Priors}
\label{sec:app_prior_details}

In the main text we considered different properties of the prior distribution
over functions induced by a spherical Gaussian distribution over the weights,
with different variance scales.

\paragraph{Prior correlation diagrams.}
In panels (a)--(c) of Figure \ref{fig:prior_dep} we show pairwise correlations
of the logits for different pairs of datapoints. 
To make these plots we produce $S = 100$ samples of the weights $w_i$ of a LeNet-5 from the prior 
distribution $\mathcal N(0, \alpha^2 I)$ and compute the logits corresponding
to class $0$ for each data point and each weight sample.
We then compute the correlations for each pair $x$, $x'$ of data points
as follows:
\begin{align*}
    &\text{corr}_{\text{logit}}(x, x') = \\
    &\frac{\sum_{i=1}^S (f(x, w_i) - \bar f(x)) (f(x', w_i) - \bar f(x'))}
    {\sqrt{\sum_{i=1}^S (f(x, w_i) - \bar f(x))^2 \cdot \sum_{i=1}^S (f(x', w_i) - \bar f(x'))^2}},
\end{align*}
where $f(x, w)$ is the logit corresponding to class $0$ of the network with 
weights $w$ on the input $x$, and $\bar f(x)$ is the mean value of the logit
$\bar f(x) = \frac 1 S \sum_i f(x, w_i)$. 
For evaluation, we use $200$ random datapoints per class for classes 
$0, 1, 2, 4, 7$ (a total of $1000$ datapoints). 
We use this set of classes to ensure that the structure is clearly 
visible in the figure.
We combine the correlations into a diagram, additionally showing the
average correlation for each pair of classes.
We repeat the experiment for different values of $\alpha \in \{0.02, 0.1, 1\}$.
For a discussion of the results see Section \ref{sec:prior_corrs}.

\paragraph{Sample functions.}
In panels (e)--(g) of Figure \ref{fig:prior_dep} we visualize the functions
sampled from the LeNet-5 network along the direction connecting a pair of 
MNIST images.
In particular, we take a pair of images $x_0$ and $x_1$ of digits $0$ and $1$,
respectively, and construct the path $x(t) = t \cdot x_0 + (1-t) \cdot x_1$.
We then study the samples of the logits $z(t) = f(x(t) \cdot \|x_0\| / \|x(t)\|, w)$ 
along the path; here we adjusted the norm of the images along the path to be constant
as the values of the logits are sensitive to the norm of the inputs.
The complexity of the samples increases as we increase the variance of the prior
distribution over the weights. 
This increased complexity of sample functions explains why we might expect the prior
correlations for pairs of images to be lower when we increase the variance of 
the prior distribution over the weights.

\paragraph{Performance dependence on prior variance.}
In panels (d), (h) of Figure \ref{fig:prior_dep} we show the test negative log-likelihood
and accuracy of SWAG applied to LeNet-5 on MNIST. 
We train the model for $50$ epochs, constructing the rank-$20$ SWAG approximation
from the last $25$ epochs. We use an initial learning rate of $0.05$ and 
SWAG learning rate of $0.01$ with the learning rate schedule of \citet{maddoxfast2019}.
We use posterior log-density as the objective, and vary the prior variance 
$\alpha^2$. 
In panels (f), (g) of Figure \ref{fig:priorscale} we perform an analogous experiment using a PreResNet-20 and a VGG-16
on CIFAR-10, using the hyper-parameters reported in \citet{maddoxfast2019} 
(for PreResNet-20 we use the hyper-parameters used with PreResNet-164 in \citet{maddoxfast2019}).
Both on MNIST and CIFAR-10 we observe that the performance is poor for 
overly small values of $\alpha$, close to optimal for intermediate values, 
and still reasonable for larger values of~$\alpha$.
For further discussion of the results see Section \ref{sec:prior_dep}.

\paragraph{Predictions from prior samples.}
Following \citet{wenzel2020good} we study the predictive distributions of
prior samples using PreResNet-20 on CIFAR-10. 
In Figure \ref{fig:priorscale} we show the sample predictive functions
averaged over datapoints for different scales $\alpha$ of the prior distribution. 
We also show the predictive distribution for each $\alpha$, which is the
average of the sample predictive distributions over $200$ samples of weights. 
In Figure \ref{fig: priordata} we show how the predictive distribution changes
as we vary the number of observed data for prior scale $\alpha = \sqrt{10}$.
We see that the marginal predictive distribution for all considered values of $\alpha$ is reasonable --- roughly 
uniform across classes, when averaged across the dataset.
For the latter experiment we used stochastic gradient Langevin dynamics (SGLD) \citep{welling2011bayesian} with 
a cosine lerning rate schedule. 
For each sample we restart SGLD, and we only use the sample obtained at the last
iteration. 
We discuss the results in Section~\ref{sec: prioreffect}.

\paragraph{Prior correlations with corrupted images.}
In Section \ref{sec:app_correlations} and Figure \ref{fig:corr_corr}
we study the decay of the prior correlations between logits on an original image  and
a perturbed image as we increase the intensity of perturbations.
For the BNN we use PreResNet-20 architecture with the standard Gaussian
prior $\mathcal N(0, I)$.
For the linear model, the correlations are not affected by the prior 
variance $\alpha^2$: 
\begin{align*}
    cov(w^T x, w^T y) &= \mathbb E (w^T x \cdot w^T y) = \\
    &\mathbb E x^T w w^T y =
    x^T \mathbb E w w^T y = 
    \alpha^2 x^T y, 
\end{align*}
and hence
\begin{align*}
    corr(w^T x,& w^T y) =  \\
    &\frac{cov(w^T x, w^T y)}{\sqrt{cov(w^T y, w^T y) \cdot cov(w^T x, w^T x)}}
    = x^T y.
\end{align*}
We use the $\mathcal N(0, I)$ prior for the weights of the linear model. 
Finally, we also evaluate the correlations associated with an RBF kernel
(see Equation \eqref{eq:rbf}).
To set the lengthscale $\ell$ of the kernel we evaluate the pairwise correlations
for the PreResnet-20 and RBF kernel on the $100$ uncorrupted CIFAR-10 images
that were used for the experiment, and ensure that the average correlations 
match. The resulting value of $\ell$ is $10000$, and the average correlation for
the RBF kernel and PreResNet was $\approx 0.9$; for the linear model the average
correlation was $\approx 0.82$.
For the perturbations we used the same set of corruptions introduced in
\citet{hendrycks2019benchmarking} as in the experiments in Section \ref{sec: emp}
with the addition of a random translation: for a random translation of intensity
$i$ we pad the image with $2 \cdot i$ zeros on each side and crop the image randomly to
$32 \times 32$.

\subsection{Rethinking Generalization}

In Section \ref{sec: rethinking},
we experiment with Bayesian neural networks and Gaussian processes on CIFAR-10
with noisy labels, inspired by the results in \citet{zhang2016understanding} that suggest
we need to re-think generalization to understand deep learning.

Following  \citet{zhang2016understanding}, we train 
PreResNet-20 on CIFAR-10 with different fractions of random labels.
To ensure that the networks fits the train data, we turn off weight decay
and data augmentation, and use a lower initial learning rate of $0.01$. 
Otherwise, we follow the hyper-parameters that were used with
PreResNet-164 in \citet{maddoxfast2019}.
We use diagonal Laplace approximation to compute an estimate
of marginal likelihood for each level of label corruption. 
Following \citet{ritter2018scalable} we use the diagonal of the Fisher 
information matrix rather than the Hessian.

We perform a similar experiment with a Gaussian process
with RBF kernel on the binary classification problem for two classes of CIFAR-10.
We use variational inference to fit the model, and we use the variational
evidence lower bound to approximate the marginal likelihood.
We use variational inference to overcome the non-Gaussian likelihood
and not for scalability reasons; i.e., we are not using inducing inputs.
We use the \texttt{GPyTorch} package \citep{gardner2018gpytorch} to train the models.
We use an RBF kernel with default initialization from \texttt{GPyTorch} and divide the
inputs by $5000$ to get an appropriate input scale.
We train the model on a binary classification problem between classes $0$ and $1$.

For the $10$-class GP classification experiment we train $10$ one-vs-all models
that classify between a given class and the rest of the data. 
To reduce computation, in training we subsample the data not belonging to the
given class to $10k$ datapoints, so each model is trained on a total of $15k$
datapoints.
We then combine the 10 models into a single multi-class model: an observation
is attributed to the class that corresponds to the one-vs-all model with
the highest confidence.
We use the same hyper-parameters as in the binary classification experiments.

\begin{figure*}[h]
	\centering
	\subfigure[10\% Corrupted (Err)]{
		\includegraphics[height=0.17\textwidth]{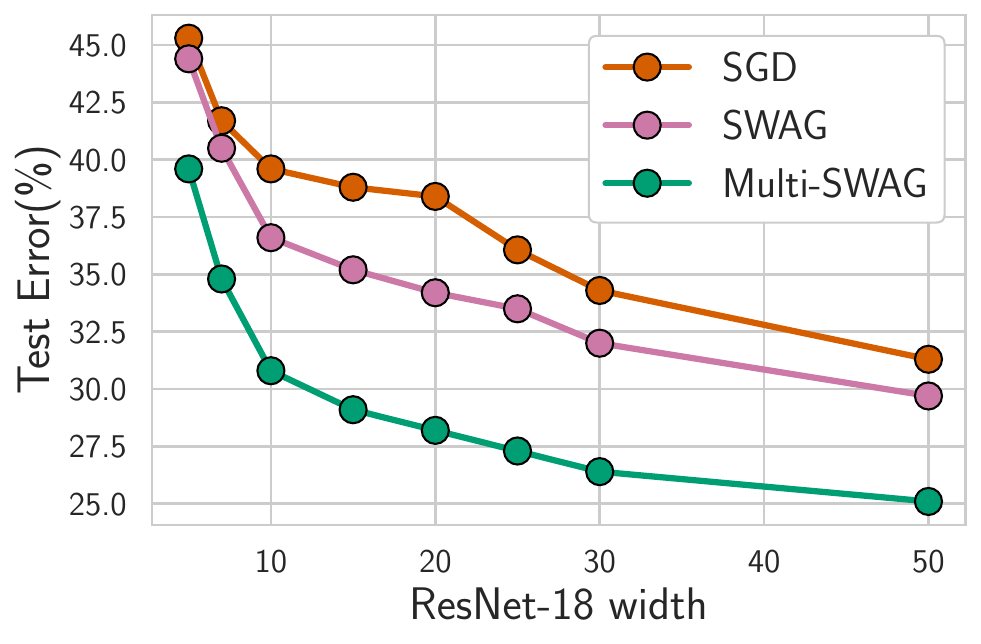}
		\label{fig:app_dd_corr_acc}
    }
    \hspace{-0.1cm}
	\subfigure[10\% Corrupted (NLL)]{
		\includegraphics[height=0.17\textwidth]{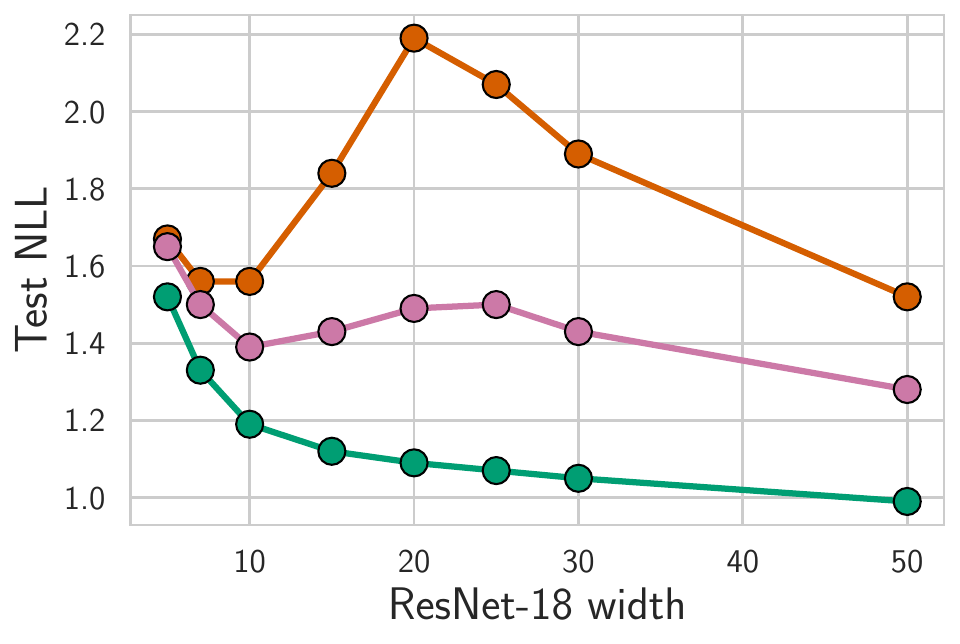}
		\label{fig:app_dd_corr_nll}
    }
    \hspace{-0.1cm}
	\subfigure[20\% Corrupted (\#~Models)]{
		\includegraphics[height=0.17\textwidth]{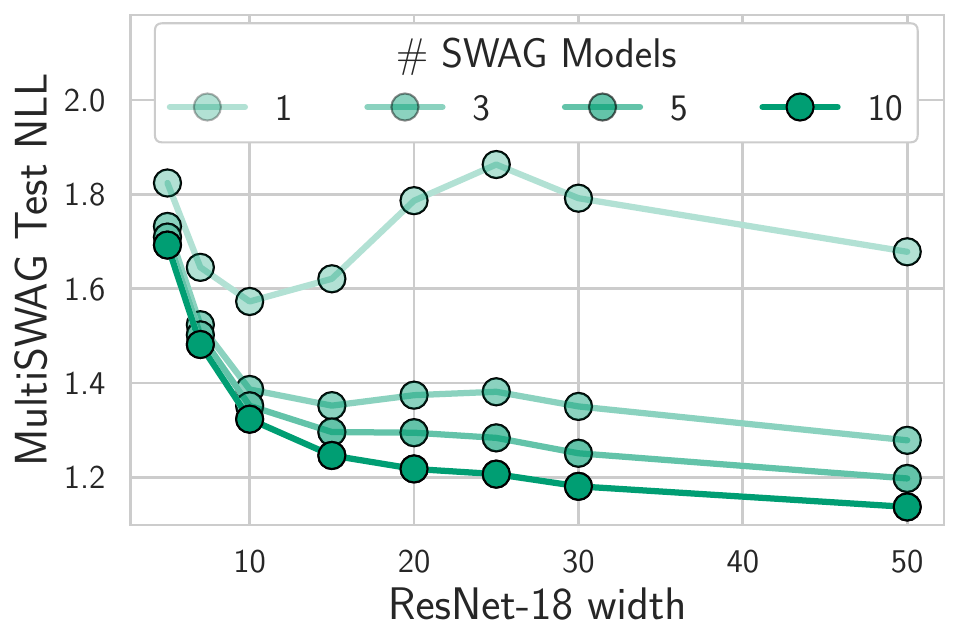}
        \label{fig:app_dd_multiswag}
    }
	\caption{\textbf{Double Descent.} 
    \textbf{(a)}: Test error and \textbf{(b)}: NLL loss for ResNet-18 with varying width on CIFAR-100
    for SGD, SWAG and MultiSWAG when 10\% of the labels are randomly reshuffled.
    MultiSWAG alleviates double descent both on the original labels and under label
    noise, both in accuracy and NLL. 
    \textbf{(e)}: Test NLLs for MultiSWAG with varying number of independent models
	under $20\%$ label corruption; 
    NLL monotonically decreases with increased number of independent models,
    alleviating double descent. 
    }
    \label{fig:app_double_descent}
\end{figure*}

\subsection{Double Descent}

In Section \ref{sec:double_descent} we evaluate SGD, SWAG and MultiSWAG for 
models of varying width. Following \citet{nakkiran2019deep} we use 
ResNet-18 on CIFAR-100; we consider original labels, $10\%$ and $20\%$ label corruption.
For networks of every width we reuse the hyper-paramerers used for PreResNet-164 in \citet{maddoxfast2019}.
For original labels and $10\%$ label corruption we use $5$ independently trained
SWAG models with MultiSWAG, and for $20\%$ label corruption we use $10$ models;
for $20\%$ label corruption we also show performance varying the number of independent
models in Figures \ref{fig:dd_multiswag} and \ref{fig:app_dd_multiswag}.
Both for SWAG and MultiSWAG we use an ensemble of $20$ sampled models from each 
of the SWAG solutions; for example, for MultiSWAG with $10$ independent 
SWAG solutions, we use an ensemble of $200$ networks.

\section{Analysis of Prior Variance Effect}
\label{sec:app_prior_analysis}

In this section we provide simple analytic results for the effect of prior variance
in ReLU networks. 
A related derivation is presented in the Appendix Section A.8 of \citet{garipov2018} 
about connecting paths from symmetries in parametrization.

We will consider a multilayer network $f(x, w)$ of the form
\begin{align*}
    f(x, \{W_i,& b_i\}_{i=1}^n) = \\
    &W_n (\ldots \phi(W_2 \phi(W_1 x + b_1) + b_2)) + b_n,
\end{align*}
where $\phi$ is the ReLU (or in fact any positively-homogeneous activation function),
$W_i$ are weight matrices and $b_i$  are bias vectors. 
In particular, $f$ can be a regular CNN with ReLU activations up to
the logits (with softmax activation removed).

Now, suppose we have a prior distribution of the form
\begin{align*}
    W_i \sim \mathcal N(0, \alpha_i^2 I),~~
    b_i \sim \mathcal N(0, \beta_i^2 I),
\end{align*}
where the identity matrices $I$ are implicitly assumed to be of appropriate shapes,
so each weight matrix and bias vector has a spherical Gaussian distribution.
We can reparameterize this distribution as
\begin{align*}
    W_i & = \alpha_i \mE_i, ~~ \mE_i \sim \mathcal N(0, I),\\
    b_i & = \beta_i \epsilon_i, ~~ \epsilon_i \sim \mathcal N(0, I).
\end{align*}
We can then express the predictions of the network on the input $x$ for
weights sampled from the prior as the random variable
\begin{align}
\begin{split}
    &f(x, \{\alpha_i, \beta_i\}_{i=1}^n) = \\
    &\alpha_n \cdot \mE_n (\ldots \phi(\alpha_1 \mE_1 x + \beta_1 \cdot \epsilon_1)) + \beta_n \cdot \epsilon_n.
\end{split}
\label{eq:priordist}
\end{align}

Through Equation \eqref{eq:priordist}, we can observe some simple properties of
the dependence between the prior scales $\alpha_i$, $\beta_i$ and the induced
function-space prior.

\begin{proposition}
\label{prop:nobias}
Suppose the network has no bias vectors, i.e. $\beta_1 = \ldots = \beta_n = 0$. 
Then the scales $\alpha_i$ of the prior distribution over the weights only affect
the output scale of the network.
\end{proposition}
\begin{proof}
In the case when there are no bias vectors Equation \eqref{eq:priordist}
simplifies to 
\begin{align*}
    f(x, &\{\alpha_i, \beta_i=0\}_{i=1}^n) = \\
    &\alpha_n \cdot \mE_n (\ldots \phi(\alpha_1 \mE_1 x + \beta_1 \cdot \epsilon_1)) + \beta_n \cdot \epsilon_n =\\
    &\alpha_n \cdot \ldots \cdot \alpha_1 \cdot\mE_n (\ldots \phi(\mE_1 x)) = \\
    &\alpha_n \cdot \ldots \cdot \alpha_1 \cdot f(x, \{\alpha_i=1, \beta_i=0\}_{i=1}^n).
\end{align*}
In the derivation above we used positive homogeneity of ReLU: $\phi(\alpha z) = \alpha \phi(z)$ for any positive $\alpha$.
\end{proof}

In other words, to sample from the distribution over functions corresponding to 
a prior with variances $\{\alpha_i, \beta_i=0\}_{i=1}^n$, we can sample 
from the spherical Gaussian prior (without bias terms) $\{\alpha_i=1, \beta_i=0\}_{i=1}^n$
and then rescale the outputs of the network by the product of variances 
$\alpha_n \cdot \ldots \cdot \alpha_2 \cdot \alpha_1$.

We note that the result above is different from the results for \textit{sigmoid}
networks considered in \citet{mackay1995probable}, where varying the 
prior on the weights leads to changing the length-scale of the sample functions.
For ReLU networks without biases, increasing prior variance only increases
the output scale of the network and not the complexity of the samples.
If we apply the softmax activation on the outputs of the last layer of such
network, we will observe increasingly confident predictions as we increase
the prior variance. 
We observe this effect in Figure \ref{fig:priorscale} and discuss it in Section
\ref{sec: prioreffect}.

In case bias vectors are present, we can obtain a similar result using
a specific scaling of the prior variances with layer, as in the following
proposition.

\begin{proposition}
\label{prop:scaling}
Suppose the prior scales depend on the layer of the network as follows
for some $\gamma > 0$:
\begin{align*}
    \alpha_i  = \gamma,~~
    \beta_i = \gamma^i,
\end{align*}
for all layers $i = 1 \ldots n$. 
Then $\gamma$ only affects the scale of the predictive distribution at any
input $x$:
\begin{align*}
    f(x, &\{\alpha_i=\gamma, \beta_i=\gamma^i\}_{i=1}^n) = 
    \gamma^n \cdot f(x, \{\alpha_i=1, \beta_i=1\}_{i=1}^n).
\end{align*}
\end{proposition}
\begin{proof}
The proof is analogous to the proof of Proposition \ref{prop:nobias}.
We can use the positive homogenety of ReLU activations to factor the prior
scales outside of the network:
\begin{align*}
    f(x, &\{\alpha_i=\gamma, \beta_i=\gamma^i\}_{i=1}^n) = \\
    &\gamma \cdot \mE_n ( \ldots \phi(\gamma \cdot \mE_1 x + \gamma \cdot \epsilon_1)) + \gamma^n \cdot \epsilon_n =\\
    &\gamma^n \cdot \big( \mE_n( \ldots \phi(\mE_1 x + \epsilon_1))) + \epsilon_n \big ) = \\
    &\gamma^n \cdot f(x, \{\alpha_i=1, \beta_i=1\}_{i=1}^n).
\end{align*}
\end{proof}

The analysis above can be applied to other simple scaling rules of the 
prior, e.g.
\begin{align}
\label{eq:priorscaling}
\begin{split}
    f(x, &\{\alpha_i=\gamma \hat \alpha_i, \beta_i=\gamma^i \hat \beta_i\}_{i=1}^n) = \\
    &\gamma^n \cdot f(x, \{\alpha_i=\hat \alpha_i, \beta_i= \hat \beta_i\}_{i=1}^n),
\end{split}
\end{align}
can be shown completely analogously to Proposition \ref{prop:scaling}.

More general types of scaling of the prior affect both the output scale of the
network and also the relative effect of prior and variance terms.
For example, by Equation \eqref{eq:priorscaling} we have
\begin{align*}
    f(x, &\{\alpha_i=\gamma, \beta_i=\gamma\}_{i=1}^n) = \\
    f(x, &\{\alpha_i=\gamma \cdot 1, \beta_i=\gamma^{i} \cdot \gamma^{1-i}\}_{i=1}^n) = \\
    &\gamma^n \cdot f(x, \{\alpha_i=1, \beta_i=\gamma^{1-i}\}_{i=1}^n).
\end{align*}

We note that the analysis does not cover residual connections and batch normalization, 
so it applies to LeNet-5 but cannot be directly applied to PreResNet-20 networks used in many of our
experiments.

\section{Prior Correlation Structure under Perturbations}
\label{sec:app_correlations}

In this section we explore the prior correlations between the logits on different 
pairs of datapoints induced by a spherical Gaussian prior on the weights of a PreResNet-20.
We sample a $100$ random images from CIFAR-10 ($10$ from each class)
and apply $17$ different perturbations introduced by \citet{hendrycks2019benchmarking}
at $5$ different levels of intensity. 
We then compute correlations between the logits $f(x, w)$ for the original image $x$
and $f(\tilde{x}, w)$ for the corrupted image $\tilde{x}$,
as we sample the weights of the network from the prior $w \sim \mathcal N(0, I)$.

In Figure \ref{fig:corr_corr} we show how the correlations decay with perturbation
intensity.
For reference we also show how the correlations decay for a linear model and for
an RBF kernel. 
For the RBF kernel we set the lengthscale so that the average correlations 
on the uncorrupted datapoints match those of a PreResNet-20.
Further experimental details can be found in Appendix \ref{sec:app_prior_details}.

For all types of corruptions except \textit{saturate}, \textit{snow}, \textit{fog}\
and \textit{brightness} the PreResNet logits decay slower compared to the RBF
kernel and linear model.
It appears that the prior samples are sensitive to corruptions that alter the brightness
or more generally the colours in the image.
For many types of corruptions (such as e.g. \textit{Gaussian Noise}) the prior 
correlations for PreResNet are close to $1$ for all levels of corruption.

Overall, these results indicate that the prior over \emph{functions} induced by a vague
prior over parameters $w$ in combination with a PreResNet has useful equivariance
properties: before seeing data, the model treats images of the same class as highly
correlated, even after an image has undergone significant perturbations representative
of perturbations we often see in the real world. These types of symmetries are a large 
part of what makes neural networks a powerful model class for high dimensional natural
signals.

\begin{figure*}[t]
	\centering
	\subfigure[Gaussian Noise]{
		\includegraphics[width=0.9\textwidth]{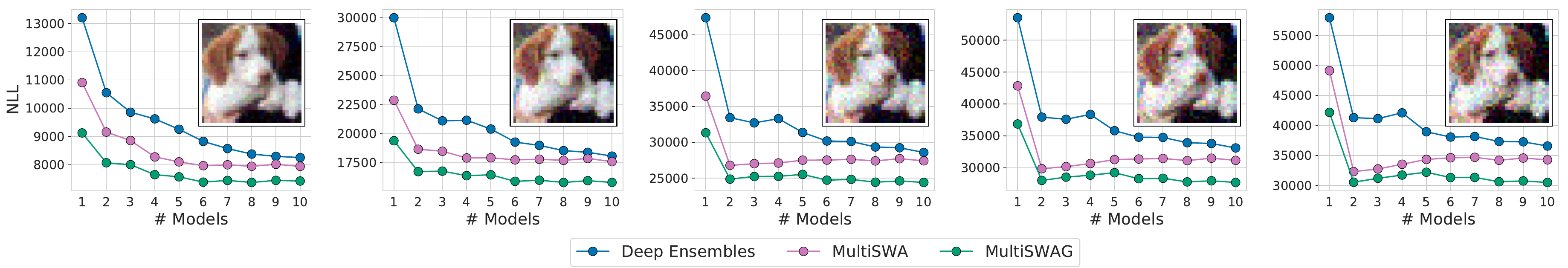}
	}
	\subfigure[Impulse Noise]{
		\includegraphics[width=0.9\textwidth]{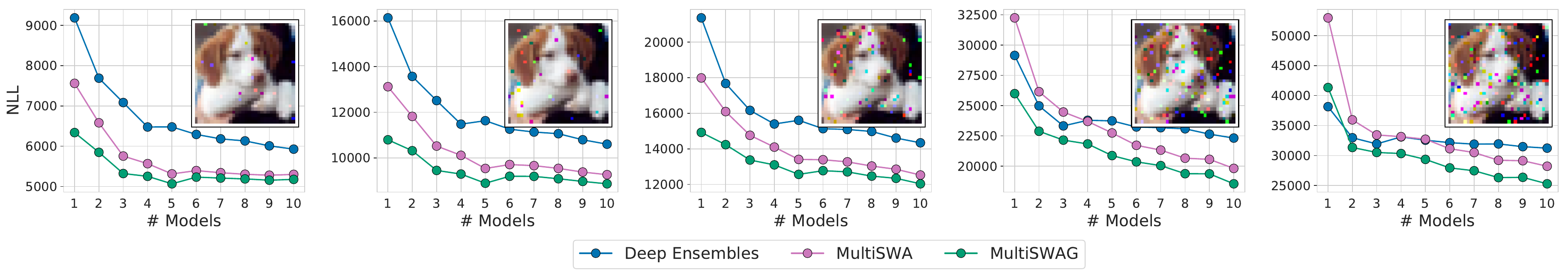}
	}
	\subfigure[Shot Noise]{
		\includegraphics[width=0.9\textwidth]{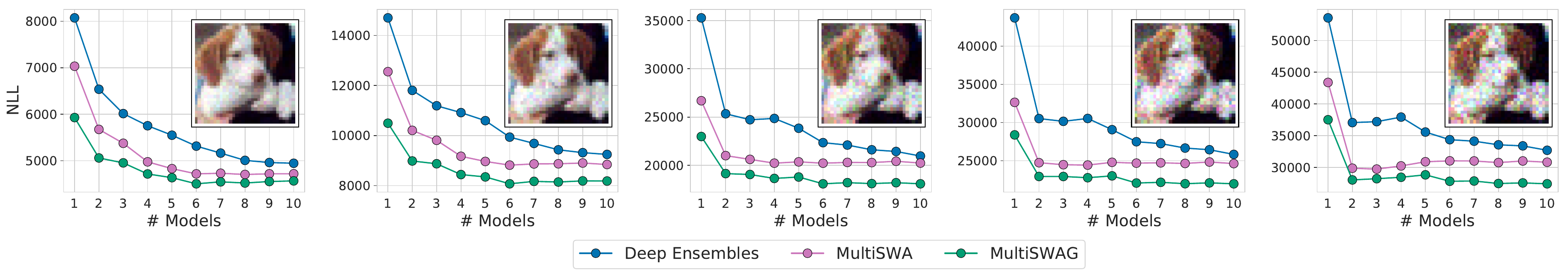}
	}
	\caption{
		\textbf{Noise Corruptions.} Negative log likelihood on CIFAR-10 with a PreResNet-20 for Deep Ensembles, MultiSWAG and MultiSWA as a 
		function of the number of independently trained models
		for different types of corruption and corruption intensity (increasing from left to right).
	}
    \label{fig:app_ovadia}
\end{figure*}
\begin{figure*}[t]
	\centering
	\subfigure[Defocus Blur]{
		\includegraphics[width=0.9\textwidth]{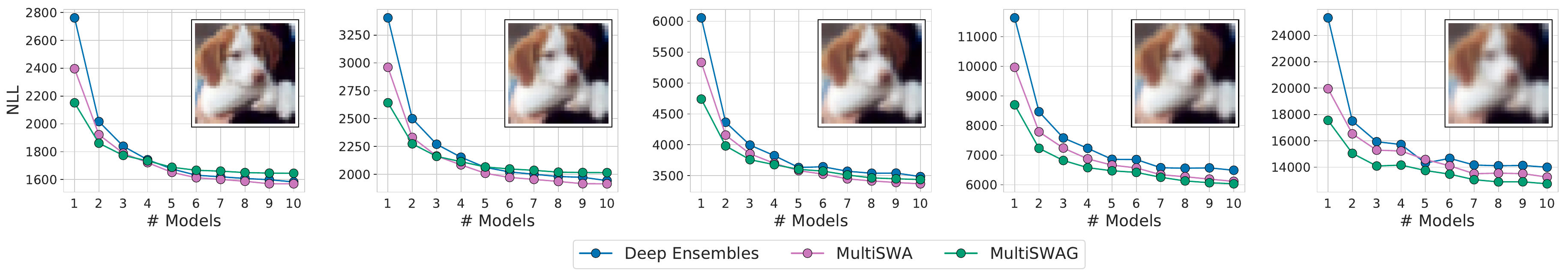}
    }	
	\subfigure[Glass Blur]{
		\includegraphics[width=0.9\textwidth]{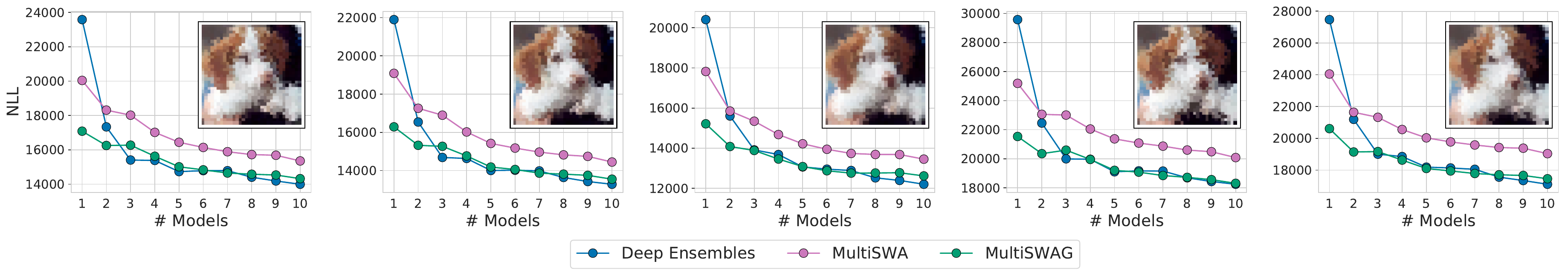}
    }	
	\subfigure[Motion Blur]{
		\includegraphics[width=0.9\textwidth]{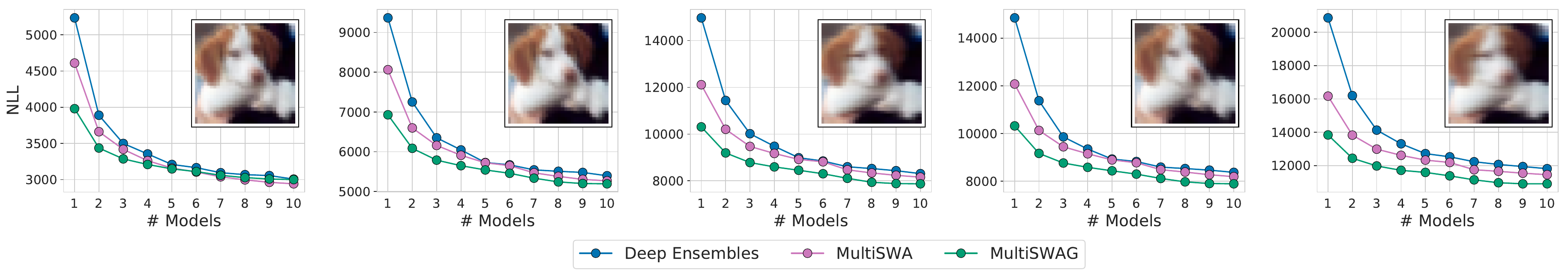}
    }	
	\subfigure[Zoom Blur]{
		\includegraphics[width=0.9\textwidth]{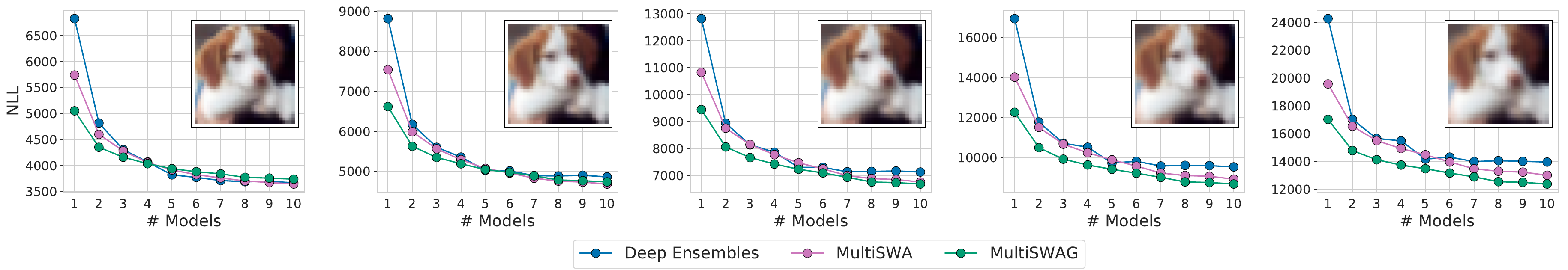}
    }	
	\subfigure[Gaussian Blur]{
		\includegraphics[width=0.9\textwidth]{figs/deep_ensembles/gaussian_blur.pdf}
	}
	\caption{
		\textbf{Blur Corruptions.} Negative log likelihood on CIFAR-10 with a PreResNet-20 for Deep Ensembles, MultiSWAG and MultiSWA as a 
		function of the number of independently trained models
		for different types of corruption and corruption intensity (increasing from left to right).
	}
    \label{fig:app_ovadia_blur}
\end{figure*}
\begin{figure*}[t]
	\centering
	\subfigure[Contrast]{
		\includegraphics[width=0.9\textwidth]{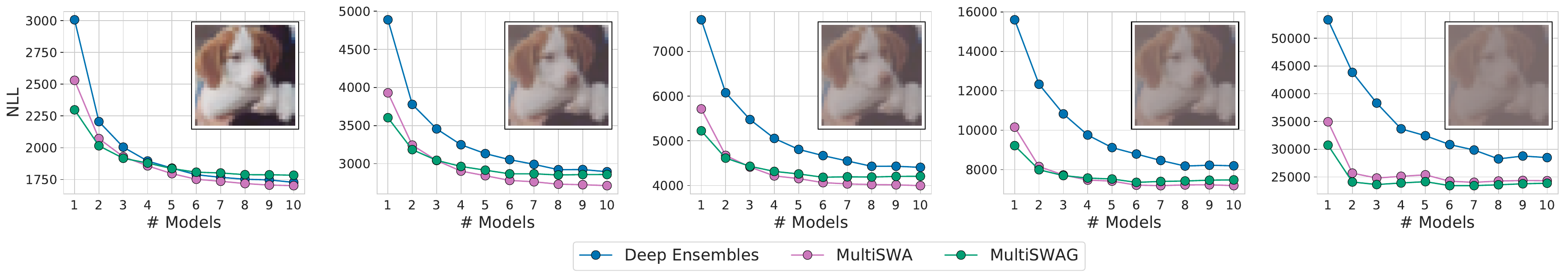}
    }	
	\subfigure[Saturate]{
		\includegraphics[width=0.9\textwidth]{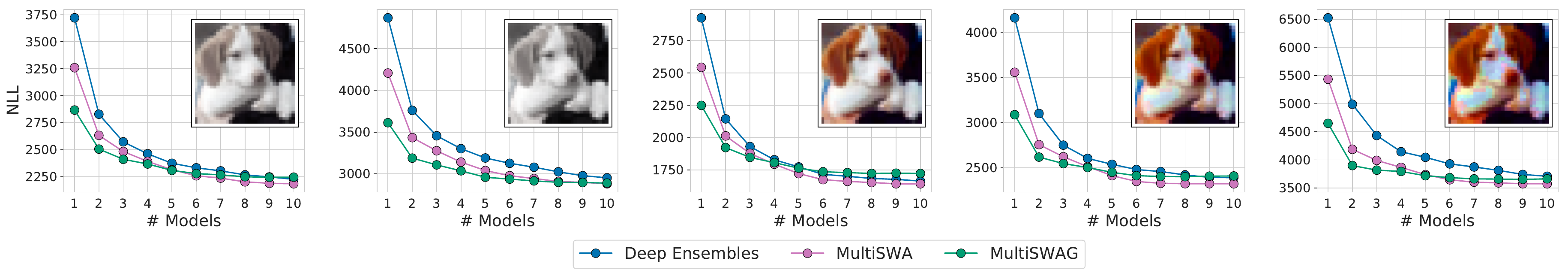}
    }	
	\subfigure[Elastic Transform]{
		\includegraphics[width=0.9\textwidth]{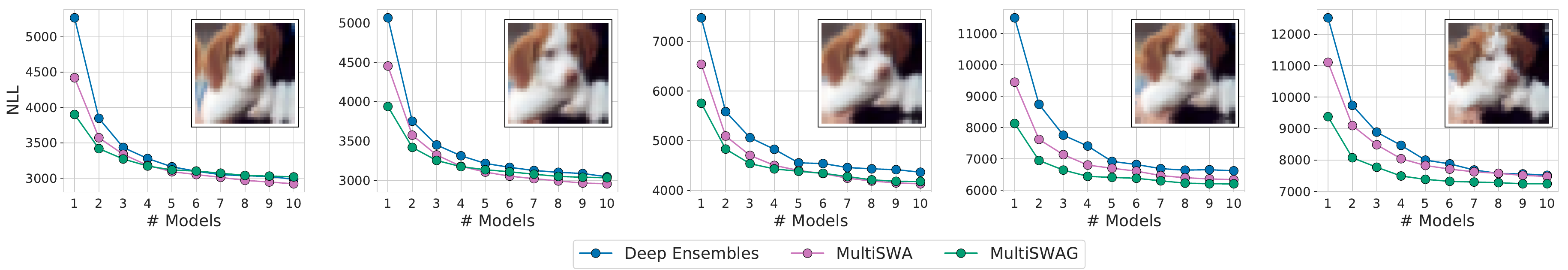}
    }	
	\subfigure[Pixelate]{
		\includegraphics[width=0.9\textwidth]{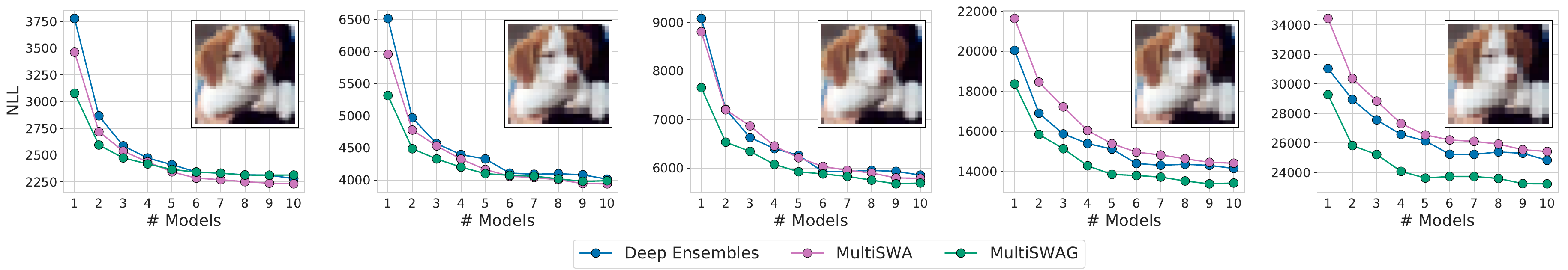}
    }	
	\subfigure[JPEG Compression]{
		\includegraphics[width=0.9\textwidth]{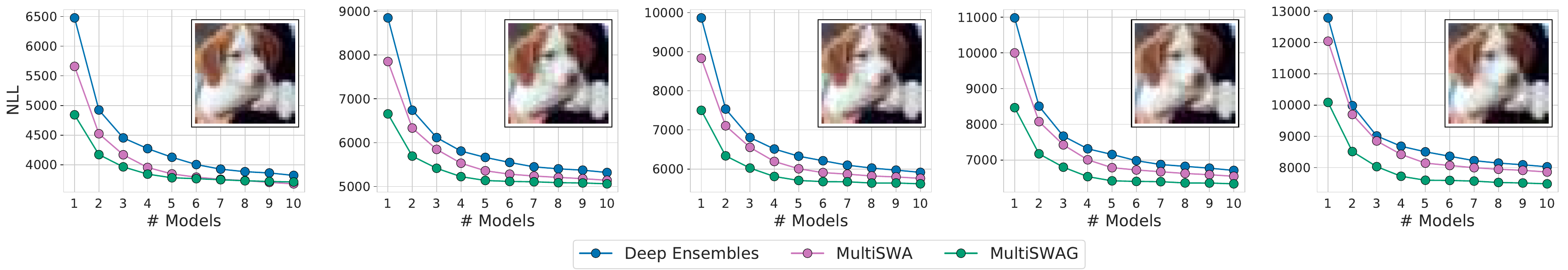}
    }	
	\caption{
		\textbf{Digital Corruptions.} Negative log likelihood on CIFAR-10 with a PreResNet-20 for Deep Ensembles, MultiSWAG and MultiSWA as a 
		function of the number of independently trained models
		for different types of corruption and corruption intensity (increasing from left to right).
	}
    \label{fig:app_ovadia_digital}
\end{figure*}
\begin{figure*}[t]
	\centering
	\subfigure[Snow]{
		\includegraphics[width=0.9\textwidth]{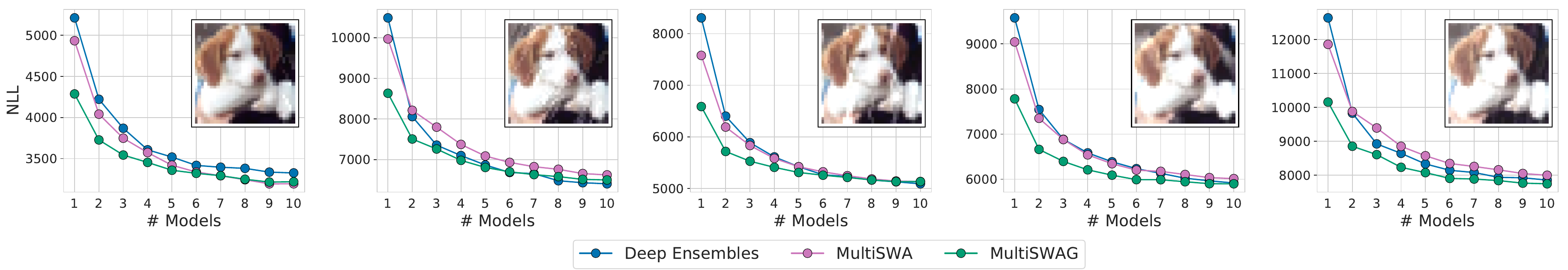}
    }	
	\subfigure[Fog]{
		\includegraphics[width=0.9\textwidth]{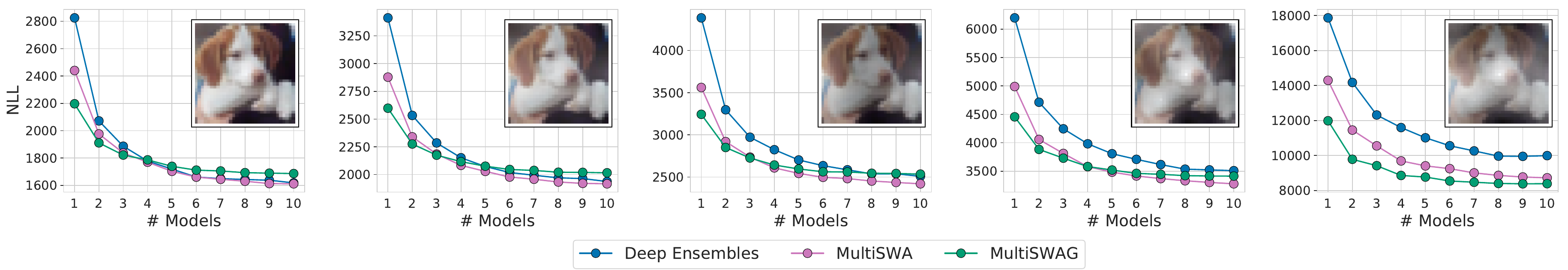}
    }	
	\subfigure[Brightness]{
		\includegraphics[width=0.9\textwidth]{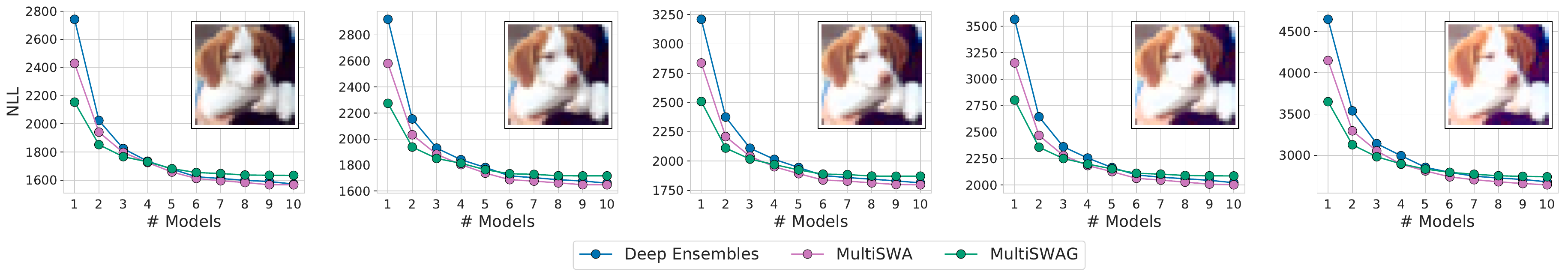}
    }	
	\caption{
		\textbf{Weather Corruptions.} Negative log likelihood on CIFAR-10 with a PreResNet-20 for Deep Ensembles, MultiSWAG and MultiSWA as a 
		function of the number of independently trained models
		for different types of corruption and corruption intensity (increasing from left to right).
	}
    \label{fig:app_ovadia_weather}
\end{figure*}

\begin{figure*}[t]
	\centering
	\subfigure{
		\includegraphics[height=0.13\textwidth]{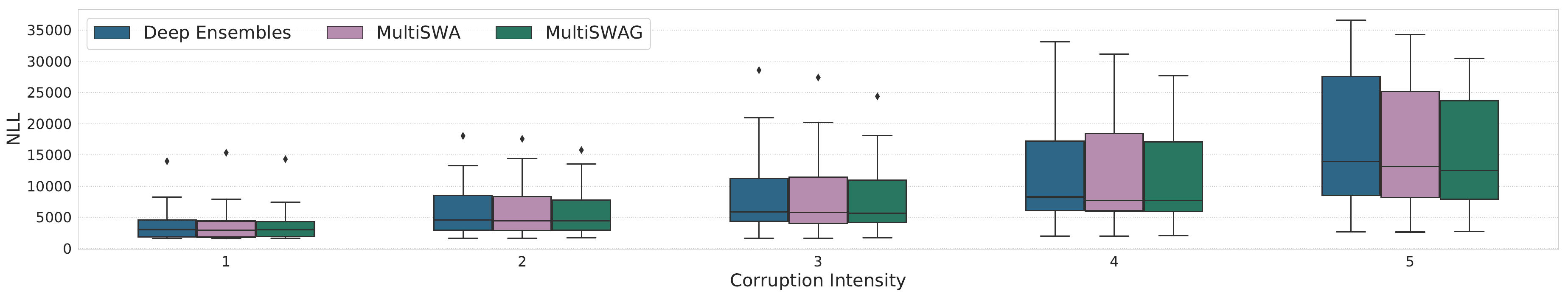}
    }	
	\subfigure{
		\includegraphics[height=0.13\textwidth]{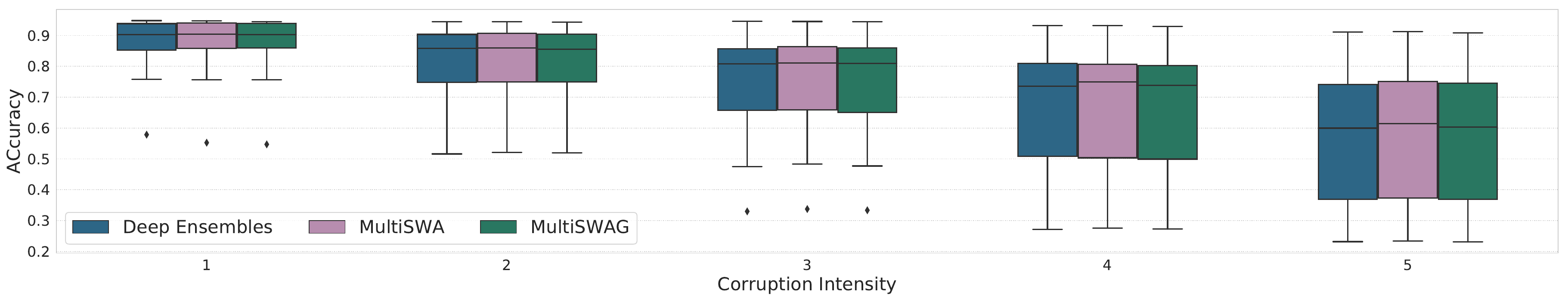}
    }	
	\subfigure{
		\includegraphics[height=0.13\textwidth]{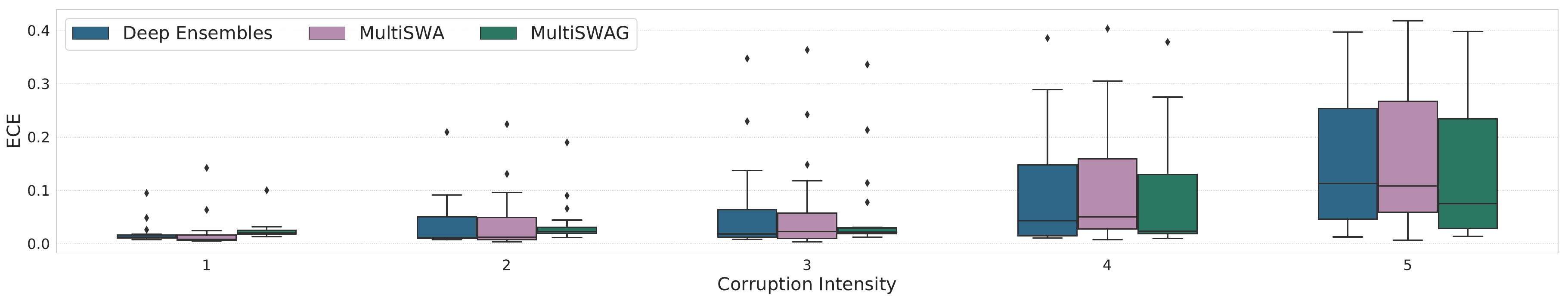}
    }	
	\caption{
		Negative log likelihood, accuracy and expected calibration error distribution 
        on CIFAR-10 with a PreResNet-20 for Deep Ensembles, MultiSWAG and MultiSWA as a 
		function of the corruption intensity. 
        Following \citet{ovadia2019can} we summarize the results for different 
        types of corruption with a boxplot.
        For each method, we use $10$ independently trained models, and for MultiSWAG
        we sample $20$ networks from each model.
        As in Figures 5, 11-14, there are substantial differences between these three methods,
        which are hard to see due to the vertical scale on this plot.
        MultiSWAG particularly outperforms 
        Deep Ensembles and MultiSWA in terms of NLL and ECE for higher corruption
        intensities.
	}
    \label{fig:app_ovadia_box}
\end{figure*}

\begin{figure*}[t]
    \def \panelwidth {0.16\textwidth}
	\centering
	\subfigure{
    \begin{tabular}{cccc}
		\includegraphics[width=\panelwidth]{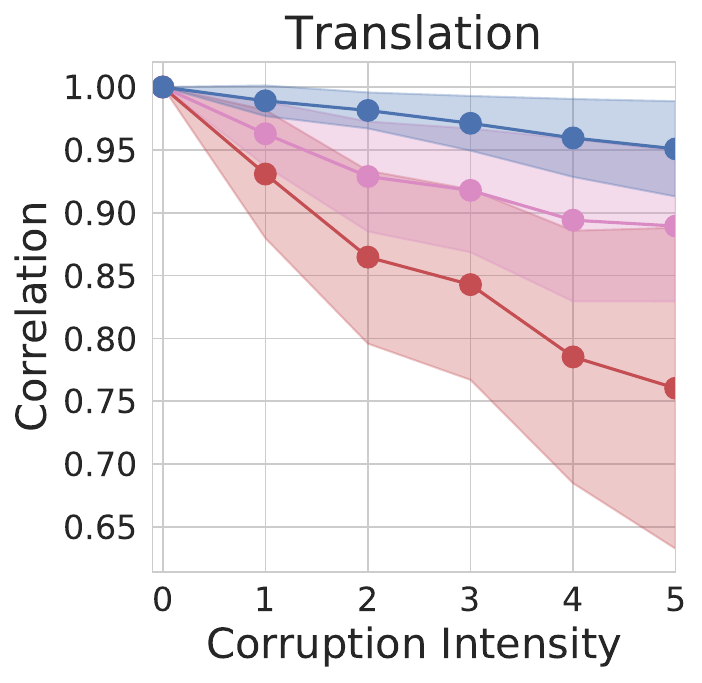} &
		\includegraphics[width=\panelwidth]{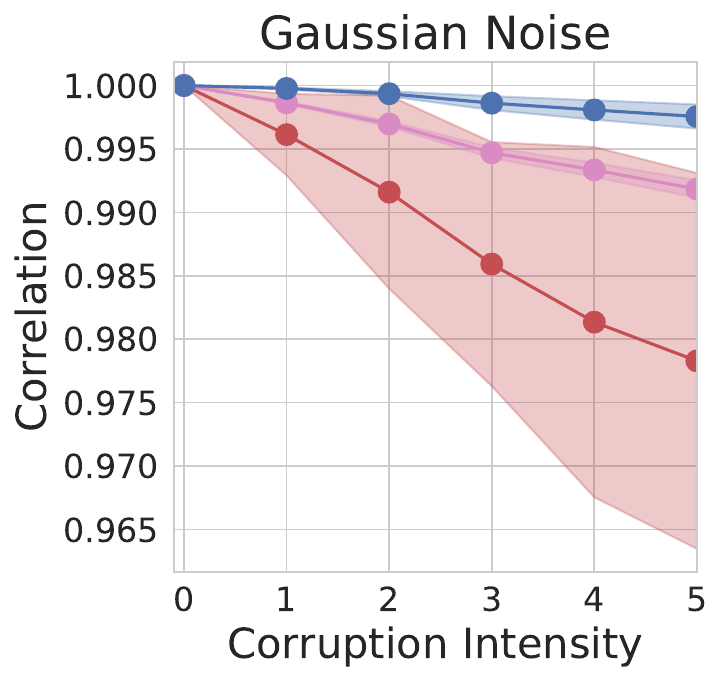} &
		\includegraphics[width=\panelwidth]{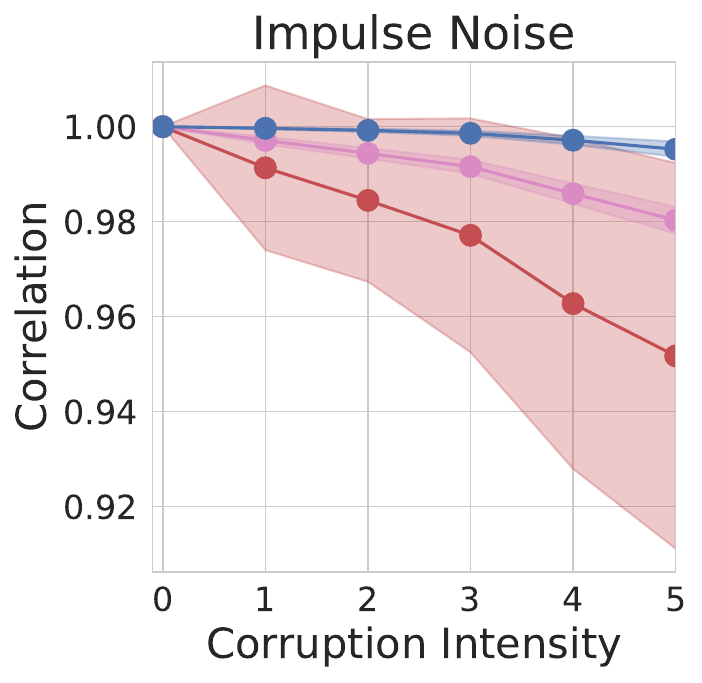} &
		\includegraphics[width=\panelwidth]{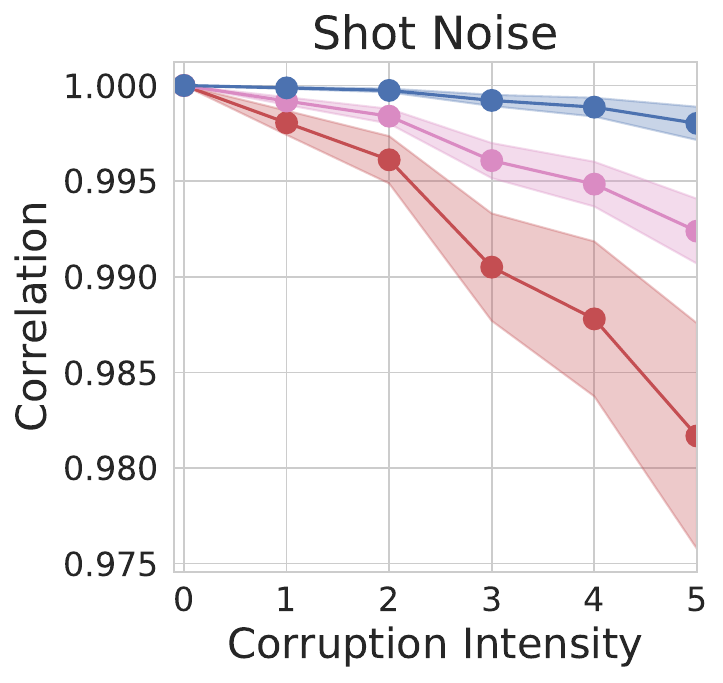}
    \end{tabular}
    }
	\subfigure{
        \begin{tabular}{ccccc}
		\includegraphics[width=\panelwidth]{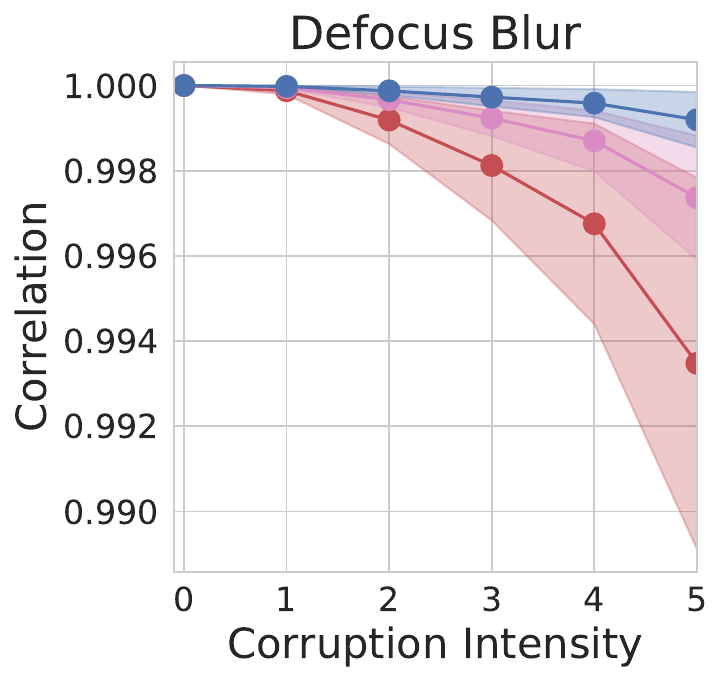} &
		\includegraphics[width=\panelwidth]{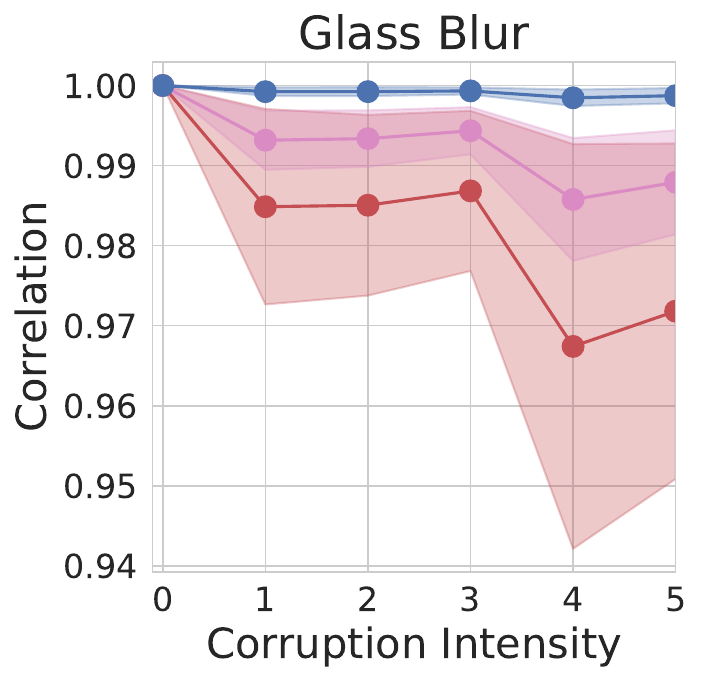} &
		\includegraphics[width=\panelwidth]{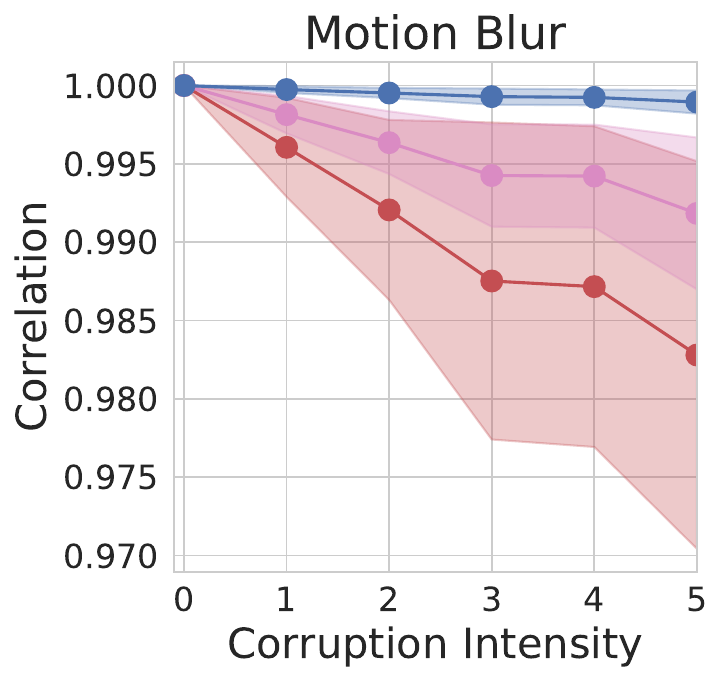} &
		\includegraphics[width=\panelwidth]{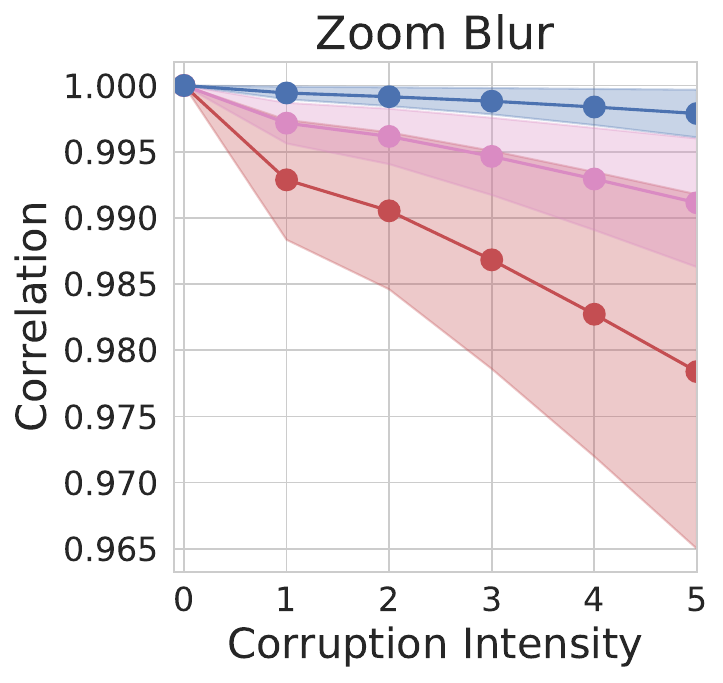}&
		\includegraphics[width=\panelwidth]{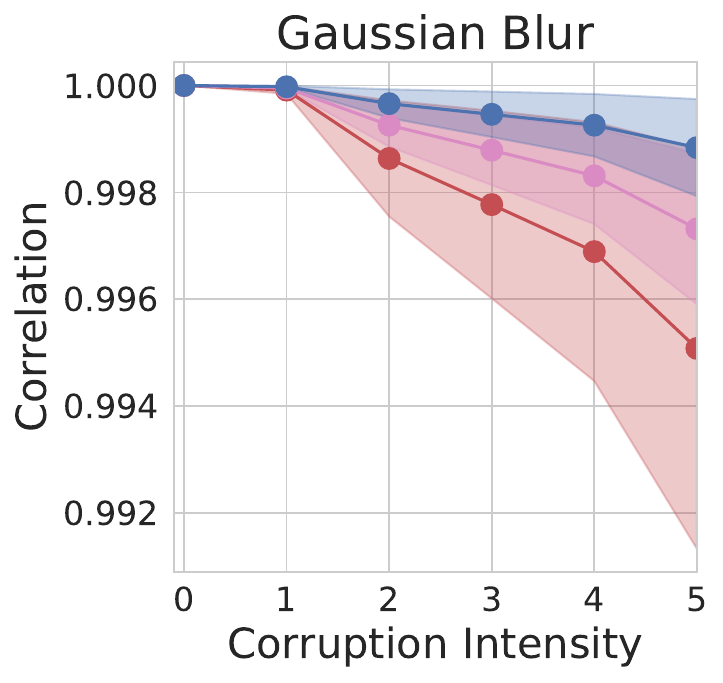}
        \end{tabular}
    }	
	\subfigure{
        \begin{tabular}{ccccc}
		\includegraphics[width=\panelwidth]{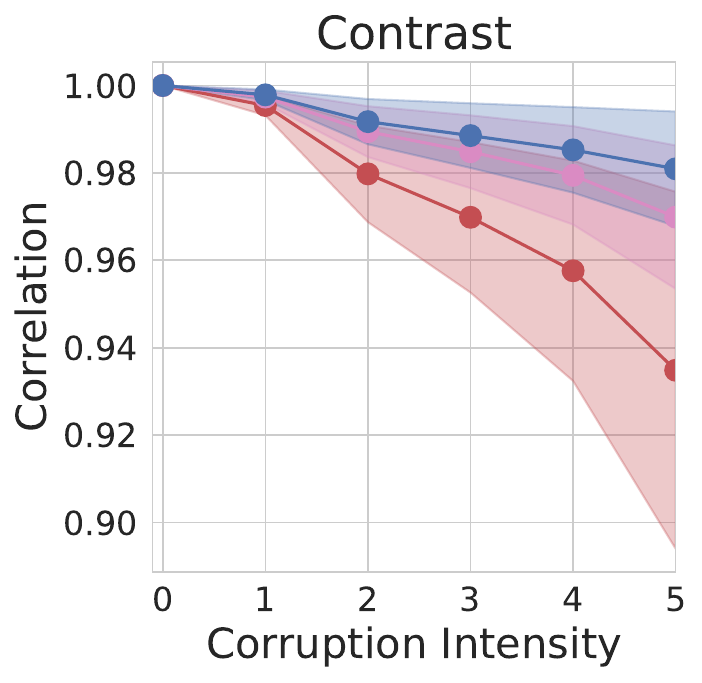} &
		\includegraphics[width=\panelwidth]{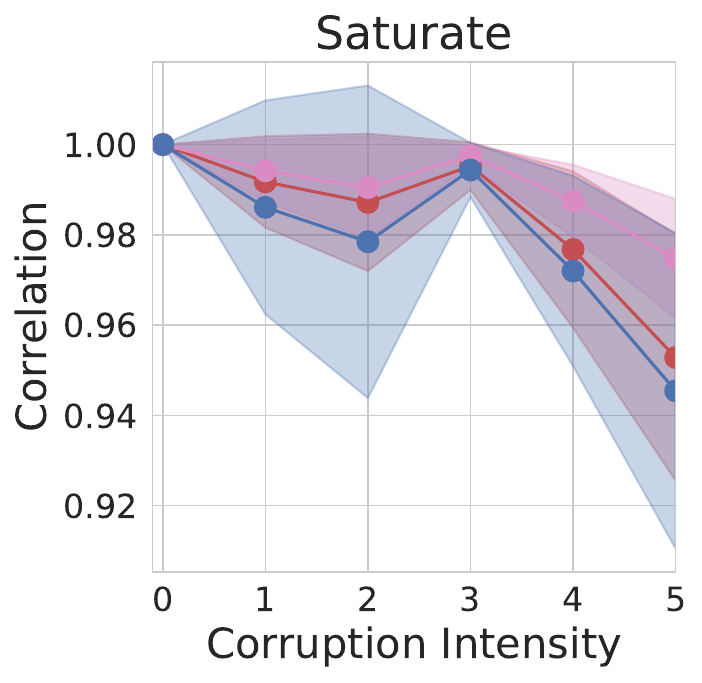} &
		\includegraphics[width=\panelwidth]{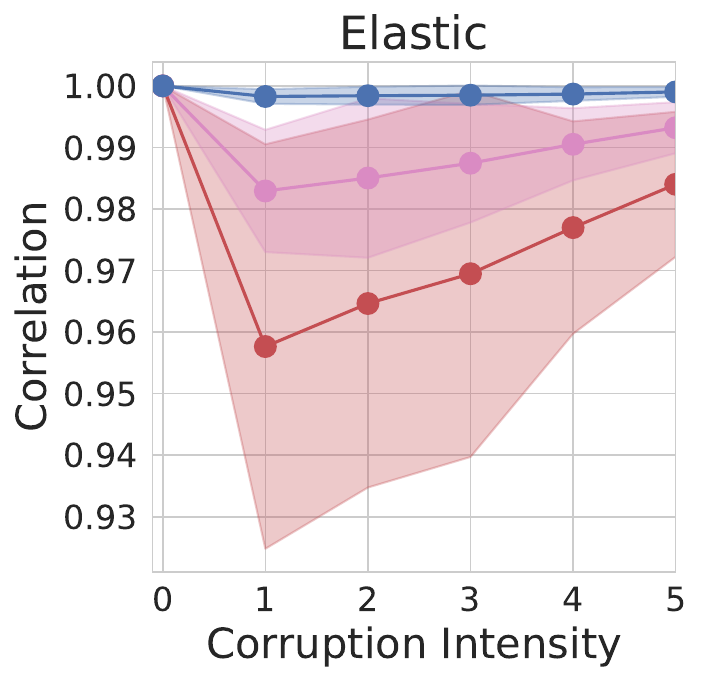} &
		\includegraphics[width=\panelwidth]{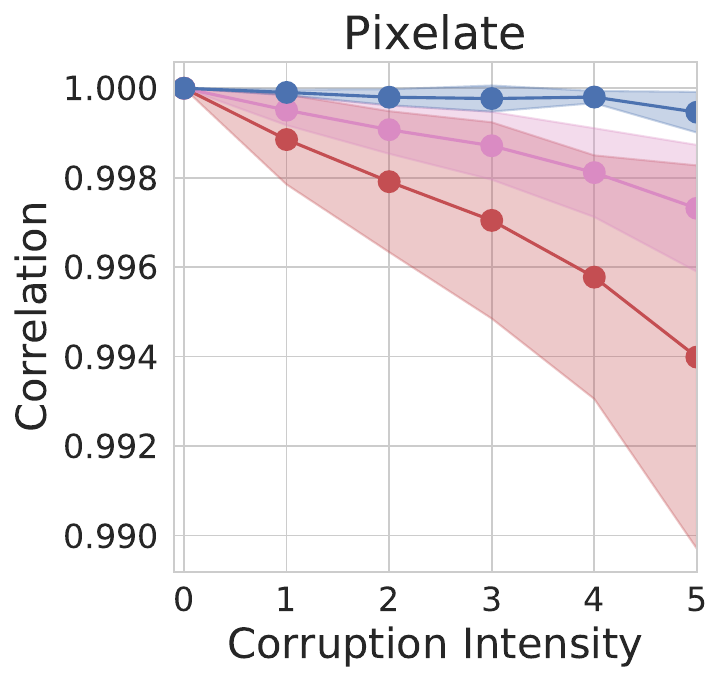} &
		\includegraphics[width=\panelwidth]{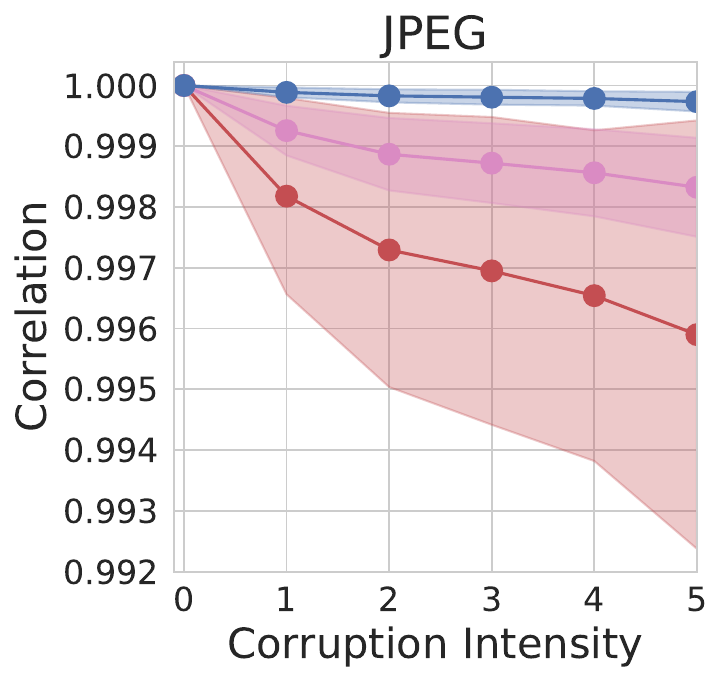}
        \end{tabular}
    }	
	\subfigure{
        \begin{tabular}{ccc}
		\includegraphics[width=\panelwidth]{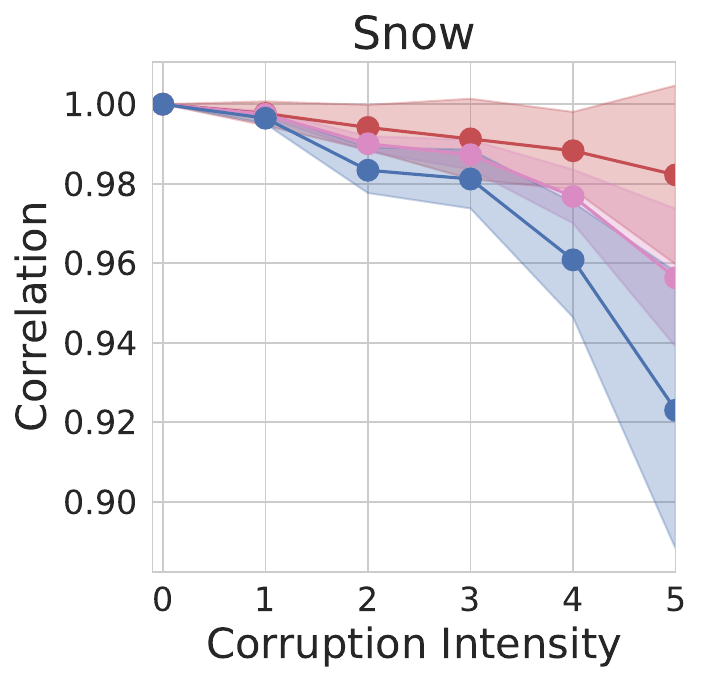} &
		\includegraphics[width=\panelwidth]{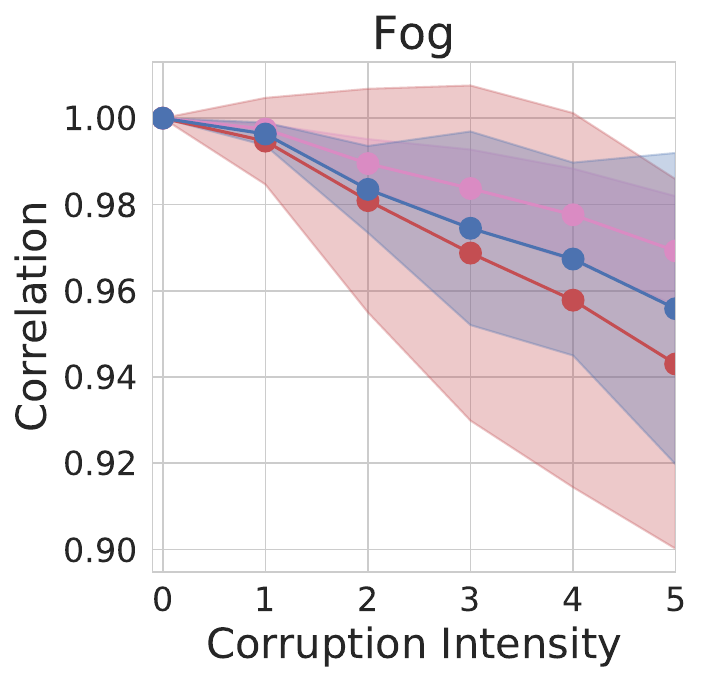} &
		\includegraphics[width=\panelwidth]{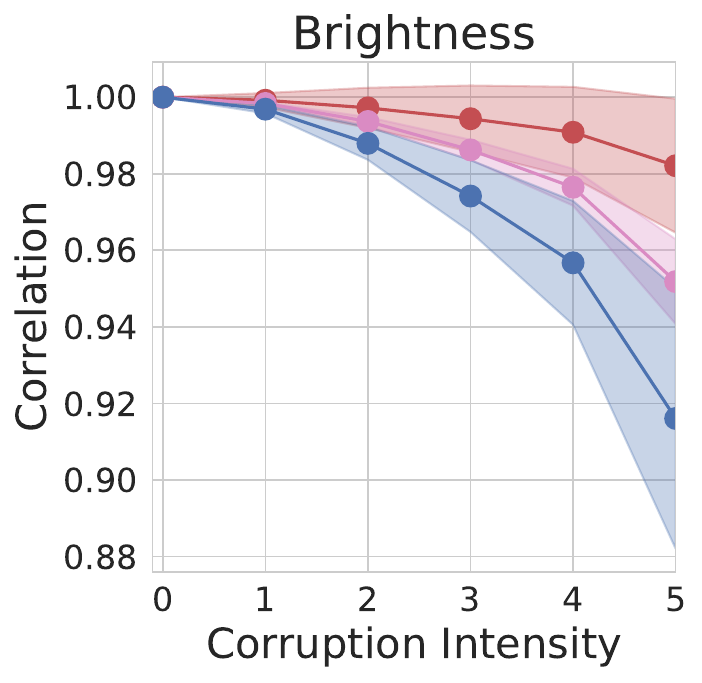} 
        \end{tabular}
    }\\
	\subfigure{
        \includegraphics[width=0.26\textwidth]{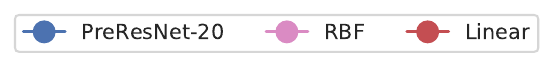}
    }
	\caption{
		\textbf{Prior correlations under corruption.}
        Prior correlations between predictions (logits) for PreResNet-20, Linear Model and RBF kernel
        on original and corrupted images as a function of corruption intensity for different types of corruptions. 
        The lengthscale of the RBF kernell is calibrated to produce similar
        correlations to PreResNet on uncorrupted datapoints.
        We report the mean correlation values over $100$ different images
        and show the $1\sigma$ error bars with shaded regions.
        For all corruptions except Snow, Saturate, Fog and Brightness the 
        correlations decay slower for PreResNet compared to baselines.
	}
    \label{fig:corr_corr}
\end{figure*}

\end{document}